\documentclass[11pt,letterpaper]{article}

\usepackage{amsmath,amssymb,amsthm}
\PassOptionsToPackage{nameinlink}{cleveref}
\usepackage{deepthink}
\usepackage{cleveref}
\usepackage{subcaption}
\usepackage{tcolorbox} 
\usepackage[utf8]{inputenc} 
\usepackage[T1]{fontenc}    
\usepackage{url}
\usepackage{amsmath,amsthm,amssymb,amsbsy,amsfonts,amscd,bm}
\usepackage{xcolor}
\usepackage{color}
\usepackage{graphicx}
\graphicspath{{./figs/}}
\usepackage{algorithm}
\usepackage{algorithmic}
\usepackage{comment}
\usepackage{multirow}
\usepackage{enumitem}
\usepackage{fancyhdr}
\usepackage{authblk}
\usepackage{subcaption}
\usepackage{wrapfig}
\usepackage[toc, page]{appendix}
\usepackage[numbers]{natbib} 

\newtheorem{lemma}{Lemma}
\newtheorem{definition}{Definition}
\newtheorem{assum}{Assumption}

\newtheorem{theorem}{Theorem}

\theoremstyle{remark}

\newcommand{\R}{\mathbb{R}}

\newcommand{\e}{\begin{equation}}
\newcommand{\ee}{\end{equation}}
\newcommand{\en}{\begin{equation*}}
\newcommand{\een}{\end{equation*}}
\newcommand{\eqn}{\begin{eqnarray}}
\newcommand{\eeqn}{\end{eqnarray}}
\newcommand{\bmat}{\begin{bmatrix}}
\newcommand{\emat}{\end{bmatrix}}

\DeclareMathAlphabet\mathbfcal{OMS}{cmsy}{b}{n}

\newcommand{\mtx}[1]{\boldsymbol{#1}}

\newcommand{\diag}{\operatorname{diag}}

\hypersetup{
    colorlinks=true,%
    citecolor=blue,%
    filecolor=blue,%
    linkcolor=blue,%
    urlcolor=blue
}

\newcommand{\mO}{\mtx{O}}

\graphicspath{{./figs/}}

\newlength{\imgwidth}
\newboolean{twoColVersion}
\setboolean{twoColVersion}{false}
\newcommand{\twoCol}[2]{\ifthenelse{\boolean{twoColVersion}} {#1} {#2} }

\newtheorem{proposition}{\bf{Proposition}}

\long\def\comment#1{}

\newcommand{\bx}{\boldsymbol{x}}

\newcommand{\bz}{\boldsymbol{z}}
\newcommand{\bA}{\boldsymbol{A}}

\newcommand{\bI}{\boldsymbol{I}}

\newcommand{\bU}{\boldsymbol{U}}
\newcommand{\bV}{\boldsymbol{V}}
\newcommand{\bW}{\boldsymbol{W}}
\newcommand{\bX}{\boldsymbol{X}}
\newcommand{\bY}{\boldsymbol{Y}}

\long\def\red#1{\bgroup\color{red}#1\egroup}

\definecolor{mich-blue}{HTML}{0027CC}
\definecolor{mich-blue-high}{HTML}{0027CC}
\definecolor{red-high}{HTML}{CA2020}
\definecolor{green-high}{HTML}{20A520}
\definecolor{mich-maize}{HTML}{FFCB05}
\definecolor{law-stone}{HTML}{655A52}
\definecolor{burton-beige}{HTML}{9B9A9D}
\definecolor{arch-ivy}{HTML}{7E732F}

 \colorlet{color1}{gray!15}

\title{Understanding Deep Representation Learning via Layerwise Feature Compression and Discrimination}

\newcommand{\jointfirst}{\textsuperscript{\dag}}
\newcommand{\corrauth}{\textsuperscript{\ddag}}

\authorblock{
  \href{https://peng8wang.github.io/}{\textbf{Peng Wang}}\jointfirst\corrauth\textsuperscript{1},
  \href{https://heimine.github.io/}{\textbf{Xiao Li}}\jointfirst\textsuperscript{1},
  \href{https://canyaras.com/}{\textbf{Can Yaras}}\textsuperscript{1},
  \href{https://zhihuizhu.github.io/}{\textbf{Zhihui Zhu}}\textsuperscript{2},
  \href{https://web.eecs.umich.edu/~girasole/}{\textbf{Laura Balzano}}\textsuperscript{1},
  \href{https://weihu.me/}{\textbf{Wei Hu}}\textsuperscript{1},
  \href{https://qingqu.engin.umich.edu/}{\textbf{Qing Qu}}\corrauth\textsuperscript{1}
}

\affiliation{
  \textsuperscript{1}University of Michigan \quad $\cdot$ \quad \textsuperscript{2}Ohio State University
}

\authornote{
  \jointfirst\ Joint first author \quad
  \corrauth\ Corresponding author
}

\abstracttext{
\noindent Over the past decade, deep learning has proven highly effective at learning meaningful features from raw data. However, it remains an open question how deep networks perform hierarchical feature learning across layers. We investigate this by examining the structure of intermediate features. Motivated by our empirical finding that linear layers mimic the role of deep layers in nonlinear networks, we study how deep linear networks (DLNs) transform input into output by analyzing the per-layer features in multi-class classification. To this end, we define metrics for within-class compression and between-class discrimination of intermediate features. Through theoretical analysis, we show that when the input data is nearly orthogonal and the network weights are minimum-norm, balanced, and approximately low-rank, feature evolution follows a simple, quantitative pattern from shallow to deep layers: each layer compresses within-class features at a geometric rate and discriminates between-class features at a linear rate, both with respect to layer depth. To our knowledge, this is the first quantitative characterization of feature evolution in hierarchical DLN representations. Extensive experiments validate our theory and reveal a similar pattern in deep nonlinear networks, consistent with recent empirical studies. Finally, we show the practical value of these results for transfer learning.
}


\date{\today}
\correspondence{\href{mailto:pengw@um.edu.mo}{pengw@um.edu.mo},\quad \href{mailto:xlxiao@umich.edu}{xlxiao@umich.edu}}
\resources{\href{https://heimine.github.io/PNC_DLN//}{Project Page} $\cdot$ \href{https://github.com/Heimine/PNC_DLN}{Code Repository}}

\headerlogo{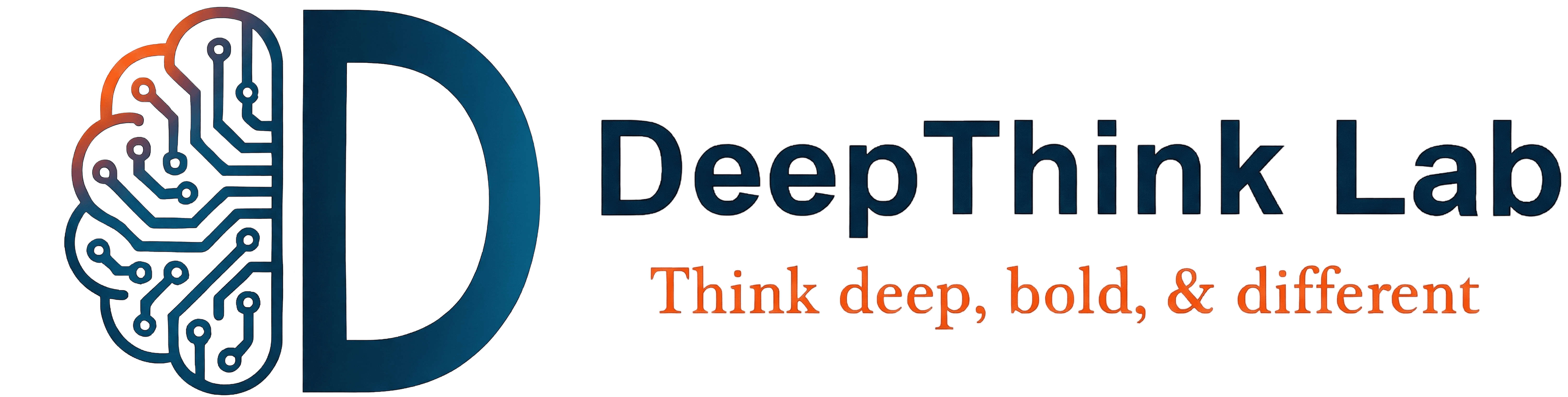}{https://deepthink-umich.github.io}

\begin{document}

\makeDeepthinkHeader



\vspace{-0.2in}
\begin{figure}[h]
    \centering
    \begin{subfigure}{0.49\textwidth}
    \includegraphics[width = 0.8\linewidth]{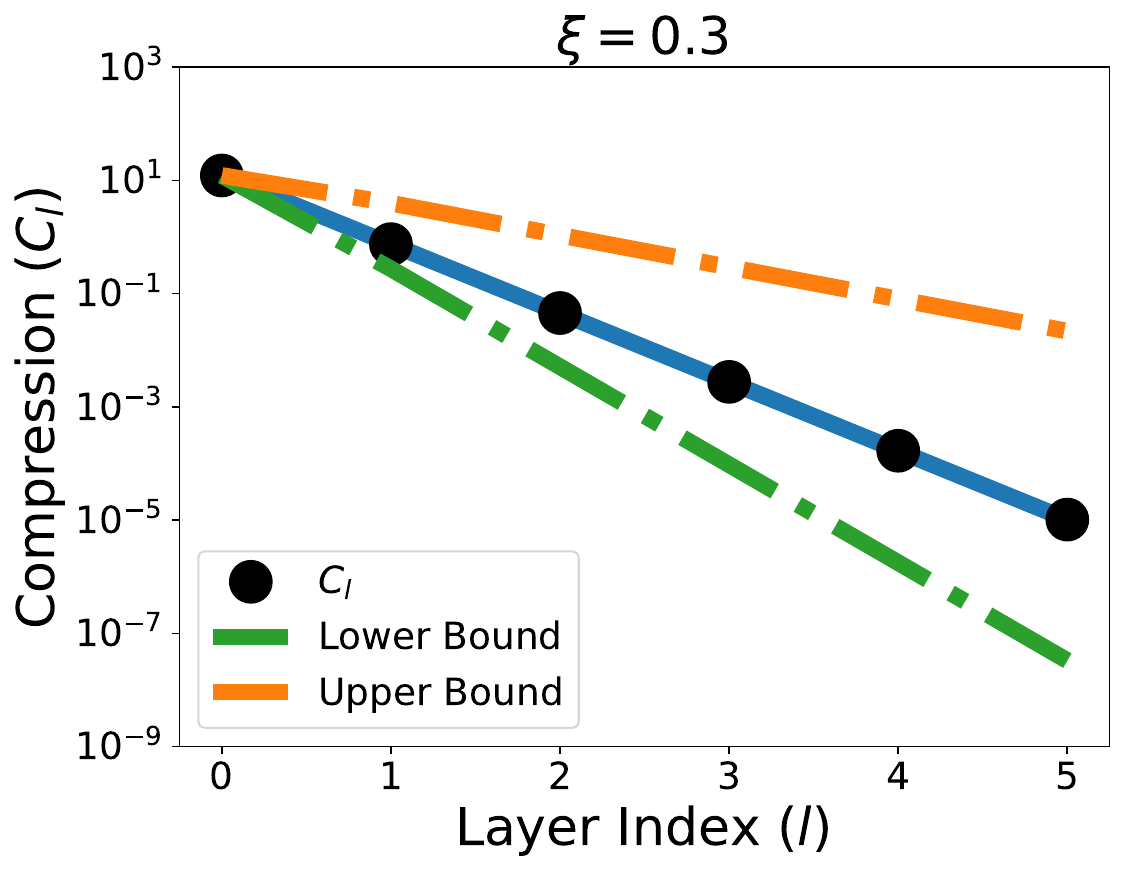}
    \caption{\footnotesize Within-class compression $C_l = \frac{\mathrm{Tr}(\bm{\Sigma}_W^l)}{\mathrm{Tr}(\bm{\Sigma}_B^l)}$} 
    \end{subfigure} 
    \begin{subfigure}{0.49\textwidth}
    \includegraphics[width = 0.8\linewidth]{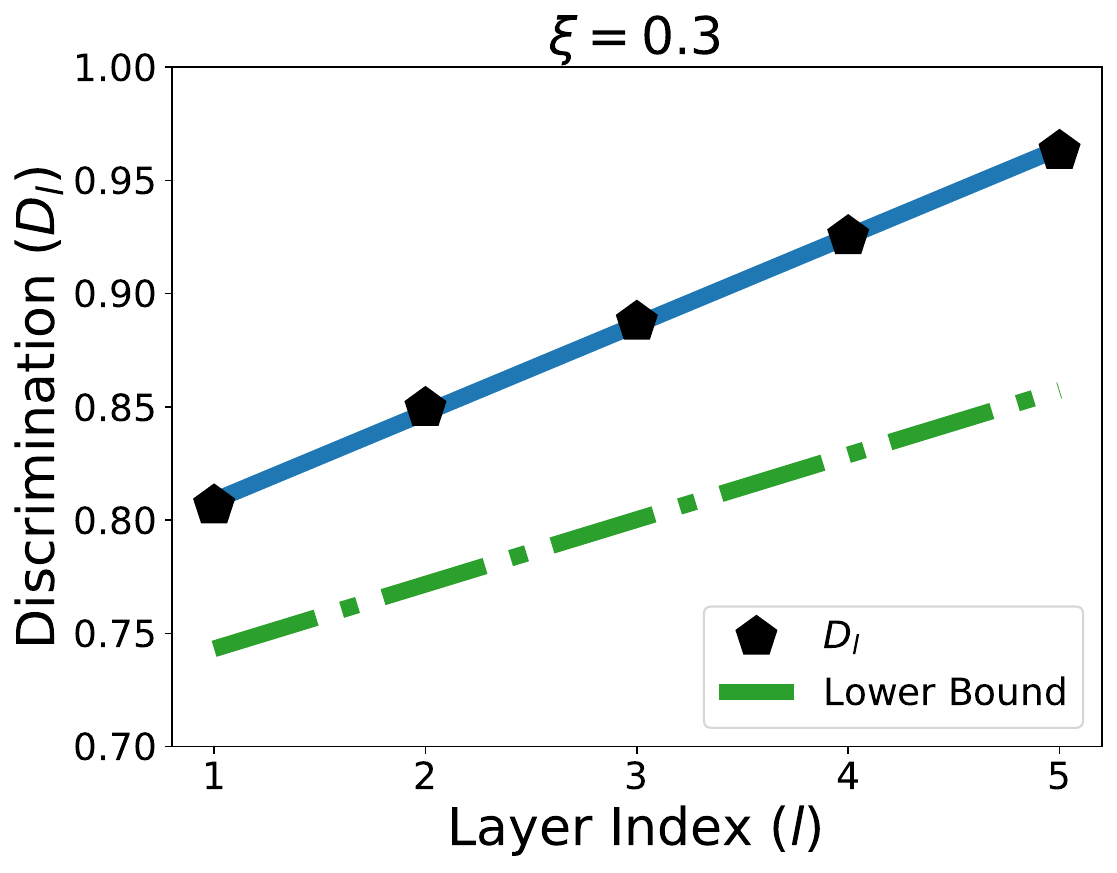}
    \caption{\footnotesize Between-class discrimination $D_l = 1 - \max_{k \neq k'} \frac{\langle \bm{\mu}_k^l, \bm{\mu}_{k'}^l \rangle}{\|\bm{\mu}_k^l\| \|\bm{\mu}_{k'}^l\|}$} 
    \end{subfigure} 
    \vspace{-0.1in}
    \caption{\footnotesize \textbf{Theoretical bounds on feature compression and discrimination.} Here, $\bm{\Sigma}_W^l$ and $\bm{\Sigma}_B^l$ denote the within-class and between-class covariance of features at layer $l$, and $\bm{\mu}_k^l$ denotes the mean feature of the $k$-th class. For a 6-layer DLN trained with $\xi=0.3$, the within-class compression metric $C_l$ (left) tightly matches its bounds in \eqref{eq2:compres}, while the between-class discrimination metric $D_l$ (right) follows the bound in \eqref{eq:discri}.}
\label{fig:check_bound_teaser}
\end{figure}




\newpage
\tableofcontents

\newpage
\section{Introduction}\label{sec:intro} 

In the past decade, deep learning has exhibited remarkable success across a wide range of applications in engineering and science \citep{lecun2015deep}, such as computer vision \citep{he2016deep,simonyan2014very}, natural language processing \citep{sutskever2014sequence,vaswani2017attention}, and health care \citep{esteva2019guide}, to name a few. It is commonly believed that one major factor contributing to the success of deep learning is its ability to perform {\em hierarchical feature learning}: deep networks can leverage their hierarchical architectures to extract meaningful and informative features\footnote{In deep networks, the (last-layer) feature typically refers to the output of the penultimate layer, which is also called {\em representation} in the literature.} from raw data \citep{allen2023backward,krizhevsky2012imagenet}. Despite recent efforts to understand deep networks, the underlying mechanism of how deep networks perform hierarchical feature learning across layers still remains a mystery, even for classical supervised learning problems, such as multi-class classification. Gaining deeper insight into this question will offer theoretical principles to guide the design of network architectures \citep{he2023law}, shed light on generalization and transferability \citep{li2022principled}, and facilitate network training \citep{xie2022hidden}.

\paragraph{Empirical results on feature expansion and compression.} Towards opening the black box of deep networks, extensive empirical research has been conducted in recent years by investigating outputs at each layer of deep networks. An intriguing line of research has empirically investigated the role of different layers in feature learning; see, e.g., \cite{alain2016understanding,ansuini2019intrinsic,chen2022layer,masarczyk2023tunnel,recanatesi2019dimensionality,zhang2022all}. In general, these empirical studies demonstrate that the initial layers expand the intrinsic dimension of features to make them linearly separable, while the subsequent layers compress the features progressively; see \Cref{fig:intro1}. For example, in image classification tasks, \cite{alain2016understanding} observed that the features of intermediate layers are increasingly linearly separable as we reach the deeper layers. Recent works \citep{ansuini2019intrinsic,recanatesi2019dimensionality}  studied the evolution of the intrinsic dimension of intermediate features across layers in trained networks using different metrics. They both demonstrated that the dimension of the intermediate features first blows up and subsequently goes down from shallow to deep layers. More recently, \cite{masarczyk2023tunnel} delved deeper into the role of different layers and concluded that the initial layers create linearly separable intermediate features, while later layers compress these features progressively.  


\begin{figure*}[t]
    \begin{subfigure}{0.48\textwidth}
    \includegraphics[width = 0.95\linewidth]{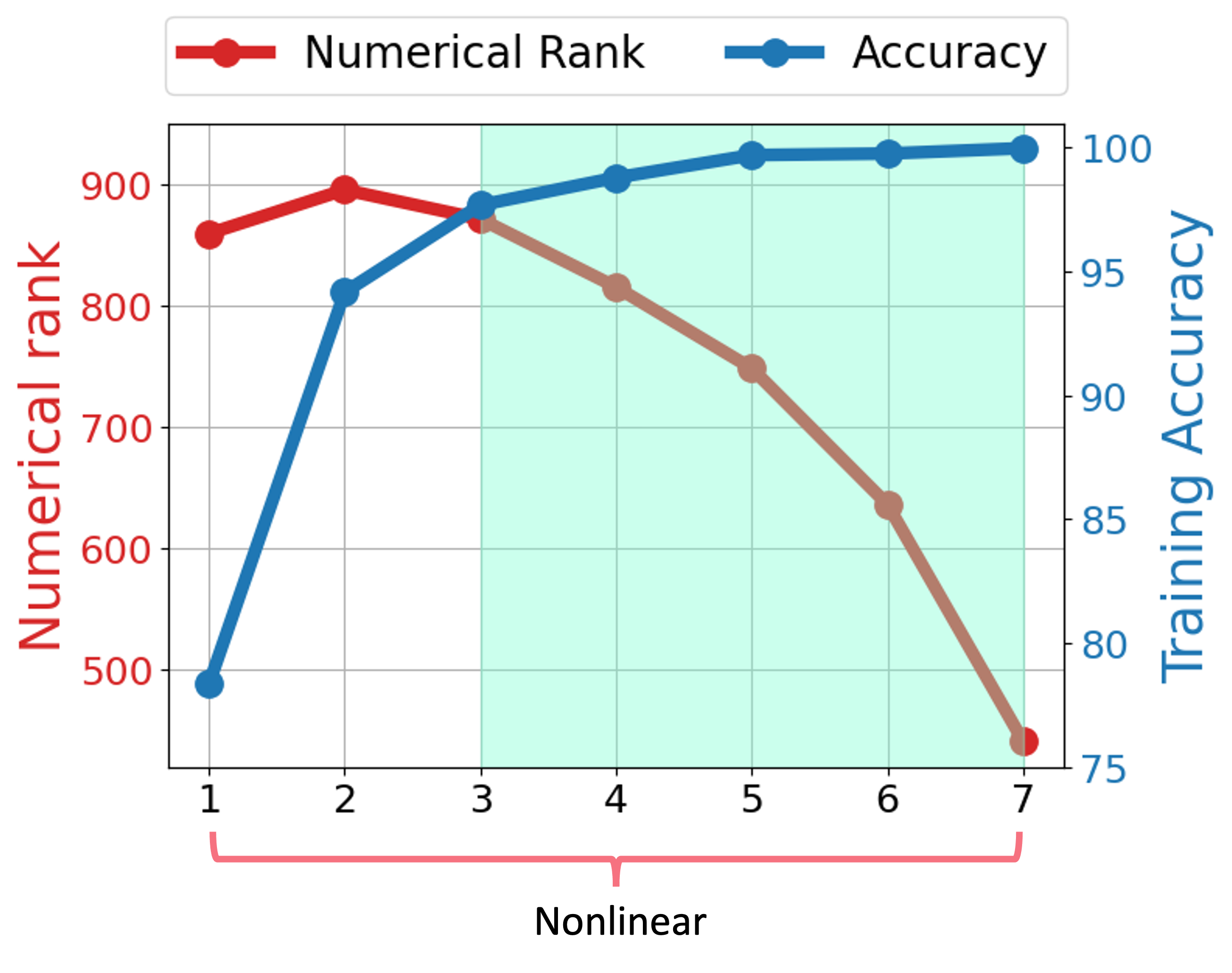}
    \caption{A 8-layer MLP network with ReLU} 
    \end{subfigure} 
    \begin{subfigure}{0.48\textwidth}
    \includegraphics[trim={0 0 0 1cm}, width = 0.94\linewidth]{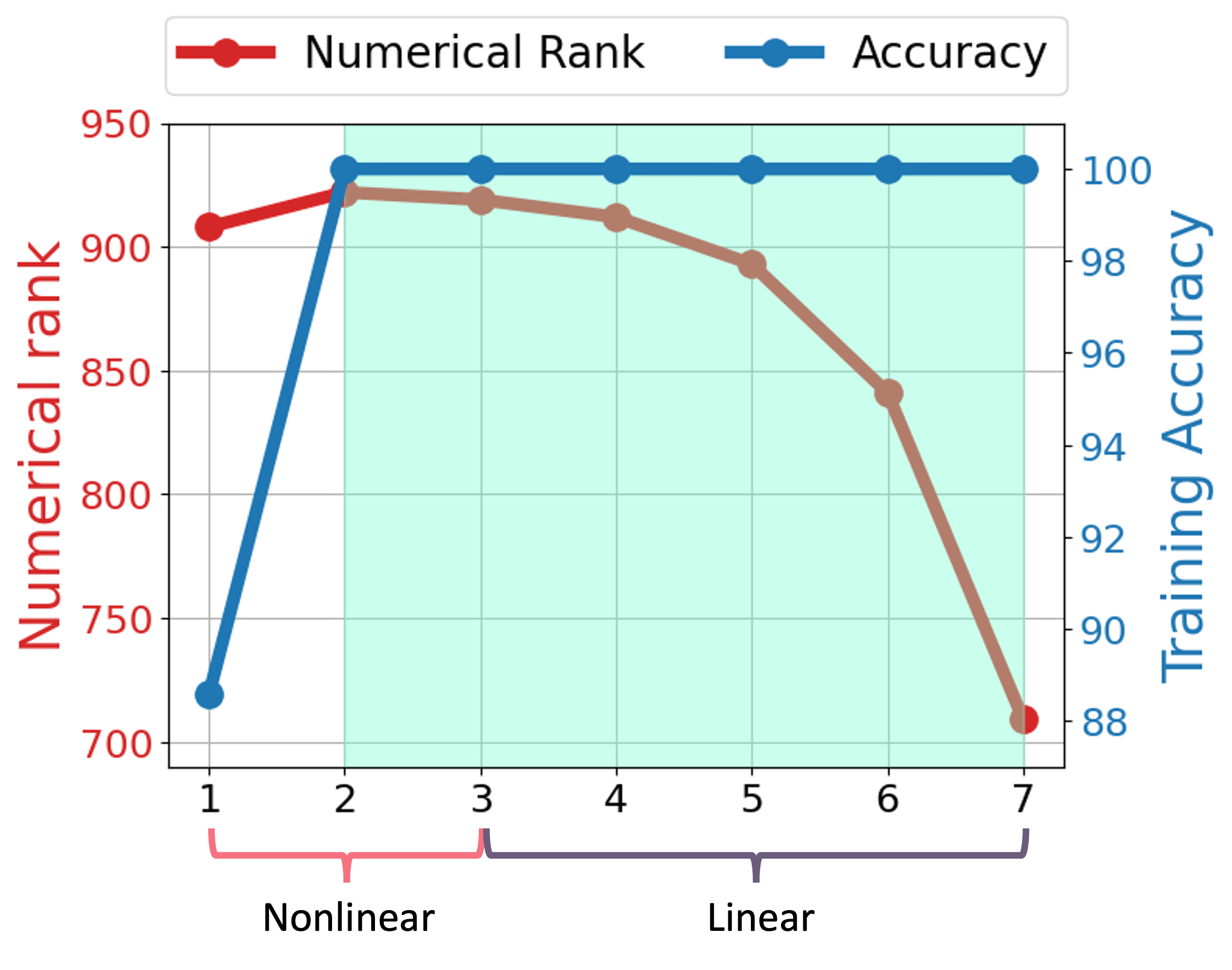}
    \caption{A 8-layer hybrid network} 
    \end{subfigure} 
   \caption{\textbf{Illustration of numerical rank and training accuracy across layers.} We train two networks with different architectures on the CIFAR-10 dataset: (a) A 8-layer multilayer perceptron (MLP) network with ReLU activation, (b) A hybrid network consisting of a 3-layer MLP with ReLU activation followed by a 5-layer linear network. For each figure, we plot the numerical rank of the features of each layer and the training accuracy obtained by applying linear probing to the output of each layer, both against the number of layers. The green shading indicates that the features at these layers are {\em approximately} linearly separable, as evidenced by the near-perfect accuracy achieved by a linear classifier. The definition of numerical rank and additional experimental details are deferred to \Cref{subsubsec:resemb_1}.} 
    \label{fig:intro1}
\end{figure*}

\paragraph{Empirical results on feature compression and discrimination.} Meanwhile, recent works \citep{papyan2020prevalence,Han2021,fang2021exploring,zhu2021geometric} have provided systematic studies on the structures of intermediate features. They revealed a fascinating phenomenon termed neural collapse (NC) during the terminal phase of training deep networks and across many different datasets and model architectures. Specifically, NC refers to a training phenomenon in which the last-layer features from the same class become nearly identical, while those from different classes become maximally linearly separable. In other words, deep networks learn within-class compressed and between-class discriminative features. Building upon these studies, a more recent line of work investigated the NC properties at each layer to understand how features are transformed from shallow to deep layers; see, e.g., \cite{ben2022nearest,he2023law,hui2022limitations,galanti2022implicit,rangamani2023feature}. In particular, \cite{he2023law} empirically showed that a progressive NC phenomenon, governed by a law of data separation, occurs from shallow to deep layers. 
\cite{rangamani2023feature} empirically showed that similar NC properties emerge in intermediate layers during training, where within-class variance decreases relative to the between-class variance as layers go deeper.

\begin{figure*}[t]
    \begin{subfigure}{1.0\textwidth}
    \includegraphics[width = 0.24\linewidth]{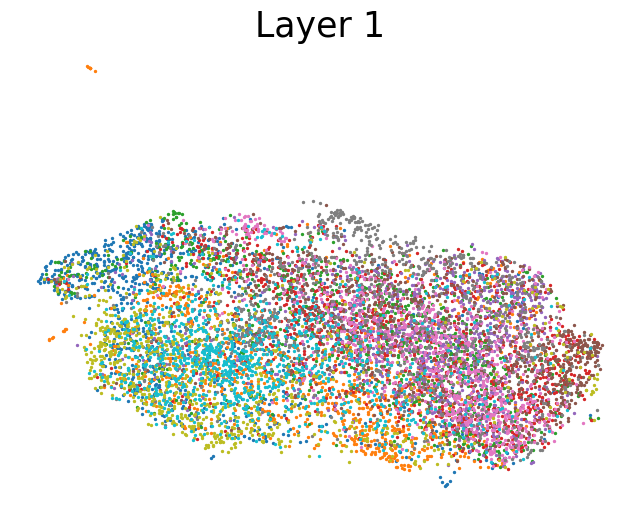}
    \includegraphics[width = 0.24\linewidth]{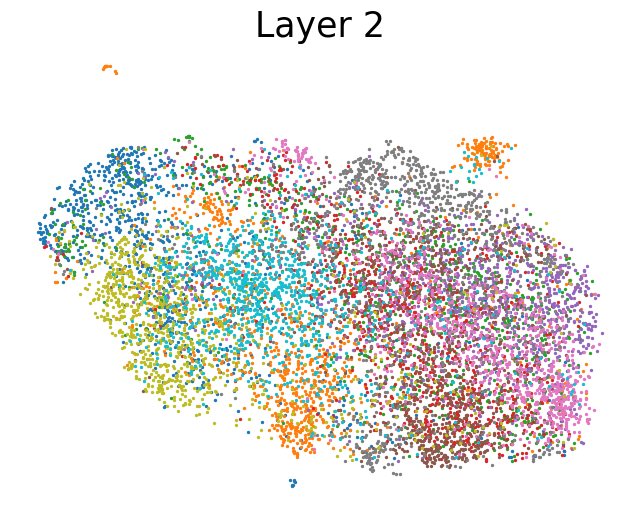}
    \includegraphics[width = 0.24\linewidth]{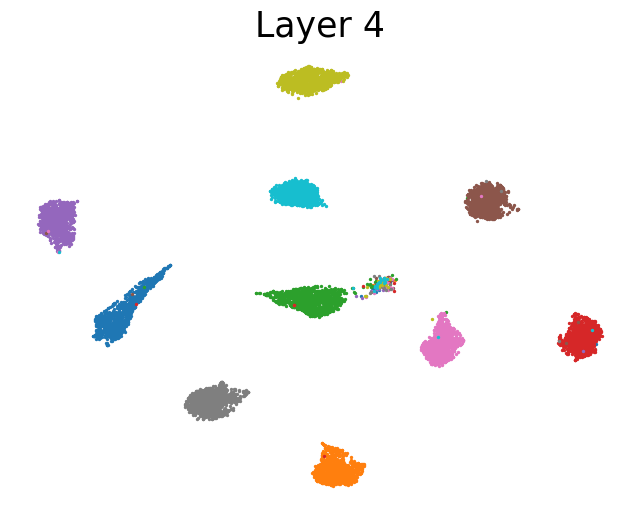}
    \includegraphics[width = 0.24\linewidth]{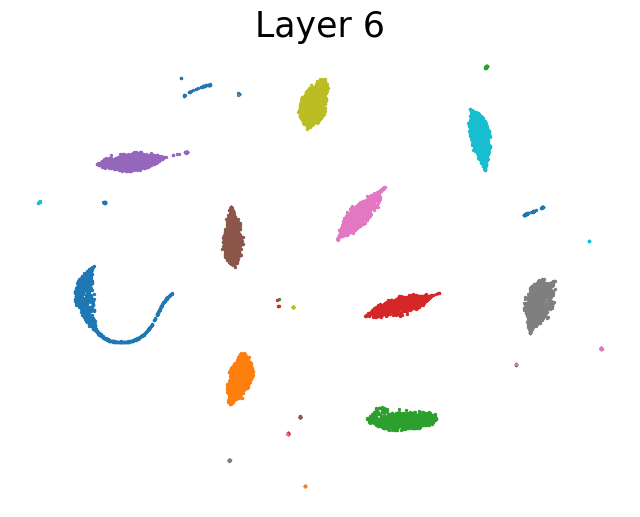}
    \caption{A 8-layer nonlinear MLP} 
    \end{subfigure} 

    \vspace{1mm}
    \begin{subfigure}{1.0\textwidth}
    \includegraphics[width = 0.24\linewidth]{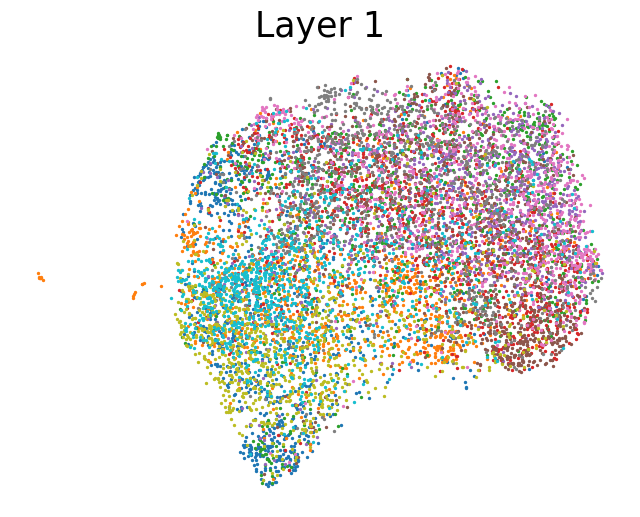}
    \includegraphics[width = 0.24\linewidth]{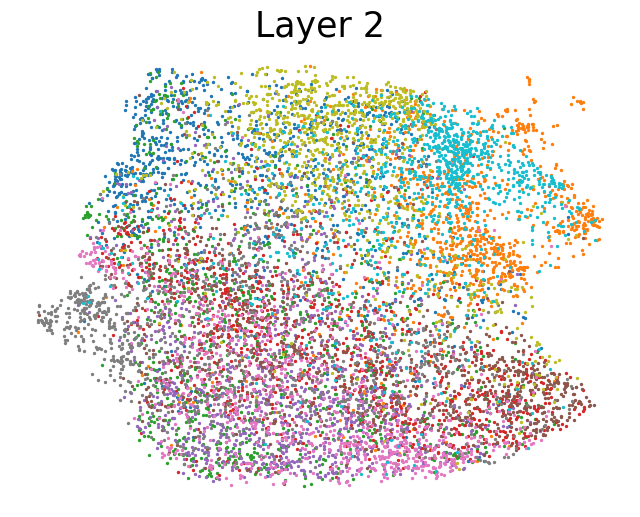}
    \includegraphics[width = 0.24\linewidth]{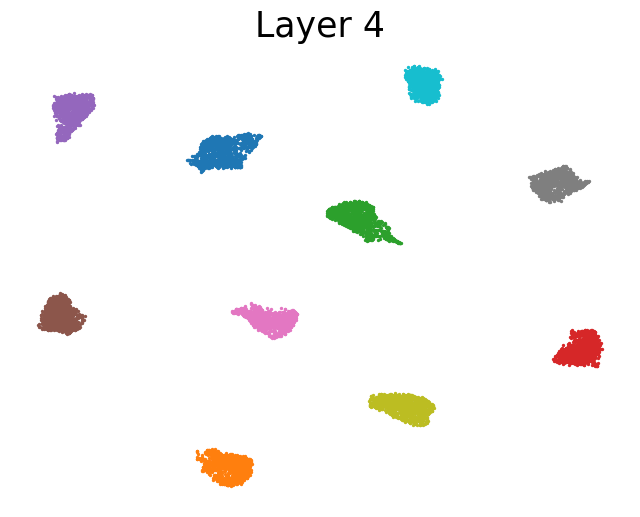}
    \includegraphics[width = 0.24\linewidth]{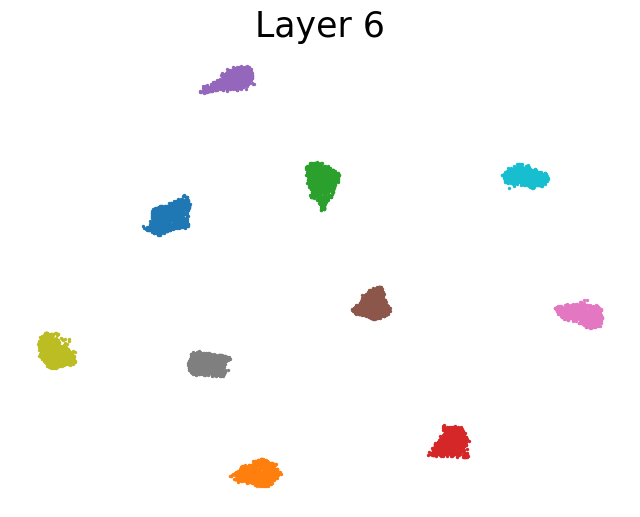}
    \caption{A hybrid network with 3-layer MLP + 5-layer DLN} 
    \end{subfigure} 

    \caption{\textbf{Visualization of feature compression \& discrimination from shallow to deep layers.} We consider the same setup as in \Cref{fig:intro1}.  For each network, we visualize the outputs of layers 1, 2, 4, and 6 on the CIFAR-10 dataset using the 2-dimensional UMAP plot \citep{mcinnes2018umap}. Additional experimental details are deferred to \Cref{subsubsec:resemb_1}.} 
    \label{fig:umap}
\end{figure*} 

In summary, extensive empirical results demonstrate that after feature expansion by initial layers, {\em deep networks progressively compress features within the same class and discriminate features from different classes from shallow to deep layers}; see \Cref{fig:umap}. This characterization provides valuable insight into how deep networks transform data into output across layers in classification tasks. Moreover, this insight sheds light on designing more advanced network architectures, developing more efficient training strategies, and achieving better interpretability. However, to the best of our knowledge, no theoretical framework has yet been established to explain this empirical observation of progressive feature compression and discrimination. In this work, we take a first step towards bridging this gap by providing a theoretical analysis based on deep linear networks (DLNs).  

\paragraph{Why study DLNs? Linear layers mimic deep layers in nonlinear networks for feature learning.} Even though DLNs lack the strong expressive power of nonlinear networks, they possess comparable abilities for feature compression and discrimination to those observed in the deeper layers of nonlinear networks, as indicated in Figures \ref{fig:intro1} and \ref{fig:umap}. By evaluating the training accuracy and the numerical rank of intermediate features in both a nonlinear network and a hybrid network\footnote{For the hybrid network, we introduce nonlinearity in the first few layers and follow them with linear layers.}, we observe that the initial-layer features in both networks are almost linearly separable, evidenced by nearly perfect training accuracy achieved through linear probing. This phenomenon is further illustrated by the feature visualization in \Cref{fig:umap}. Meanwhile, the linear layers in the hybrid network mimic the role of their counterpart in the nonlinear network by performing feature compression and discrimination, as evidenced by the decreasing feature rank in \Cref{fig:intro1} and the increasing separation of different-class features in \Cref{fig:umap} across layers in both types of networks.

Broadly speaking, DLNs have been recognized as valuable prototypes for studying nonlinear networks, as they resemble certain behaviors of their nonlinear counterparts \citep{alain2016understanding,ansuini2019intrinsic,masarczyk2023tunnel,recanatesi2019dimensionality} while maintaining simplicity \citep{arora2018optimization,gidel2019implicit,saxe2019mathematical}. For instance, \cite{huh2021low} empirically demonstrated the presence of a low-rank bias at both initialization and after training for both linear and nonlinear networks. \cite{saxe2019mathematical} showed that a DLN exhibits a striking hierarchical progressive differentiation of structures in its internal hidden representations, resembling patterns observed in their nonlinear counterparts.  

\paragraph{The role of depth in DLNs: improving generalization, feature compression, and training speed.} Although stacking linear layers in DLNs ultimately results in an end-to-end linear transformation from input to output, the overparameterization in these DLNs distinguishes them from a basic linear operator: increasing the depth of DLNs can significantly improve their generalization capabilities, enhance feature compression, and facilitate network training.
Specifically, recent works have demonstrated that linear over-parameterization by depth (i.e., expanding one linear layer into a composition of multiple linear layers) in deep nonlinear networks yields better generalization performance across different network architectures and datasets \citep{guo2020expandnets,huh2021low,kwon2023compressing}. This is corroborated by our experiments in \Cref{fig:why_dln}, where increasing the depth of linear layers of a hybrid network leads to improved test accuracy. Moreover, our results in \Cref{fig:test-init-dataset} suggest that increasing the depth of DLN also leads to improved feature compression. 
We refer interested readers to \cite{nichani2021empirical} for further discussion on the role of depth in DLNs.

\begin{figure*}[t]
    \begin{subfigure}{0.48\textwidth}
    \includegraphics[width = 0.98\linewidth]{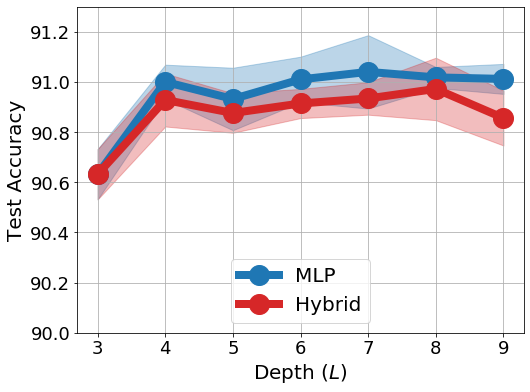} \vspace{-0.05in}
    \caption{FashionMNIST} 
    \end{subfigure} 
    \begin{subfigure}{0.48\textwidth}
    \includegraphics[width = 0.95\linewidth]{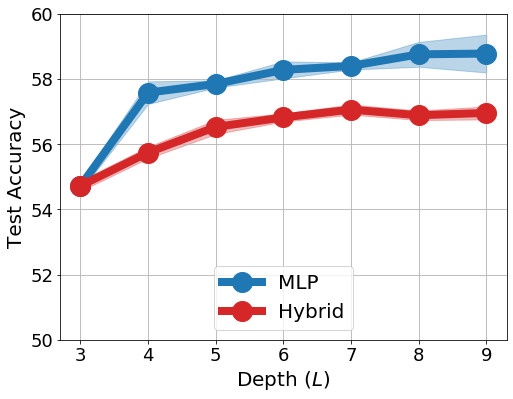} \vspace{-0.05in}
    \caption{CIFAR-10} 
    \end{subfigure}
   \caption{\textbf{Depth of DLNs lead to better generalization performance.} We train hybrid networks consisting of a 2-layer MLP with ReLU activation followed by $(L-2)$ linear layers on the FashionMNIST and CIFAR-10 datasets, respectively. As a reference, we also train nonlinear networks comprised exclusively of MLP layers. We plot the test accuracy against the different number of layers averaged over $5$ different runs. It is observed that adding either linear layers or MLP layers can improve generalization performance. More experimental details are deferred to \Cref{subsubsec:linear_gene}.}
    \label{fig:why_dln} 
\end{figure*}

\begin{figure}[t]
    \begin{subfigure}{0.48\textwidth}
    \includegraphics[width = 0.8\linewidth]{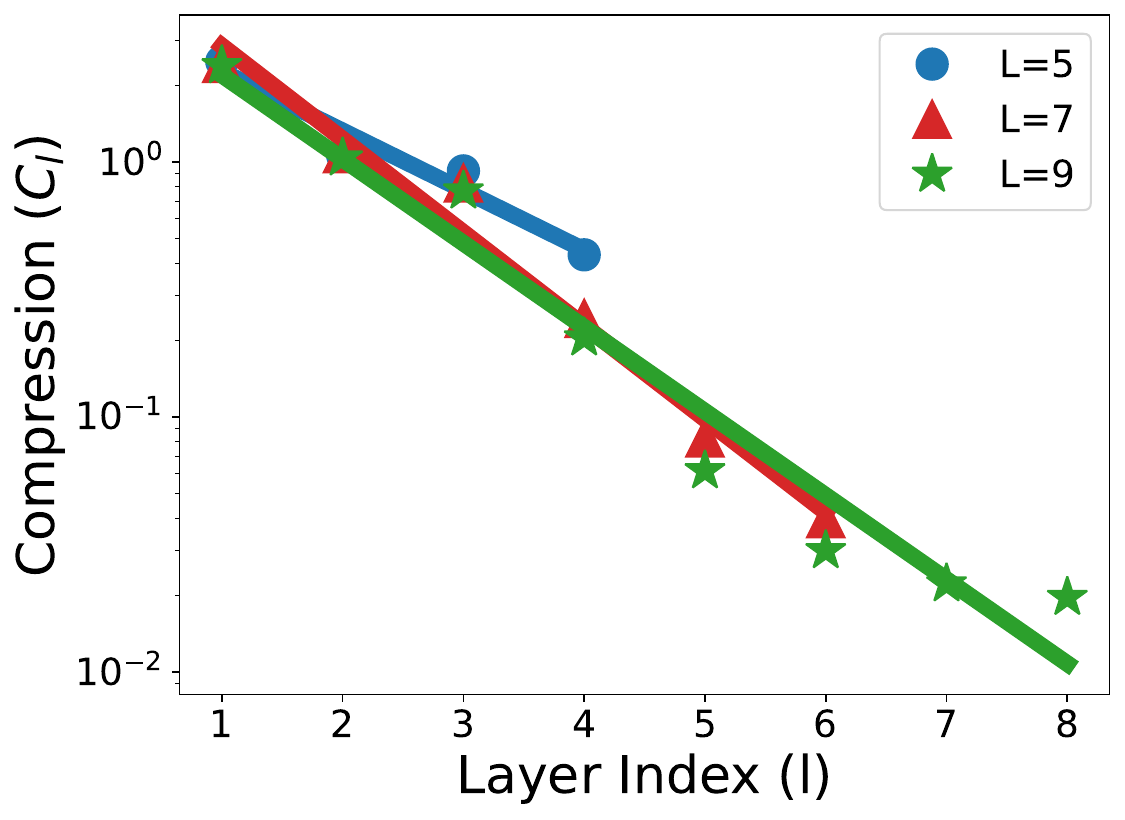}
    \caption{FashionMNIST} 
    \end{subfigure} 
    \begin{subfigure}{0.48\textwidth}
    \includegraphics[width = 0.8\textwidth]{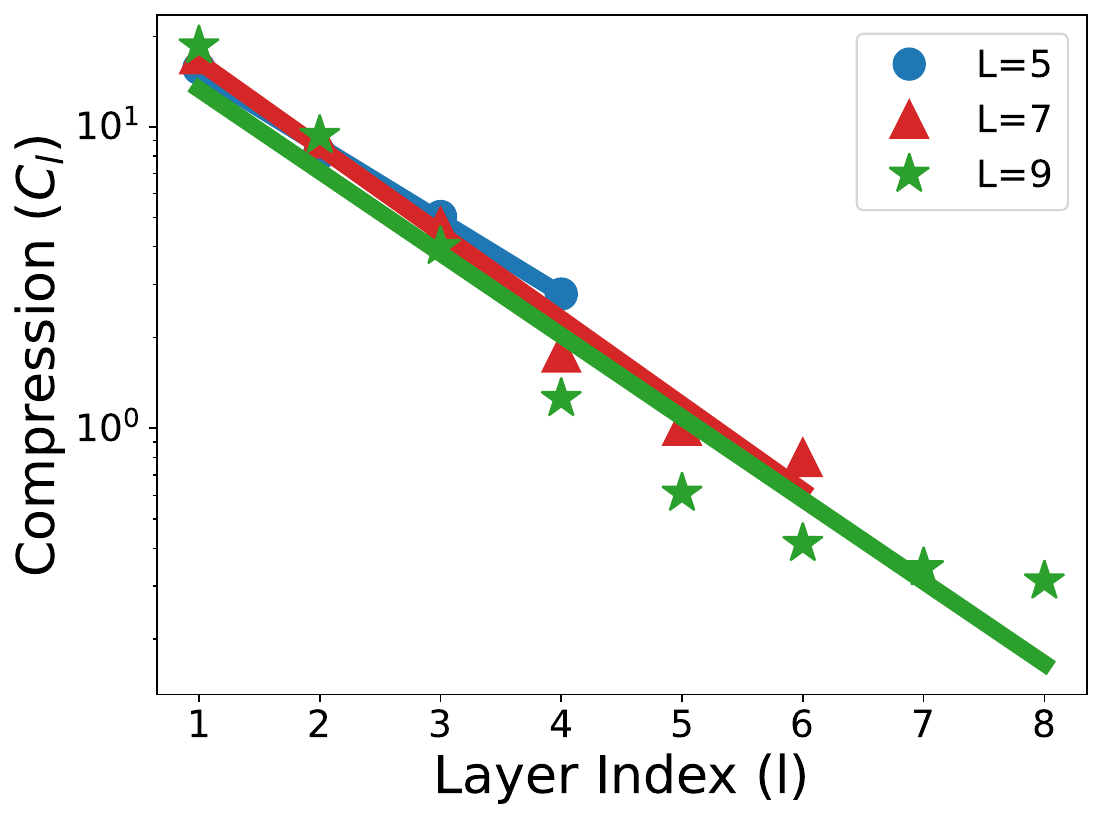}
    \caption{CIFAR-10}
    \end{subfigure} 
    \caption{\textbf{Progressive feature compression on DLNs trained with default initialization and real datasets.} Using the DLNs trained in \Cref{fig:why_dln}, we plot the within-class compression metrics $C_l$ (see Definition \ref{def:nc1}) against the layer indices. It is observed that progressive linear decay still (approximately) happens without the orthogonal initialization and datasets described in Assumption~\ref{AS:1}.   }
    \label{fig:test-init-dataset}
\end{figure} 

\subsection{Our Contributions}

In this work, we study hierarchical representations of deep networks for multi-class classification problems. Towards this goal, we explore how a DLN transforms input data into a one-hot encoding label matrix from shallow to deep layers by investigating features at each layer. To characterize the structures of intermediate features, we define metrics to measure within-class feature compression and between-class feature discrimination at each layer, respectively (see Definition \ref{def:nc1}). We establish a unified framework to analyze these metrics and reveal a simple and quantitative pattern in the evolution of features from shallow to deep layers: 
\begin{tcolorbox}\centering 
\emph{The compression metric \textbf{decays at a geometric rate}, while the discrimination metric \textbf{increases at a linear rate} with respect to the number of layers.}
\end{tcolorbox}
More specifically,  we rigorously prove the above claim under the following assumptions: 
\begin{itemize}[leftmargin=*]
    \item \emph{Assumption on the training data.} We study a $K$-class classification problem with balanced training data, where each class contains $n$ training samples. We assume that the training samples are $\theta$-nearly orthonormal for a constant $\theta \in [0,1/4)$, and the feature dimension is larger than the total number of samples so that they are linearly separable (see Assumption \ref{AS:1}).
    
    \item \emph{Assumption on the trained weights.} We assume that an $L$-layer DLN is trained such that its weights are minimum-norm, $\delta$-balanced, and $\varepsilon$-approximately low-rank, due to the implicit bias of gradient descent, where $\delta,\varepsilon \in (0,1)$ are constants (see Assumption \ref{AS:2}).
\end{itemize}
We discuss the validity of our assumptions in \Cref{subsec:discuss}. Based upon these assumptions, we show in Theorem~\ref{thm:NC} that the ratio of the within-class feature compression metric between the $(l+1)$-th layer and the $l$-th layer is $O(\varepsilon^2/n^{1/L})$, implying a geometric decay from shallow to deep layers. Moreover, we show that the between-class discrimination metric increases linearly with respect to (w.r.t.) the number of layers, with a slope of $O((\theta+4\delta)/L)$. To the best of our knowledge, this is the first quantitative characterization of feature evolution in hierarchical representations of DLNs in terms of feature compression and discrimination. Notably, this analytical framework can be extended to nonlinear networks \citep{jacot2025wide}. Finally, we substantiate our theoretical findings in \Cref{sec:exp} on both synthetic and real-world datasets, demonstrating that our claims hold for DLNs and also manifest empirically in nonlinear networks. 

\paragraph{Significance of our results.} In recent years, there has been a growing body of literature studying hierarchical feature learning to open the black box of deep networks. These studies include works on neural tangent kernel \citep{Jacot2018,huang2020dynamics}, intermediate feature analysis \citep{alain2016understanding,masarczyk2023tunnel,rangamani2023feature}, neural collapse \citep{dang2023neural,he2023law,Sukenik2023,tirer2022extended}, and learning dynamic analysis \citep{allen2023backward,bietti2022learning,Damian2022}, among others. We refer readers to \Cref{subsec:prior-arts} for a more comprehensive discussion. Our work contributes to this emerging area by showing that each layer of deep networks plays an equally important role in hierarchical feature learning, which compresses within-class features at a geometric rate and discriminates between-class features at a linear rate w.r.t. the number of layers. This provides a simple and precise characterization of how deep networks transform data hierarchically from shallow to deep layers. It also addresses an open question about neural collapse and offers a new perspective justifying the importance of depth in feature learning; see the discussion in \Cref{subsec:thm}. 
Moreover, our result explains why projection heads \citep{chen2020simple}, which usually refer to one or several MLP layers added between the feature layer and final classifier during pretraining and discarded afterwards, can improve the performance of transfer learning on downstream tasks \citep{li2022principled,kornblith2021why,galanti2022on}. In summary, our result provides important guiding principles for deep learning practices related to interpretation, architecture design, and transfer learning.  

\paragraph{Differences and connections to the existing literature.} Finally, we highlight the differences and connections between our work and two closely related recent works \citep{he2023law,saxe2019mathematical} as follows:
\begin{itemize}[leftmargin=*]
    \item First, \cite{he2023law} empirically showed that a {\em progressive NC} phenomenon governed by a law of data separation occurs from shallow to deep layers. Specifically, they observed that in trained over-parameterized nonlinear networks for classification problems, a metric of data separation decays at a geometric rate w.r.t. the number of layers. This is similar to our studied within-class feature compression at a geometric rate, albeit with different metrics (see the remark after Definition~\ref{def:nc1}). However, they do not provide any theoretical explanation for the progressive NC phenomenon. While our research focuses on DLNs and may not provide a complete understanding of phenomena in nonlinear networks, our theoretical analysis offers insights into the empirical behavior of deeper nonlinear layers. This is supported by our findings presented in \Cref{fig:intro1}, which demonstrate that linear layers can replicate the function of deep nonlinear layers in terms of feature learning. 
    
    \item Second, \cite{saxe2019mathematical} reveals that during training, nonlinear neural networks exhibit a hierarchical progressive differentiation of structure in their internal representations---a phenomenon they refer to as {\em progressive differentiation}. While both their study and ours reveal and justify a progressive separation phenomenon based on DLNs, the focus of each study is entirely different yet highly complimentary to each other. Specifically, we investigate how a trained neural network separates data according to their class membership from shallow to deep layers {\em after training}, while they investigate how weights of a neural network change w.r.t. training time and their impact on class differentiation {\em during training}.  
This distinction can be illustrated by the example in \cite{saxe2019mathematical}. Suppose a neural network is trained to classify eight items: sunfish, salmon, canary, robin, daisy, rose, oak, and pine. We study how these eight items are represented by the neural networks from shallow to deep layers in the trained neural network. In comparison, \cite{saxe2019mathematical} explain why animals versus plants are first distinguished at the initial stage of training, then birds versus fish, then trees versus flowers, and finally individual items, throughout the training process.
\end{itemize} 

\subsection{Notation and Paper Organization} 

\paragraph{Notation.} Let $\R^n$ be the $n$-dimensional Euclidean space and $\|\cdot\|$ be the Euclidean norm. Given a matrix $\bm A \in \R^{m \times n}$, we use $\bm a_i$ to denote its $i$-th column; we use $\|\bm A\|_F$ to denote the Frobenius norm of $\bm A$; we use $\sigma_{\max}(\bm A)$ (or $\|\bm A\|$), $\sigma_i(\bm A)$, and $\sigma_{\min}(\bm A)$ to denote the largest, the $i$-th largest, and the smallest singular values, respectively. 
 We use $\bm A^\dagger$ to denote the pseudo-inverse of a matrix $\bm A$. Given $L \in \mathbb{N}$, we use $[L]$ to denote the index set $\{1,\cdots,L\}$. 
Let $\mathcal{O}^n = \{\bm{Z} \in \R^{n \times n}: \bm{Z}^T\bm{Z} = \bI\}$ denote the set of all $n \times n$ orthogonal matrices.  We denote the Kronecker product by $\otimes$.  
Given weight matrices $\bW_1,\dots,\bW_L$, let $\bW_{l:1} := \bW_l\cdots\bW_1$ denote a matrix multiplication from $\bW_l$ to $\bW_1$  for all $l\in [L]$. 

\paragraph{Organization.} The rest of the paper is organized as follows. In \Cref{sec:prelim}, we introduce the basic problem setup. In \Cref{sec:main}, we present the main results and discuss their implications, with proofs provided in Section \ref{sec:proof} in the appendix. In \Cref{subsec:prior-arts}, we discuss the connections of our results to related works.
In \Cref{sec:exp}, we validate our theoretical claims and investigate nonlinear networks through numerical experiments. Finally, we conclude and discuss future directions in \Cref{sec:conclusion}. We defer all the auxiliary technical results to the appendix. 

\section{Preliminaries}\label{sec:prelim}

In this section,  we first formally introduce the problem of training DLNs for solving multi-class classification problems in \Cref{subsec:setup}, and then present the metrics for measuring within-class compression and between-class discrimination of features at each layer in \Cref{subsec:metrics}. 
 
\subsection{Problem Setup}\label{subsec:setup} 

\paragraph{Multi-class classification problem.} We consider a $K$-class classification problem with training samples and labels $\{(\bm{x}_{k,i},\bm{y}_k)\}_{ i \in [n_k], k \in [K] }$, where $\bx_{k,i} \in \mathbb R^{d}$ is the $i$-th sample in the $k$-th class, and $\bm{y}_k \in \mathbb R^K$ is an one-hot label vector with the $k$-th entry being $1$ and $0$ elsewhere. We denote by $n_k$ the number of samples in the $k$-th class for each $k \in [K]$. Here, we assume that the number of samples in each class is the same, i.e., $n_1 = \cdots = n_K = n$. This assumption is commonly used in recent studies on deep representation learning for classification problems \citep{zhu2021geometric,yaras2022neural,zhou2022optimization,zhou2022all}.
Moreover, we denote the total number of samples by $N = nK$. Without loss of generality, we arrange the training samples in a class-by-class manner such that 
\begin{align}\label{eq:X Y}
\bm X \;=\; \left[\bm x_{1,1},\ldots,\bm x_{1,n_1},\ldots,\bm x_{K,1},\ldots, \bm x_{K,n_K} \right] \in \R^{d\times N},\quad \bm Y \;=\; \bm{I}_K \otimes \bm{1}_n^T \in \R^{K\times N},
\end{align}
where $\otimes$ denotes the Kronecker product. For convenience, we also use $\bm x_i$ to denote the $i$-th column of $\bm X$.  

\paragraph{DLNs for classification problems.} In this work, we consider an $L$-layer ($L \geq 2$) linear network $ f_{\bm \Theta}(\cdot): \R^{d} \rightarrow \R^{K}$, parameterized by $\bm \Theta = \{\bW_l\}_{l=1}^L$ with input $\bm x \in \R^{d}$, i.e.,
\begin{equation}\label{eq:DLN}
    f_{\bm \Theta} (\bx) \;:=\; \bW_L \cdots \bW_1 \bx \;=\; \bW_{L:1}\bm x ,
\end{equation}
where $\bm W_1 \in \R^{d_1 \times d}$, $\bm W_l \in \R^{d_{l} \times d_{l-1}}$ for $l=2,\dots,L-1$, and $\bm W_L \in \R^{K \times d_{L-1}}$ are the weight matrices. The last-layer weight $\bm W_L$ is referred to as the linear {\em classifier} and the output of the $l$-th layer as the $l$-th layer \emph{feature} for all $l \in [L-1]$. As discussed in Section \ref{sec:intro}, DLNs are often used as prototypes for studying practical deep networks \citep{hardt2016identity,kawaguchi2016deep,laurent2018deep,lu2017depth}. We train an $L$-layer linear network to learn weights $\bm \Theta = \{\bW_l\}_{l=1}^L$ via minimizing the mean squared error (MSE) loss between $f_{\bm \Theta} (\bx)$ and $\bm y$ over the training data, i.e.,
\begin{align}\label{eq:obj}
\min_{\bm{\Theta}} \ell(\bm{\Theta}) 
\;=\; \frac{1}{2}\sum_{k=1}^K\sum_{i=1}^{n_k}\left\|  \bW_{L:1}\bx_{k,i} - \bm{y}_k \right\|^2 \;=\; \frac{1}{2}\| \bW_{L:1}\bm X - \bm Y\|_F^2. 
\end{align}
Before we proceed, we make some remarks on this problem. 
\begin{itemize}[leftmargin=*]
\item First, the cross-entropy (CE) loss is arguably the most popular loss function used to train neural networks for classification problems \citep{zhu2021geometric}. However, recent studies \citep{hui2020evaluation,zhou2022optimization} have demonstrated through extensive experiments that the MSE loss achieves performance comparable to or even superior to that of the CE loss across various tasks. 

\item Second, while DLNs may appear simple, the loss landscape of the objective function as defined in \eqref{eq:obj} is highly nonconvex, leading to highly nonlinear learning dynamics in gradient descent (GD) \citep{chen2025complete,achour2024loss}. Notably, the nonlinear GD dynamics in DLNs closely mirror that in their nonlinear counterparts \citep{lampinen2018analytic,saxe2019mathematical}. 

\item Finally, DLNs tend to be over-parameterized, with width of networks $d_l$ and feature dimensions $d$ exceeding the training samples $N$. In this setting, Problem \eqref{eq:obj} has infinitely many solutions that can achieve zero training loss, i.e., $\bm W_{L:1} \bm X = \bm Y$. Nevertheless, many studies have examined the convergence behavior of GD by closely examining its learning trajectory, revealing that GD---when initialized appropriately---exhibits an implicit bias towards {\em minimum norm solutions} \citep{bartlett2020benign,min2021explicit} with approximately {\em balanced} and {\em low-rank} weights \citep{arora2018convergence,min2021explicit}, which will be discussed further in \Cref{subsec:discuss}.
\end{itemize}

\subsection{The Metrics of Feature Compression and Discrimination}\label{subsec:metrics}

In this work, we focus on studying feature structures at each layer in a trained DLN. Given the weights $\bm \Theta = \{\bW_l\}_{l=1}^L$ satisfying $\bm W_{L:1} \bm X = \bm Y$, the weights $\{\bm W_l\}_{l=1}^L$ of the DLN transform the input data $\bm X$ into the membership matrix $\bm Y$ at the final layer. However, the hierarchical structure of the DLN prevents us from gaining insight into the underlying mechanism of how it transforms the input data into output from shallow to deep layers. To unravel this puzzle, we probe the features learned at intermediate layers. In our setting, we write the $l$-th layer's feature of an input sample $\bx_{k,i}$ as
\begin{align}\label{eq:zl}
\bm{z}_{k,i}^l = \bW_l \dots \bW_1 \bx_{k,i} = \bW_{l:1} \bm{x}_{k,i},\ \forall l = 1,\dots,L-1,
\end{align}
and we denote $\bz_{k,i}^0 = \bx_{k,i}$. For $l=0,1,\dots,L-1$, let $\bm{\Sigma}_W^l$ and $\bm{\Sigma}_B^l$ respectively denote the sum of the within-class and between-class covariance matrices for the $l$-th layer, i.e.,
\begin{align} 
   & \bm{\Sigma}_W^l := \frac{1}{N} \sum_{k=1}^K \sum_{i=1}^{n_k} \left( \bz_{k,i}^l - \bm \mu_k^l \right)\left( \bz_{k,i}^l - \bm \mu_k^l \right)^T,\label{eq:Sigma W} \\
   & \bm{\Sigma}_B^l := \frac{1}{N}\sum_{k=1}^K n_k\left(  \bm \mu_k^l - \bm \mu^l \right)\left( \bm \mu_k^l - \bm \mu^l \right)^T,  \label{eq:Sigma B}
\end{align}
where 
\begin{align}\label{eq:mu}
\bm \mu_k^l := \frac{1}{n_k} \sum_{i=1}^{n_k} \bz_{k,i}^l,\quad \bm \mu^l := \frac{1}{K} \sum_{k=1}^{K} \bm \mu_k^l
\end{align}
denote the mean of the $l$-th layer's features in the $k$-th class and the global mean of the $l$-th layer's features, respectively. 
Equipped with the above setup, we can measure the compression of features within the same class and the discrimination of features between different classes using the following metrics.    

\begin{definition}[Intermediate layer-wise feature compression and discrimination]\label{def:nc1}
For all $l = 0,1,\dots, L-1$, we say that 
\begin{align}\label{eq:nc1}
\quad C_l = \frac{\mathrm{Tr}(\bm{\Sigma}_W^l)}{\mathrm{Tr}(\bm{\Sigma}_B^l)}\quad \text{and}\quad D_l = 1 - \max_{k \neq k^\prime} \frac{ \langle \bm \mu_k^l, \bm \mu_{k^\prime}^l\rangle }{\|\bm \mu_k^l\|\|\bm \mu_{k^\prime}^l\|} 
\end{align}
are the metrics of within-class compression and between-class discrimination of intermediate features at the $l$-th layer, respectively. 
\end{definition}
From now on, we will use these two metrics to study the evolution of features across layers. Intuitively, the features in the same class at the $l$-th layer are more compressed if $C_l$ decreases, while the features from different classes are more discriminative if $D_l$ increases. Before we proceed, let us delve deeper into the rationale behind each metric.  

\paragraph{Discussion on the metric of feature compression.}  The study of feature compression has recently caught great attention in both supervised \citep{yu2020learning,fang2021exploring,papyan2020prevalence} and unsupervised \citep{shwartz2023compress} deep learning. For our definition of feature compression in \eqref{eq:nc1}, the numerator $\mathrm{Tr}(\bm{\Sigma}_W^l)$ of the metric $C_l$ measures how well the features from the same class are compressed towards the class mean at the $l$-th layer. More precisely, the features of each class are more compressed around their respective means as $\mathrm{Tr}(\bm{\Sigma}_W^l)$ decreases. The denominator $\mathrm{Tr}(\bm{\Sigma}_B^l)$ prevents reporting spuriously small values of $C_l$ in near-collapse cases where all features approach a single mean, while simultaneously serving as a normalization factor that renders the metric invariant to feature scaling. Specifically, given some weights $\bm W_1,\dots,\bm W_l$ with $C_l = \mathrm{Tr}(\bm \Sigma_W^l)/\mathrm{Tr}(\bm \Sigma_B^l)$, if we scale them to $t\bm W_1,\dots,t\bm W_l$ for some $t > 0$, then $C_l$ does not change, while the corresponding numerator becomes $t^l\mathrm{Tr}(\bm{\Sigma}_W^l)$. It should be noted that this metric and similar ones have been studied in recent works \citep{kothapalli2023neural,rangamani2023feature,tirer2023perturbation,yaras2023law}. For instance, \cite{tirer2023perturbation} employed this metric to measure the variability of within-class features to simplify theoretical analysis. Moreover, a similar metric $\mathrm{Tr}(\bm{\Sigma}_W^l \bm{\Sigma}_B^{l^\dagger})$ has been used to characterize within-class variability collapse in recent studies on neural collapse in terms of last-layer features \citep{fang2021exploring,papyan2020prevalence,rangamani2023feature,yaras2022neural, zhu2021geometric}. Our studied metric $C_l$ can be viewed as its simplification.  Additionally, \cite{he2023law} employed $\mathrm{Tr}(\bm{\Sigma}_W^l \bm{\Sigma}_B^{l^\dagger})$ to measure how well the data are separated across intermediate layers. Due to their similarity, $C_l$ can also serve as a metric for measuring data separation.   


\paragraph{Discussion on the metric of feature discrimination.} It is worth pointing out that learning discriminative features has a long history, tracing back to unsupervised dictionary learning \citep{donoho2003optimally,arora2014new,sun2016complete,Qu2020Geometric,zhai2020complete}. For our definition of feature discrimination, $\arccos\left( \langle \bm \mu_k^l, \bm \mu_{k^\prime}^l\rangle /(\|\bm \mu_k^l\|\|\bm \mu_{k^\prime}^l\|)\right)$ computes the angle between class means $\bm\mu_k^l$ and $\bm\mu_{k^\prime}^l$.
This, together with  \eqref{eq:nc1}, indicates that $D_l$ measures the feature discrimination by calculating the smallest angles among feature means of all pairs. Moreover, we can equivalently rewrite $D_l$ in \eqref{eq:nc1} as
\begin{align*}
D_l = \frac{1}{2}\min_{k\neq k^\prime} \left\| \frac{\bm \mu_k^l}{\|\bm \mu_k^l\|} - \frac{\bm \mu_{k^\prime}^l}{\|\bm \mu_{k^\prime}^l\|} \right\|^2. 
\end{align*}
This indicates that $D_l$ computes the smallest distance between normalized feature means. According to these two interpretations, features between classes become more discriminative as $D_l$ increases. 
Recently, \cite{masarczyk2023tunnel} considered a variant of the inter-class variance $\sum_{k=1}^K \sum_{k^\prime \neq k}\|\bm \mu_k - \bm \mu_{k'}\|^2 $ to measure linear separability of representations of deep networks.  

\section{Main Results}\label{sec:main}


In this section, we present our main theoretical results based on \Cref{sec:prelim}, first describing the theorem in \Cref{subsec:thm} and then discussing the assumptions in \Cref{subsec:discuss}.

\subsection{Main Theorem}\label{subsec:thm}

Before we present the main theorem, we make the following assumptions on the input data and the weights of the DLN in \eqref{eq:DLN}.  

\begin{assum}\label{AS:1}
For the data matrix $\bm X \in \R^{d\times N}$, the data dimension is no smaller than the number of samples, i.e., $d \ge N$. Moreover, the data is $\theta$-nearly orthonormal, i.e., there exists an $\theta \in [0,1/4)$ such that 
\begin{align}\label{eq:orth}
\left|\|\bm x_{i}\|^2 - 1 \right| \le \frac{\theta}{N},\    \left|\langle \bm x_{i}, \bm x_{j}\rangle \right| \le \frac{\theta}{N},\ \text{for all}\ 1 \le i \neq j \le N, 
\end{align}
where $\bm x_i$ denotes the $i$-th column of $\bm X$. 
\end{assum}

\paragraph{Discussion on Assumption \ref{AS:1}.} We make this assumption primarily to simplify our analysis. It can be relaxed with a more refined analysis, and in practice, it may even be violated in empirical scenarios. Here, the condition $d \ge N$ guarantees that $\bm X$ is linearly separable in the sense that there exists a linear classifier $\bm W \in \R^{K\times d}$ such that $\bm W \bm x_{k,i} = \bm y_k$ for all $i,k$. We should point out that the same condition has been studied in \citep{chatterji2023deep,chatterji2022interplay,frei2023implicit}, and similar linear separability conditions have been widely used for studying implicit bias of gradient descent \citep{nacson2019convergence,phuong2020inductive,soudry2018implicit}. Notably, this condition also holds for nonlinear networks in the sense that the intermediate features generated by the initial layers exhibit linear separability as shown in \cite{alain2016understanding,ansuini2019intrinsic,masarczyk2023tunnel,recanatesi2019dimensionality} (see Figure \ref{fig:intro1}, where the near-perfect accuracy at intermediate layers shows that the features at that layer are already linearly separable.).  In addition, nearly orthonormal data is commonly used in the theoretical analysis of learning dynamics for training neural networks; see, e.g., \cite{boursier2022gradient,frei2023implicit,phuong2020inductive}. In particular, this condition holds with high probability for well-conditioned Gaussian distributions \citep{frei2023implicit}, and it generally applies to a broad class of subgaussian distributions, as demonstrated in Claim 3.1 of \cite{hu2020surprising}.

\paragraph{Implicit bias of GD.} Since the DLN in Problem \eqref{eq:obj} is over-parameterized, it has infinitely many solutions satisfying $\bm W_{L:1}\bm X = \bm Y$. However, GD for training networks typically has an implicit bias towards certain solutions with benign properties \citep{arora2019implicit,Gunasekar2017,Ji2019,min2021explicit,Shah2020,soudry2018implicit}. In particular, prior work has demonstrated that, with assumptions on network initialization and the dataset, gradient flow tends to favor solutions with minimum norms and balanced weights; see, e.g., \cite{min2021explicit,chatterji2023deep,arora2018optimization,du2018algorithmic}. Recent studies also reveal that GD primarily updates a minimal invariant subspace of the weight matrices, thereby preserving the approximate low-rankness of the weights across all layers \citep{huh2021low,yaras2023law}. Based on these findings, we assume that the trained weights $\bm \Theta$ satisfy the following benign properties to investigate how trained deep networks hierarchically transform input data into labels.  

\begin{assum}\label{AS:2}
For an $L$-layer DLN with weights $\bm{\Theta} = \{\bW_l\}_{l=1}^L$ described in \eqref{eq:DLN} with $d_l = d > 2K$ for all $l \in [L-1]$, the weights $\bm{\Theta}$ satisfy \\ 
(i) {\em Minimum-norm solution:}
\begin{align}\label{eq:min norm}
\bW_{L:1} = \bY(\bm X^T\bm X)^{-1}\bX^T. 
\end{align} 
(ii) {\em $\delta$-Balancedness:} There exists a constant $\delta > 0$ such that 
\begin{align}\label{eq:bala}
    \bW_{l+1}^T\bW_{l+1} = \bW_l\bW_l^T, \forall l \in [L-2],\ \|\bW_{L}^T\bW_{L} - \bW_{L-1}\bW_{L-1}^T\|_F \le   \delta. 
\end{align} 
(iii) {\em $\varepsilon$-Approximate low-rank weights:} There exist positive constants  $\varepsilon \in (0,1)$ and $\rho \in [0,\varepsilon)$ such that for all $l \in [L-1]$, 
\begin{align}\label{eq:singu}
 \varepsilon - \rho \le \sigma_i(\bW_l) \le \varepsilon,\ \forall i=K+1,\dots,d-K. 
\end{align}
\end{assum}
We defer the discussion of Assumption~\ref{AS:2} to \Cref{subsec:discuss}, where we provide both theoretical and empirical evidence in support of it.  It is worth noting that Assumption~\ref{AS:2} serves as an initial framework for understanding progressive feature behavior and can be relaxed to more general conditions. For example, a follow-up work \cite[Theorem 3.1]{jacot2025wide} proves related but narrower results under weakened versions of the conditions in Assumption \ref{AS:2} (e.g., \eqref{eq:bala} holds only approximately). Building upon the above assumptions, we now present our main theorem on hierarchical representations in terms of feature compression and discrimination.

\begin{theorem}\label{thm:NC}
Consider a $K$-class classification problem on the training data $(\bm X,\bm Y) \in \mathbb R^{d \times N} \times \R^{K\times N}$, where the matrix $\bm X$ satisfies Assumption \ref{AS:1} with parameter $\theta$. Suppose that we train an $L$-layer DLN with weights $\bm \Theta = \left\{ \bm W_l \right\}_{l=1}^L$ such that $\bm \Theta$ satisfies Assumption \ref{AS:2}, with the parameters $(\delta,\rho,\varepsilon)$ of weight balancedness and low-rankness satisfying 
\begin{align}\label{eq:delta2}
\delta \le \min\left\{ \frac{(2n)^{1/L}}{30L^2},  \frac{n^{1/L}}{128\sqrt{K}}, \frac{1}{16\sqrt{K}}\right\},\ \varepsilon \le \min\left\{\frac{n^{\frac{1}{2L}}}{4}, 1 \right\},\ \rho \le \frac{\varepsilon}{2L \sqrt{n}}. 
\end{align} 
(i) {\em Progressive within-class feature compression:} For $C_{l}$ in Definition \ref{def:nc1}, it holds that
\begin{align}
&  \frac{c\varepsilon^2 }{\kappa(2n)^{1/L}} \le \frac{C_1}{C_0} \le \frac{2\kappa\varepsilon^2}{c(n/2)^{1/L}}, \label{eq1:compres}\\
&  \frac{c\varepsilon^2 }{\kappa(2n)^{1/L}} \le \frac{C_{l+1}}{C_{l}} \le  \frac{\kappa\varepsilon^2}{c(n/2)^{1/L}},\ \forall l \in [L-2],\label{eq2:compres}
\end{align}
where  
\begin{align*}
c = \frac{(n-3)K-1}{(n-1)K+1},\ \kappa = \frac{1+n^{-\Omega(1)}}{1-n^{-\Omega(1)}}. 
\end{align*}
(ii) {\em Progressive between-class feature discrimination:} For $D_{l}$ in Definition \ref{def:nc1}, for all $l \in [L-1]$, we have 
\begin{align}\label{eq:discri}
D_l \ge 1 - 32 \left( \theta + 4\delta \right) \left(2-\frac{l+1}{L}\right) - n^{-\Omega(1)}. 
\end{align}
\end{theorem} 
We defer the detailed proof to \Cref{sec:proof}. As we discussed in \Cref{sec:intro}, numerous empirical studies have been conducted to explore the feature structures of intermediate and final layers in deep networks, particularly concerning feature compression and discrimination; see, e.g., \cite{alain2016understanding, ansuini2019intrinsic, ben2022nearest, chen2022layer, galanti2022implicit, he2023law, hui2022limitations, masarczyk2023tunnel, rangamani2023feature, recanatesi2019dimensionality, zhang2022all}. However, there remains a scarcity of theoretical analysis to elucidate their observations. Despite the acknowledged limitations of DLNs, our work takes the first step towards rigorously developing a unified framework to quantify feature compression and discrimination across different layers. Specifically, Theorem~\ref{thm:NC} shows that, given input data satisfying Assumption~\ref{AS:1} and a trained DLN with weights satisfying Assumption~\ref{AS:2}, features within the same class are compressed at a {\em geometric rate} on the order of ${\varepsilon^2}/{n^{1/L}}$ when $n$ is sufficiently large, while the features between classes are discriminated at a {\em linear rate} on the order of $1/L$ across layers. In \Cref{fig:progressive NC}, we also show empirical support for these findings in \Cref{fig:progressive NC} (top figures), where we simulate data and train a linear network according to our assumptions and plot the metrics at each layer. This phenomenon also appears in nonlinear networks, as shown in  \Cref{fig:progressive NC} (bottom figures), where we train a nonlinear MLP with the same synthetic training data and plot the metrics at each layer. In the following, we discuss the implications of our main result.

\begin{figure}[t]
\begin{center}
    \begin{subfigure}{1.0\textwidth}
        \centering
        \includegraphics[trim = 8 8 8 8 ,width=.4\textwidth]{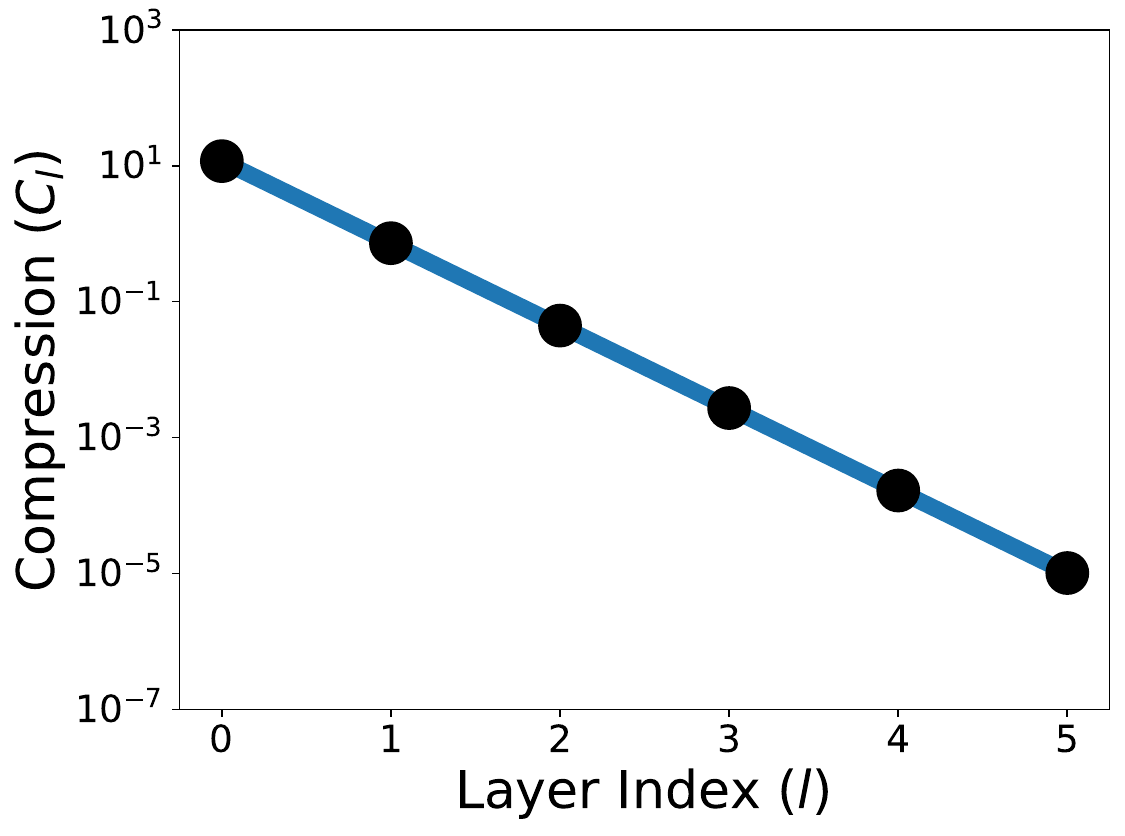}
        \hspace{0.5in}
        \includegraphics[trim = 8 8 8 8,width=.4\textwidth]{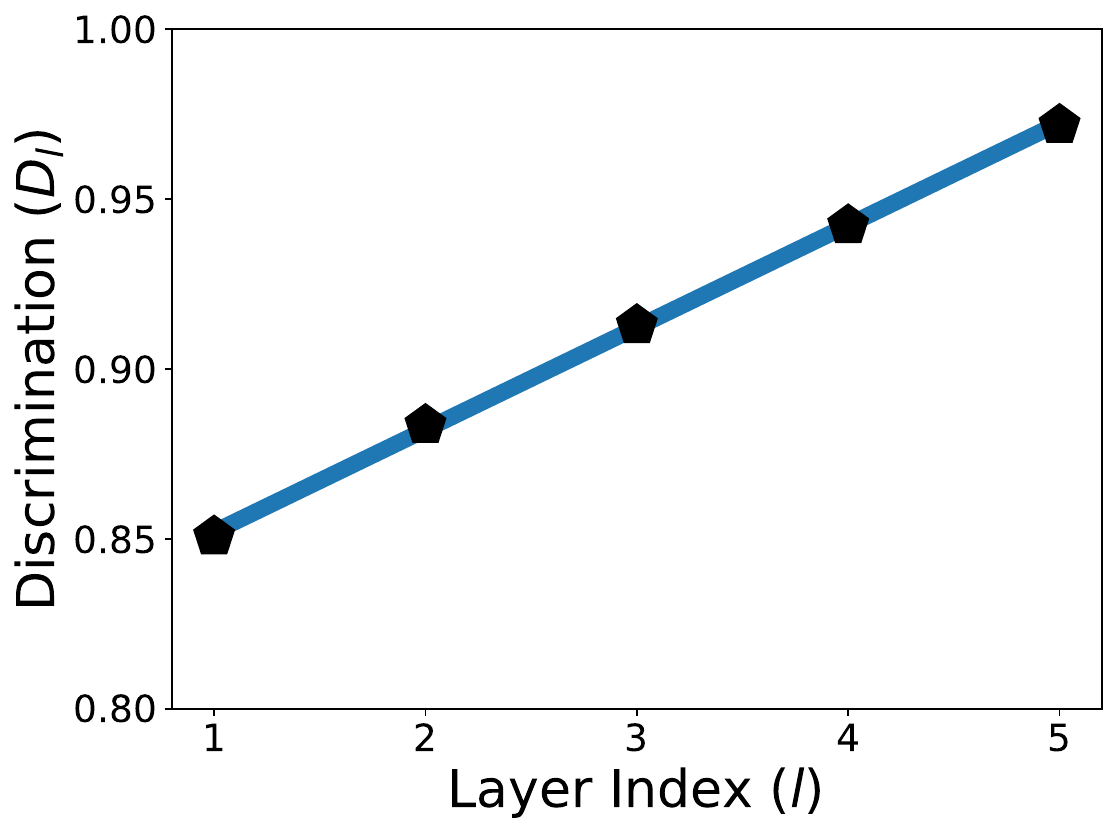}
        \caption{A $6$-layer \textbf{linear} network} 
    \end{subfigure} 

    \vspace{0.1in}
    \begin{subfigure}{1.0\textwidth}
        \centering
        \includegraphics[trim = 8 8 8 8 ,width=.4\textwidth]{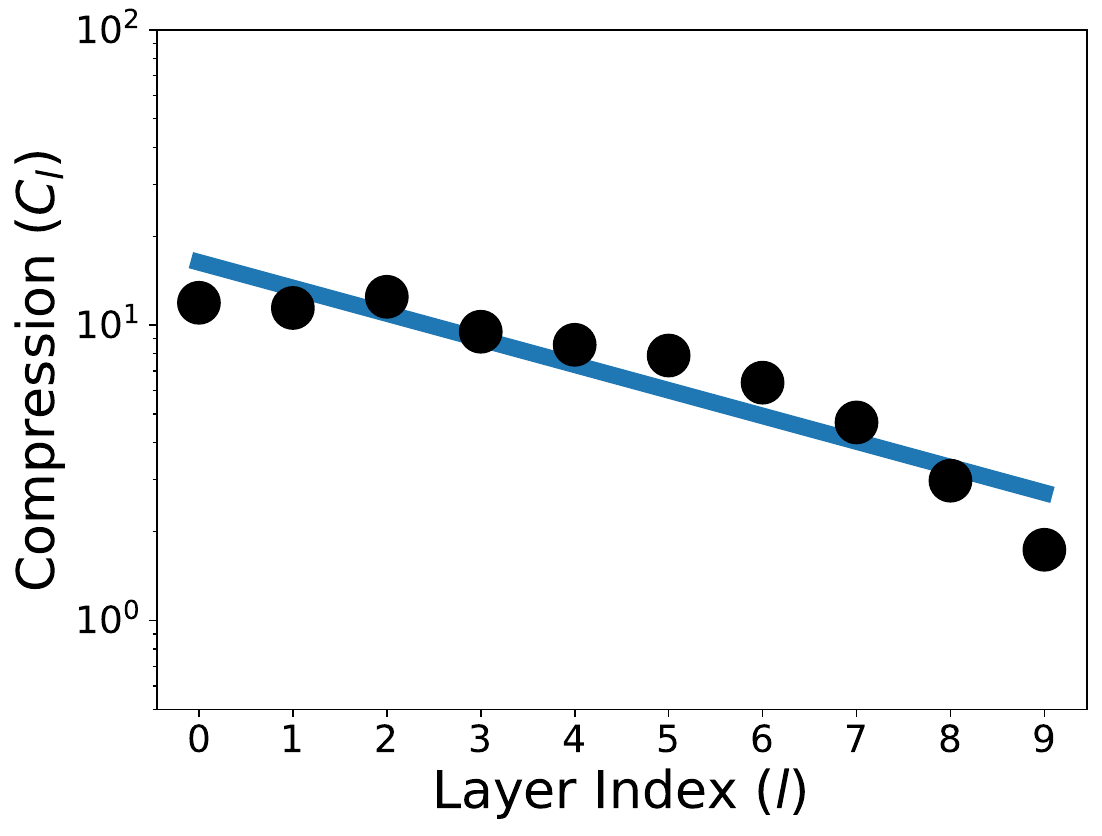}
        \hspace{0.5in} 
        \includegraphics[trim = 8 8 8 8,width=.4\textwidth]{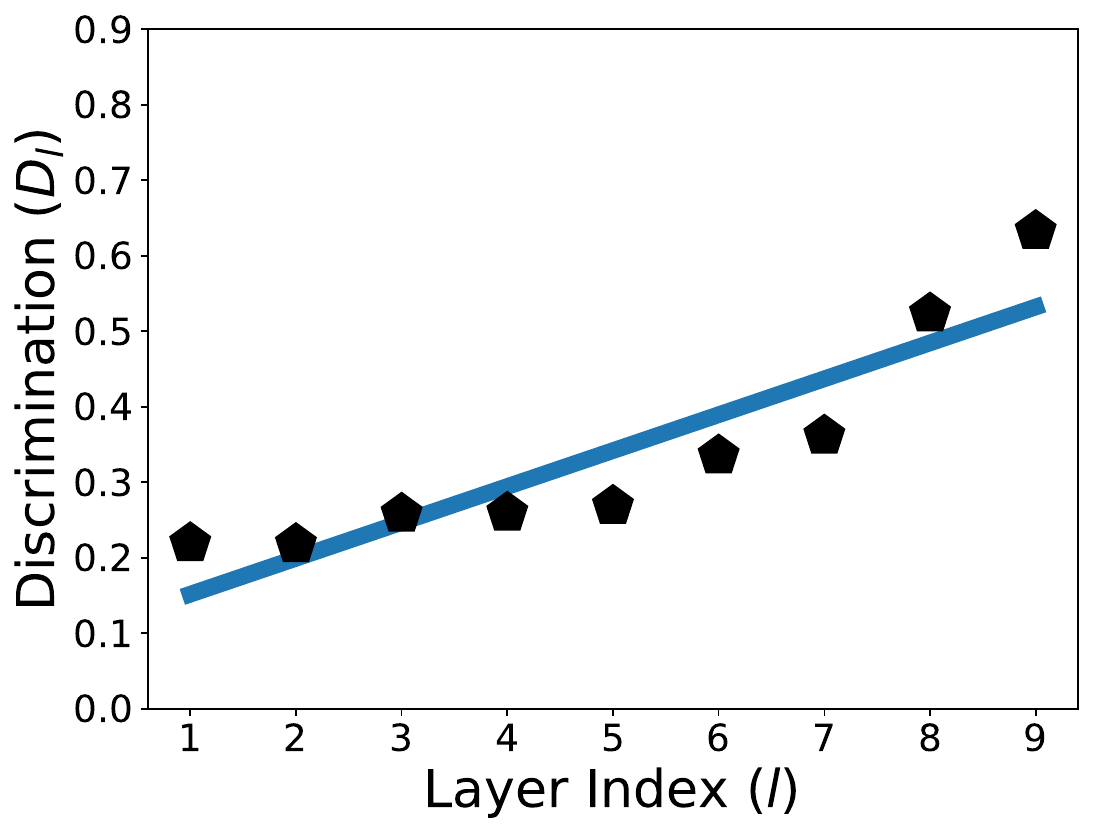}
        \caption{A $10$-layer \textbf{nonlinear} network} 
    \end{subfigure} 
\end{center}\vspace{-0.1in} 
\caption{\textbf{Progressive feature compression and discrimination on both linear and nonlinear networks.} We plot the feature compression and discrimination metrics defined in \eqref{eq:nc1} for $l=1,\dots,L-1$ on both the linear network (top row) and nonlinear network (bottom row). We train both networks using a nearly orthogonal dataset as described in Assumption~\ref{AS:1}, initializing the network weights satisfying \eqref{eq:init}, with an initialization scale of $\xi=0.3$. We train both networks via gradient descent until convergence. In each figure, the $x$-axis denotes the number of layers from shallow to deep, with layer-0 denoting the inputs. In the left figures, the $y$-axis denotes the compression measure $C_l$ in the logarithmic scale; In the right figures, the $y$-axis denotes the discrimination measure $D_l$. More experimental details can be found in \Cref{subsec:exp-thm}.}
\label{fig:progressive NC}  
\end{figure} 

\paragraph{From linear to nonlinear networks.} Although our result is rooted in DLNs, it provides valuable insights into the feature evolution in nonlinear networks. Specifically, the linear separability of features learned by initial layers in nonlinear networks (see \Cref{fig:intro1}) allows later layers to be effectively replaced by linear counterparts for compressing within-class features and discriminating between-class features. We empirically support this claim in \Cref{fig:intro2}, where we observe that both the linear and subsequent nonlinear layers exhibit consistent trends of feature compression and discrimination with respect to depth.
Therefore, studying DLNs helps us understand the role of the nonlinear layers after the initial layers in nonlinear networks for learning features. This understanding also sheds light on the pattern where within-class features compress at a geometric rate and between-class features discriminate at a linear rate in nonlinear networks, as illustrated in \Cref{fig:progressive NC} (bottom figures). A natural direction is to extend our current analysis framework to nonlinear networks, especially homogeneous neural networks \citep{lyu2019gradient}.

\begin{figure*}[t]
\begin{center}
       \begin{subfigure}{0.48\textwidth}
    \includegraphics[width = 0.952\linewidth]{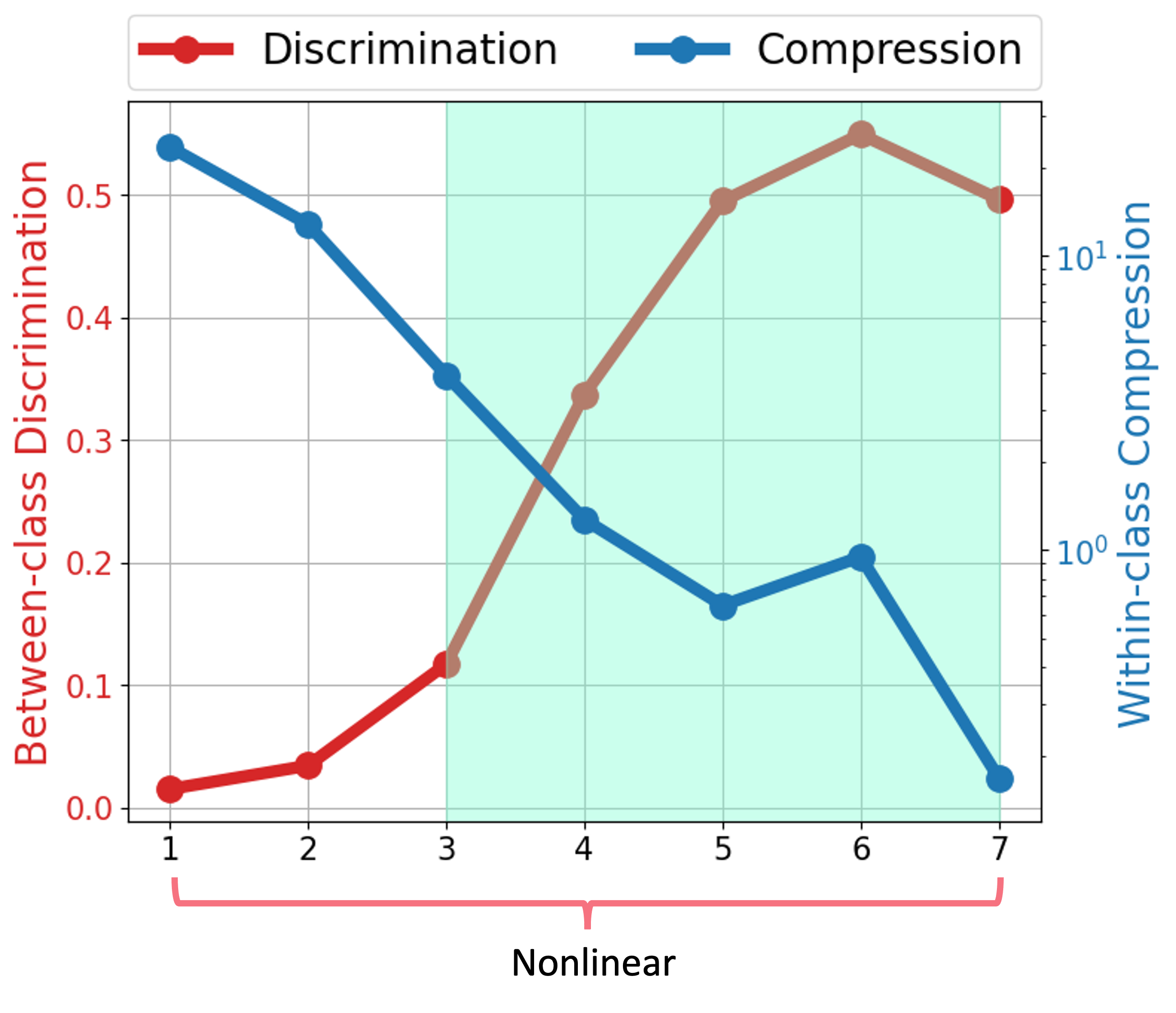}
    \caption{A 8-layer MLP network with ReLU} 
    \end{subfigure} 
    \hfill 
    \begin{subfigure}{0.47\textwidth}
    \includegraphics[width = 0.95\linewidth]{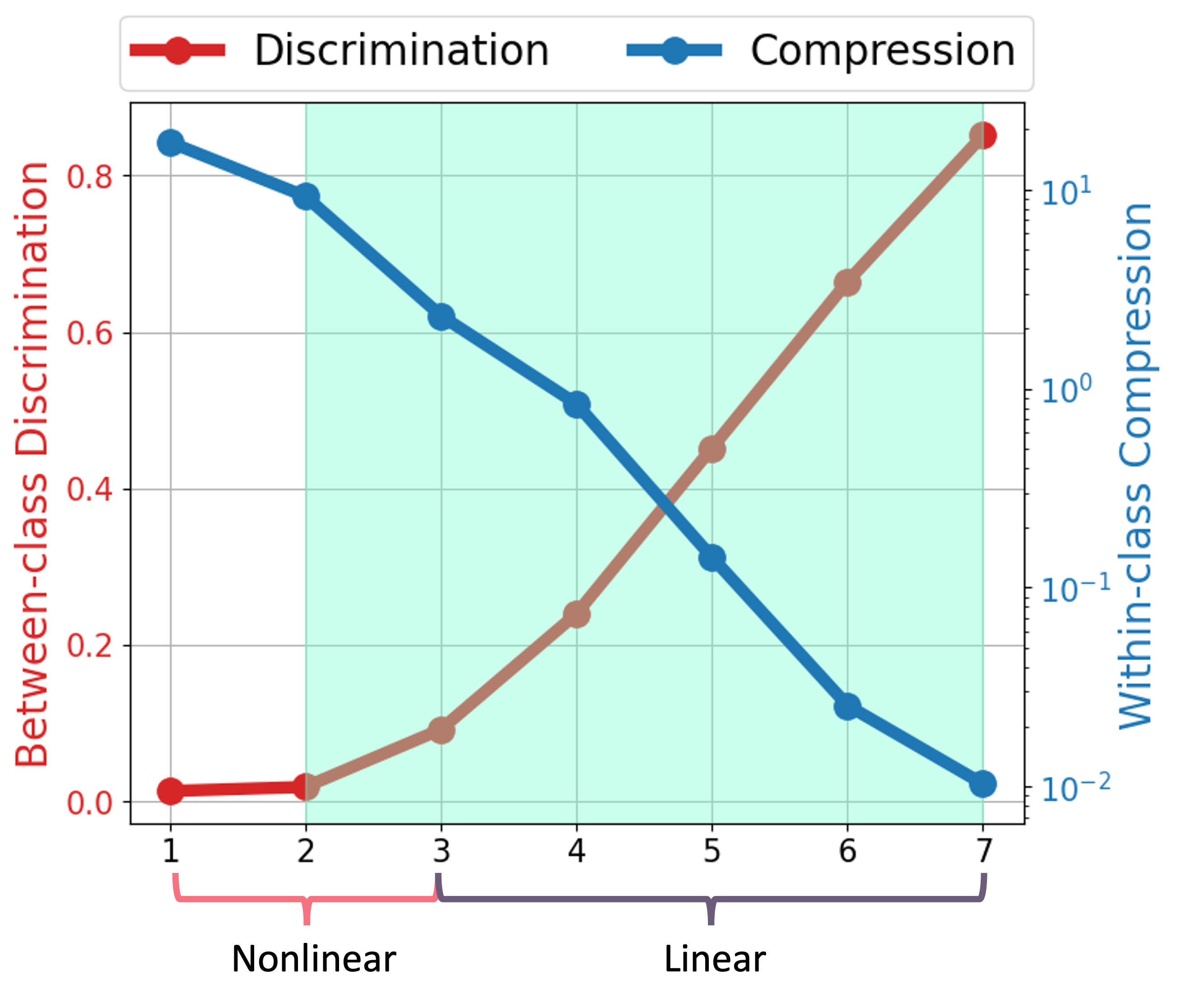}
    \caption{A 8-layer hybrid network} 
    \end{subfigure} 
\end{center}\vspace{-0.1in} 
    \caption{\textbf{Within-class compression and between-class discrimination of features of two 8-layer networks.} We use the same setup as in Figure \ref{fig:intro1} and plot the metrics of within-class compression and between-class discrimination (see Definition \ref{def:nc1}) of the two networks across layers after training, respectively. Additional experimental details are deferred to \Cref{subsubsec:resemb_1}.} 
    \label{fig:intro2}
\end{figure*}

\paragraph{Neural collapse beyond the unconstrained feature model.} One important implication of our result is that it addresses an open problem about NC. Specifically, almost all existing works assume the {\em unconstrained features model} \citep{Mixon2020,papyan2020prevalence} to analyze the NC phenomenon, where the last-layer features of the deep network are treated as free optimization variables to simplify interactions across layers; see, e.g., \cite{zhu2021geometric,wang2022linear,zhou2022all,yaras2022neural,zhou2022optimization,li2023neural}. However, a major drawback of this model is that it overlooks the hierarchical structure of deep networks as well as the structure of the input data. 
In this work, we address this issue without assuming the unconstrained feature model. Specifically, according to Theorem~\ref{thm:NC}, for a sufficiently deep linear network trained on nearly orthogonal data, the last-layer features within the same class concentrate around their class means, while the last-layer features from different classes are nearly orthogonal to each other. This directly implies that the last-layer features exhibit variability collapse and convergence to simplex equiangular right frame approximately after centralization. 

Since the submission of our manuscript, several follow-up works have investigated NC beyond the unconstrained feature model in different settings; see, e.g., \cite{hong2024beyond,wang2024progressive,jacot2025wide,kothapalli2024kernel,xu2023dynamics}. In particular, \cite{hong2024beyond} studied 2-layer and 3-layer ReLU neural networks and investigated the impact of network width, depth, data dimension, and statistical properties of input data on NC. \cite{wang2024progressive} characterized the geometric properties of intermediate layers of ResNet and proposed a conjecture on progressive feedforward collapse. Recently, \cite{jacot2025wide} studied a hybrid network with nonlinear layers followed by at least two linear layers and provided generic guarantees on progressive NC. It is worth noting that their derived conditions extend those in Assumption~\ref{AS:2} and are slightly more general. 

\paragraph{Guidance on network architecture design.} Progressive feature compression and discrimination in Theorem~\ref{thm:NC} provide guiding principles for network architecture design. Specifically, according to \eqref{eq1:compres}, \eqref{eq2:compres}, and \eqref{eq:discri}, features are more compressed within the same class and more discriminated between different classes, improving the separability of input data as the depth of the network increases. This is also supported by our experiments in \Cref{fig:test-init-dataset}. This indicates that the network should be deep enough for effective data separation in classification problems. However, it is worth noting that the belief that deeper networks are better is not always true. Indeed, it becomes increasingly challenging to train a neural network as it gets deeper, especially for DLNs \citep{glorot2010understanding}. Moreover, over-compressed features may degrade the out-of-distribution performance of deep networks as shown in \cite{masarczyk2023tunnel}. This, together with our result, indicates that neural networks should not be too deep for improved out-of-distribution generalization performance. 

\paragraph{Understanding projection heads for transfer learning.} In contrastive learning \citep{chen2020simple,chen2021exploring}, a successful empirical approach to improving transfer learning performance in downstream tasks involves the use of projection heads \citep{li2022principled}. These projection heads, typically consisting of one or several MLP layers, are added between the feature extractor and the final classifier layers during pre-training. For downstream tasks, the projection head is discarded and only the features learned by the feature extractor are utilized for transfer learning. Recent works \citep{li2022principled,kornblith2021why,galanti2022on,xie2022hidden} have established an empirical correlation between the degree of feature collapse during pre-training and downstream performance: less feature collapse leads to more transferable models. However, it remains unclear why features prior to projection heads exhibit less collapse and greater transferability. Our study addresses this question and offers a theoretical insight into the utilization of projection heads. According to Theorem~\ref{thm:NC}, it becomes apparent that features from the final layers tend to be more collapsed than the features before the projection head. This, together with the empirical correlation between feature collapse and transferability, implies that leveraging features before projection heads improves transferability. 
We provide further empirical evidence with experiments in \Cref{subsec:TL}. Moreover, the progressive compression pattern in Theorem~\ref{thm:NC} also provides insight into the phenomenon studied in \cite{yosinski2014transferable}, which suggests that deeper layers in a neural network become excessively specialized for the pre-training task, consequently limiting their effectiveness in transfer learning.

\subsection{Discussions on Assumption~\ref{AS:2}}\label{subsec:discuss}
 
In this subsection, we justify the properties of trained network weights in Assumption~\ref{AS:2} using both theoretical and experimental findings. While the solution described in Assumption~\ref{AS:2} is only one of infinitely many global optima for Problem \eqref{eq:obj}, the implicit bias of GD ensures that, with proper initialization, iterations almost always converge to the desired solution. In the following discussion, we delve into this phenomenon in greater depth. 

\paragraph{GD with orthogonal initialization.} We consider training a DLN for solving Problem \eqref{eq:obj} by GD, i.e., for all $l \in [L]$,
\begin{align}\label{eq:gd}
\bm W_l{(t+1)} = \bm W_l{(t)} - \eta \frac{\partial \ell(\bm \Theta{(t)})}{\partial \bm W_l},\ t \ge 0. 
\end{align}
Notably, when the learning rate is infinitesimally small, i.e., $\eta \to 0$, GD in \eqref{eq:gd} reduces to gradient flow. Moreover, we initialize the weight matrices $\bW_l$ for all $l \in [L]$ using $\xi$-scaled orthogonal matrices for a constant $\xi > 0$, i.e., 
\begin{align}\label{eq:init}
   \bm W_l{(0)^T} \bm W_l{(0)} = \xi^2 \bm I_d, \forall l \in [L-1], \quad \bm W_L{(0)} = \left[\xi\bm U\ \bm 0 \right],
\end{align}
where $\bm U \in \mathcal{O}^K$. It is worth noting that orthogonal initialization is widely used to train deep networks, which can speed up the convergence of GD; see, e.g., \cite{pennington2018emergence,xiao2018dynamical,Hu2020}.  

\paragraph{Theoretical justification of Assumption~\ref{AS:2}.} Theoretically, we can prove the conditions outlined in Assumption \ref{AS:2} when the gradient flow is trained on a square and orthogonal data matrix using the results in \cite{arora2018optimization,yaras2023law,yaras2024compressible,kwon2023compressing}.  

\begin{proposition}\label{prop:AS}
Suppose that the data matrix $\bm X \in \R^{d\times N}$ is square, i.e., $d=N$, and satisfies Assumption \ref{AS:1} with $\theta = 0$. Suppose in addition that we apply GD \eqref{eq:gd} to solve Problem \eqref{eq:obj} with $\eta \to 0$ and the initialization in \eqref{eq:init}. If GD converges to a global optimal solution, then \eqref{eq:min norm}, \eqref{eq:bala}, and \eqref{eq:singu} hold with $\delta =  \xi^2\sqrt{d-K}$, $\varepsilon = \xi$, and $\rho=0$. 
\end{proposition}
The proof of this proposition is deferred to \Cref{subsec:pf prop}. Although we can only prove Assumption~\ref{AS:2} in this restrictive setting, we see empirically that it seems to hold in more general settings. We substantiate this claim by drawing support from empirical evidence and existing findings in the literature, as elaborated in the following.

\paragraph{Empirical justifications of Assumption~\ref{AS:2}.} Here, we run GD \eqref{eq:gd} using the initiation \eqref{eq:init} on nearly orthogonal data satisfying Assumption~\ref{AS:1}. We refer the reader to \Cref{subsec:theory_and_assump} for the experimental setup. After training, we plot the metrics of minimum norm residual, balancedness residual, and variance of singular values (see the legends in \Cref{fig:assump_2_all})  against depth (resp. width) in  \Cref{fig:assump_2_all}(a) (resp. \Cref{fig:assump_2_all}(b)). It is observed that the magnitudes of these metrics are very small for different depths and widths of neural networks. This indicates that Assumption~\ref{AS:2} approximately holds for GD in a broader setting. 
 
\begin{figure*}[t]
    \begin{subfigure}{0.5\textwidth}
    \includegraphics[width = 0.9\linewidth]{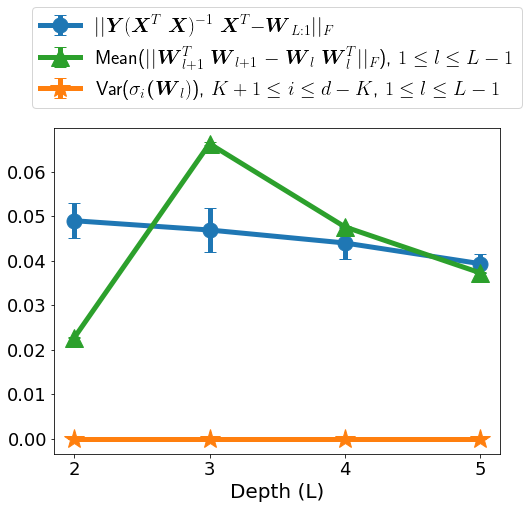}
    \caption{Fix width $d=400$, change depth $L$} 
    \end{subfigure} 
    \begin{subfigure}{0.5\textwidth}
    \includegraphics[width = 0.9\textwidth]{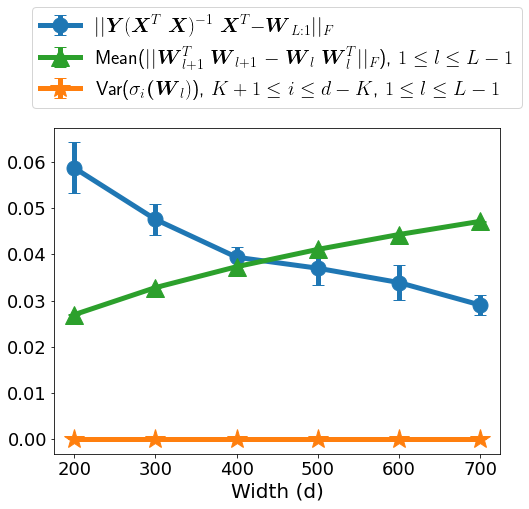}
    \caption{Fix depth $L=5$, change width $d$}
    \end{subfigure} 
    \caption{\textbf{Assumption~\ref{AS:2} holds approximately among different network configurations.} We train DLNs using the orthogonal initialization with varying widths and depths. For each network width $d$ and depth $L$, we show the minimum-norm residual, the average weight balancedness residual of all weights, and the variance of the singular values of all weights (see \Cref{metric:bias} for details) for the DLNs that have been trained until convergence. Each data point represents the mean value of 10 runs with different random seeds, with error bars indicating standard deviation. Here, we plot these metrics in (a) with fixed network width $d$ and varying depth $L$ and (b) with fixed network depth $L=5$ and varying network width $d$. More experimental details are deferred to \Cref{subsec:exp-assump}.}
    \label{fig:assump_2_all}
\end{figure*}

\paragraph{Prior arts on implicit bias support Assumption~\ref{AS:2}.} Moreover, Assumption~\ref{AS:2} is also well supported by many related results in the literature.
\begin{itemize}[leftmargin=*]
    \item \textbf{Minimum-norm solutions and balanced weights.} It has been extensively studied in the literature that GD with proper initialization converges to minimum-norm and balanced solutions, especially in the setting of gradient flow \citep{arora2018optimization,chatterji2023deep,du2018algorithmic,min2021explicit}. Specifically, \cite{du2018algorithmic,arora2018optimization} proved that for linear networks, the iterates of gradient flow satisfy
\begin{align}\label{eq:gf}
\frac{d}{dt}\left( \bm W_{l}(t)\bm W_{l}(t)^T - \bm W_{l+1}(t)^T\bm W_{l+1}(t) \right) = \bm 0,\ \forall l \in [L-1],
\end{align}
where $t$ denotes a continuous time index. This, together with \eqref{eq:init}, implies that \eqref{eq:bala} holds with $\delta=\xi\sqrt{d-K}$. Moreover, using the result in \cite{min2021explicit}, if we initialize the weights of two-layer networks as in \eqref{eq:init} with $\bm W_2(0)\bm W_1(0) \in \mathrm{span}(\bm X)$, gradient flow always yields the minimum-norm solution upon convergence. However, the conditions \eqref{eq:min norm} and \eqref{eq:bala} are rarely studied in the context of GD for training deep networks. In this context, we empirically verify these two conditions as shown in \Cref{fig:assump_2_all} under different network configurations.
\item \textbf{Approximate low-rankness of the weights.}  Recently, numerous studies have demonstrated that GD exhibits a bias towards low-rank weights \citep{arora2019implicit,Gunasekar2017,yaras2023law,kwon2023compressing}. In particular, \cite{yaras2023law,kwon2023compressing} showed that when the input data $\bm X$ is orthonormal, learning dynamics of GD for DLNs in \eqref{eq:gd}, with the initialization in \eqref{eq:init}, only updates an invariant subspace of dimension $2K$ for each weight matrix across all layers, where the invariant subspace is spanned by the singular vectors associated with the $K$-largest and $K$-smallest singular values, and the rest singular values remain unchanged. More experimental demonstration can be found in \Cref{app:extra-exp}. 
\end{itemize}

\section{Relationship to Prior Arts}\label{subsec:prior-arts}

In this section, we discuss the relationship between our results and prior works on the empirical and theoretical study of deep networks.  

\paragraph{Hierarchical feature learning in deep networks.}  Deep networks, organized in hierarchical layers, can perform effective and automatic feature learning \citep{allen2023backward}, where the layers learn useful representations of the data. To better understand hierarchical feature learning, plenty of studies have been conducted to investigate the structures of features learned at intermediate layers. One line of these works is to investigate the neural collapse (NC) properties at intermediate layers. For example, \cite{tirer2022extended} extended the study of neural collapse to three-layer nonlinear networks with the MSE loss, showing that the features of each layer exhibit neural collapse. \cite{dang2023neural} generalized the study of neural collapse for DLNs with imbalanced training data, and they drew a similar conclusion to \cite{tirer2022extended} that the features of each layer are collapsed. However, these results are based on the unconstrained features model and thus cannot capture the input-output relationship of the network. Moreover, the conclusion they draw on the collapse of intermediate features is far from what we observe in practice. Indeed, it is seen for both linear and nonlinear networks in \Cref{fig:progressive NC} that the intermediate features are progressively compressed across layers rather than exactly collapsing to their means for each layer. In comparison, under the assumption that the input $\bm X$ is nearly orthogonal in Assumption~\ref{AS:1}, our result characterizes the progressive compression and discrimination phenomenon to capture the training phenomena on practical networks as demonstrated in \Cref{fig:progressive NC}. 

On the other hand, another line of work \citep{yu2020learning,chan2022redunet,ma2022principles,yu2023white} argues that deep networks prevent within-class feature compression while promoting between-class feature discrimination, contrasting with our study where features are compressed across layers. Such a difference can be attributed to the choice of loss function. In our work, we focused on the study of the commonly used MSE loss, with the goal of understanding a prevalent phenomenon in classical training of deep networks. It has been empirically shown that increasing feature compression is beneficial for improving in distribution generalization and robustness \citep{papyan2020prevalence,li2022principled,ben2022nearest,he2023law,Chen2022a}. In comparison, \cite{yu2020learning} introduced a new maximum coding rate reduction loss that is intentionally designed to prevent feature compression. 


\paragraph{Learning dynamics and implicit bias of GD for training deep networks.} Our main result in Theorem~\ref{thm:NC} is based on Assumption~\ref{AS:2} for the trained weights. As discussed in \Cref{subsec:discuss}, Assumption~\ref{AS:2} holds as a consequence of results in recent works on analyzing the learning dynamics and implicit bias of gradient flow or GD in training deep networks. We briefly review the related results as follows. For training DLNs, \cite{arora2018convergence} established linear convergence of GD based upon whitened data and with a similar setup to ours. For Gaussian initialization, \cite{Du2019} showed that GD also converges globally at a linear rate when the width of hidden layers is larger than the depth, whereas \cite{Hu2020} demonstrated the advantage of orthogonal initialization over random initialization by showing that linear convergence of GD with orthogonal initialization is independent of the depth. More recent developments can be found in \cite{gidel2019implicit,Nguegnang2021,shin2022effects} for studying GD dynamics. On the other hand, another line of works focused on studying gradient flow for learning DLNs due to its simplicity \citep{bah2022learning,Eftekhari2020,min2021explicit,Tarmoun2021,bah2022learning}, by analyzing its convergence behavior. 

Numerous studies have shown that the effectiveness of deep learning is partially due to the implicit bias of its learning dynamics, which favors some particular solutions that generalize exceptionally well without overfitting in the over-parameterized setting \citep{Belkin2019,huh2021low,neyshabur2017implicit}. To gain insight into the implicit bias of GD for training deep networks, a line of recent work has shown that GD tends to learn simple functions \citep{cao2023implicit,Gunasekar2017,Ji2018,kunin2022asymmetric,Shah2020,VallePerez2018}.  For instance, some studies have shown that GD is biased towards max-margin solutions in linear networks trained for binary classification via separable data \citep{soudry2018implicit,kunin2022asymmetric}.  In addition to the simplicity bias, another line of work showed that deep networks trained by GD exhibit a bias towards low-rank solutions \citep{Gunasekar2017,huh2021low,yaras2023law}. The works \citep{gidel2019implicit,arora2019implicit} demonstrated that adding depth to matrix factorization enhances an implicit tendency towards low-rank solutions, leading to more accurate recovery.   

\section{Experimental Results}\label{sec:exp}

In this section, we conduct various numerical experiments to validate our assumptions, verify our theoretical results, and investigate the implications of our results on both synthetic and real data sets.  All of our experiments are conducted on a PC with 8GB memory and an Intel(R) Core i5 1.4GHz CPU, except for those involving large datasets such as CIFAR and FashionMNIST, which are conducted on a server equipped with NVIDIA A40 GPUs. Our code is implemented in Python and made available at \url{https://github.com/Heimine/PNC_DLN}. Throughout this section, we will repeatedly use multilayer perception (MLP) networks, where each layer of MLP networks consists of a linear layer and a batch norm layer \citep{ioffe2015batch} followed by ReLU activation. In addition, in accordance with our theoretical analysis, the compression metric is plotted for layers $0,1,\dots,L-1$, while the discrimination metric is plotted for layers $1,\dots,L-1$.   

The remaining sections are organized as follows. In \Cref{subsec:linear_approx_nonlinear}, we provide detailed experimental setups for the results discussed in \Cref{sec:intro} and shown in Figures \ref{fig:intro2}, \ref{fig:intro_init}, and \ref{fig:intro_init_2}. In \Cref{subsec:theory_and_assump}, we provide experimental results to support Theorem~\ref{thm:NC} in \Cref{sec:main} and validate Assumption \ref{AS:2} discussed in \Cref{subsec:discuss}. Finally, we empirically explore the implications of our results in \Cref{subsec:exp_explore} beyond our assumptions. 

\subsection{Significance of Studying DLNs in Feature Learning}\label{subsec:linear_approx_nonlinear} 

\begin{figure*}[t]
    \begin{subfigure}{0.48\textwidth}
    \includegraphics[width = 0.948\linewidth]{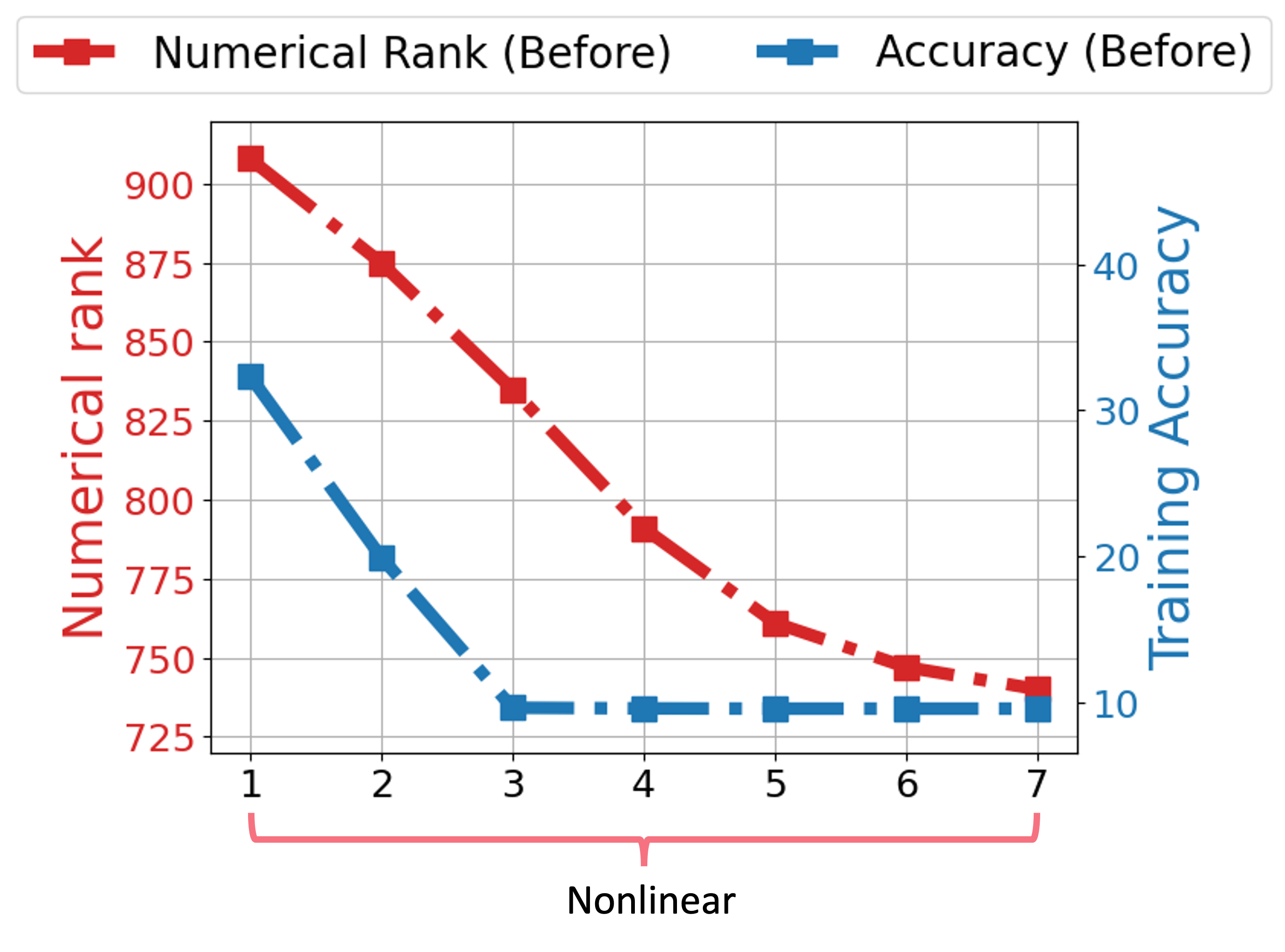}
    \caption{A 8-layer MLP network with ReLU} 
    \end{subfigure} 
    \begin{subfigure}{0.48\textwidth}
    \includegraphics[width = 0.954\linewidth]{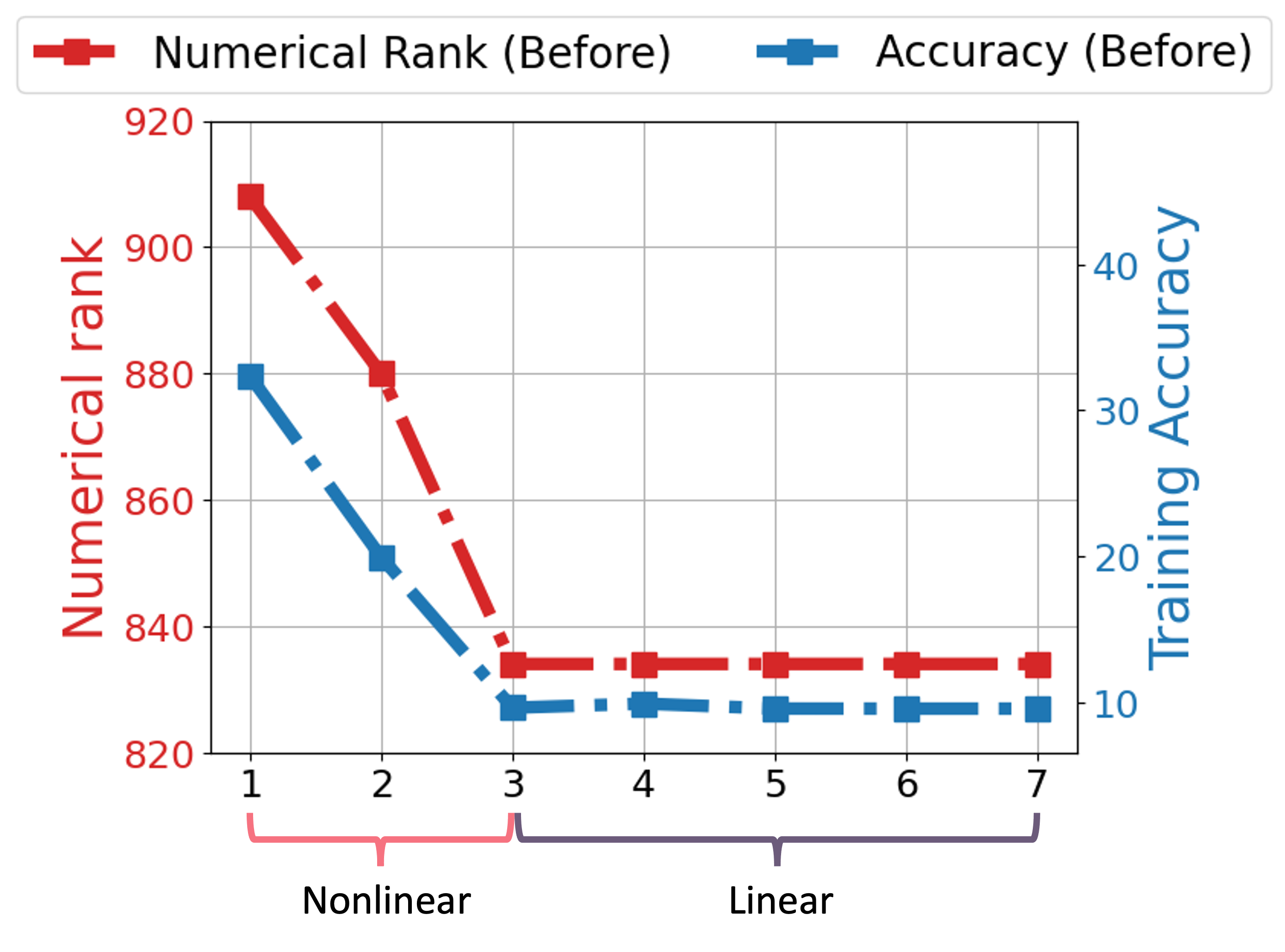}
    \caption{A 8-layer hybrid network} 
    \end{subfigure} 
    \caption{\textbf{Accuracy and numerical rank of two 8-layer networks before training.} We use the same setup as in Figures \ref{fig:intro1} and \ref{fig:intro2} but instead plot the numerical rank and training accuracy of the two networks across layers \textbf{before} training. }
    \label{fig:intro_init}
\end{figure*}

\begin{figure*}[t]
\begin{subfigure}{0.48\textwidth}
    \includegraphics[width = 0.952\linewidth]{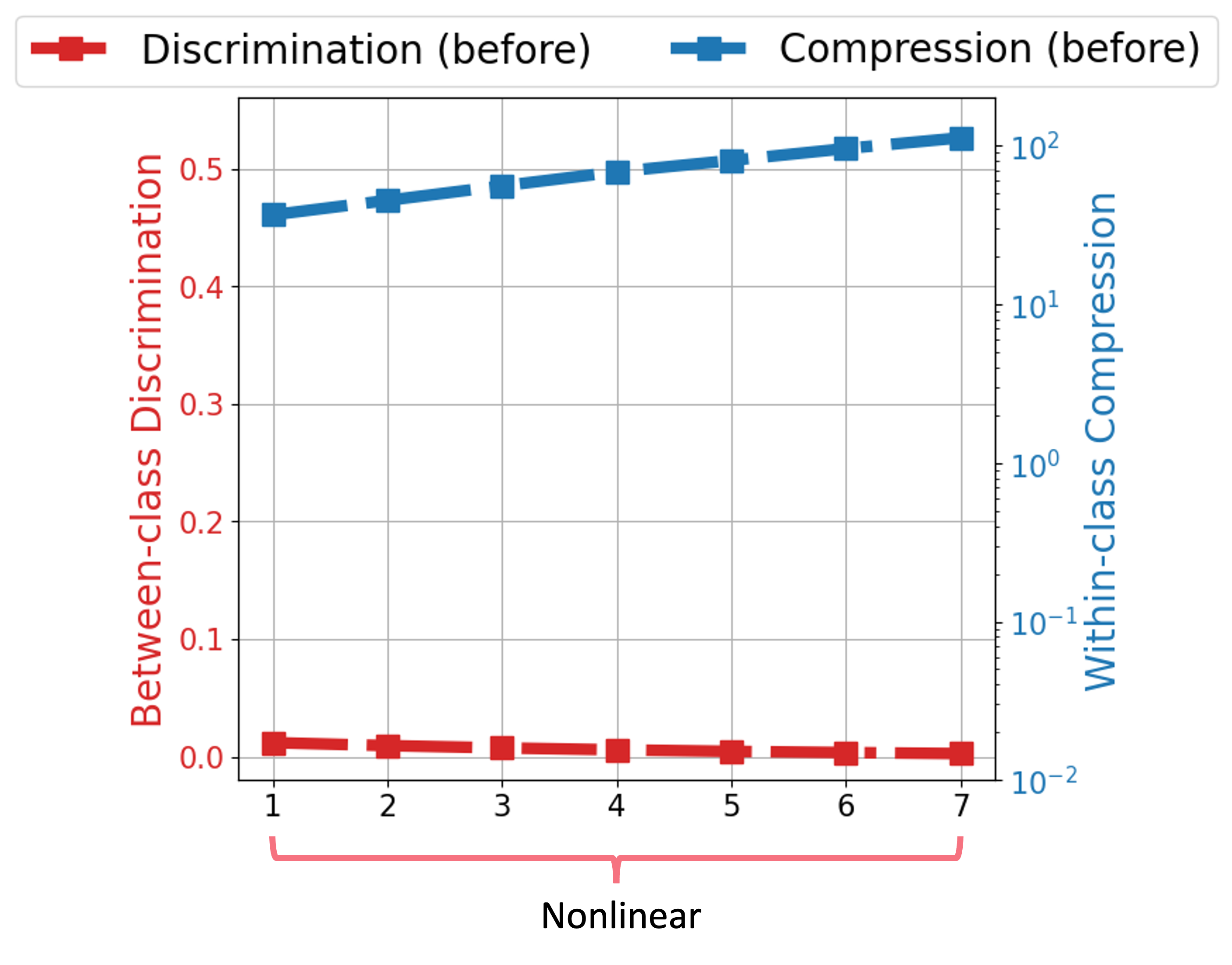}
    \caption{A 8-layer MLP network with ReLU} 
    \end{subfigure} 
    \begin{subfigure}{0.48\textwidth}
    \includegraphics[width = 0.947\linewidth]{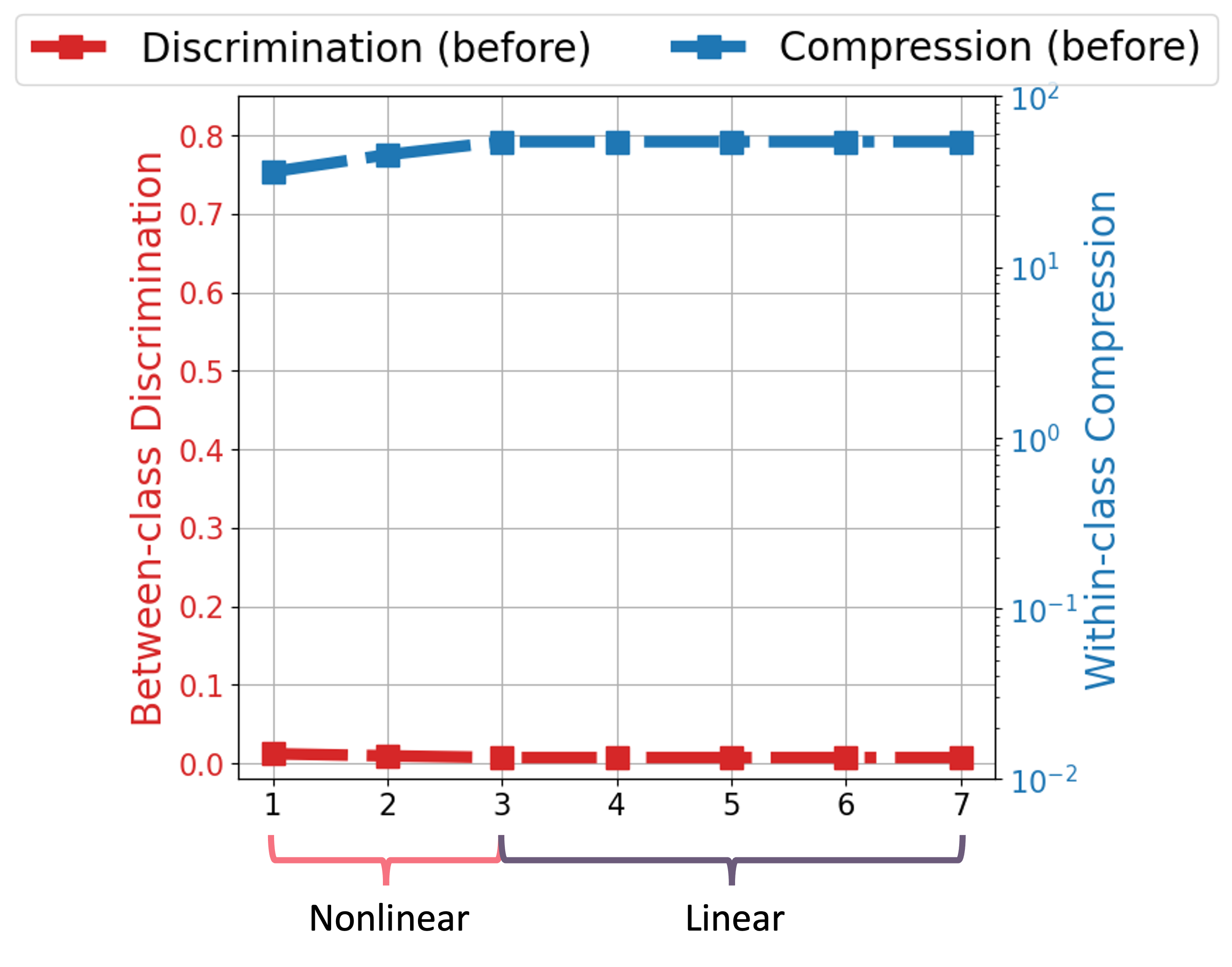}
    \caption{A 8-layer hybrid network} 
    \end{subfigure} 
    \caption{\textbf{Within-class compression and Between-class discrimination of two 8-layer networks before training.} We use the same setup as in \Cref{fig:intro1} and \Cref{fig:intro2} but instead plot the compression and discrimination statistics of the two networks across layers \textbf{before} training. }
    \label{fig:intro_init_2}
\end{figure*} 

In this subsection, our goal is to empirically demonstrate the significance of studying DLNs for feature learning as discussed in \Cref{sec:intro}. To achieve this, we investigate (i) the roles of linear layers and MLP layers at deep layers in feature learning and (ii) the importance of the depth of DLNs for generalization.  

\subsubsection{Linear Layers Mimic Deep Layers in MLPs for Feature Learning}\label{subsubsec:resemb_1}

In this subsection, we study the roles of linear layers and MLP layers in deep layers for feature learning. To begin, we provide the experimental setup for our experiments. 

\paragraph{Network architectures.} In these experiments, we construct two  8-layer networks: (a) a nonlinear MLP network and (b) a hybrid network formed by a 3-layer MLP followed by a 5-layer linear network. In both cases, the final layer ($L$-th layer) serves as a linear classifier. We set the hidden dimension $d=1024$ for all linear layers. 

\paragraph{Training dataset and training methods.} We employ the SGD optimizer to train the networks by minimizing the MSE loss on the CIFAR-10 dataset \citep{krizhevsky2009learning}. For the settings of the SGD optimizer, we use a momentum of 0.9, a weight decay of $10^{-4}$, and a dynamically adaptive learning rate ranging from $10^{-3}$ to  $10^{-5}$, modulated by a CosineAnnealing learning rate scheduler as detailed in \cite{loshchilov2017sgdr}. We use the orthogonal initialization in \eqref{eq:init} with $\xi=0.1$ to initialize the network weights. 
The neural networks are trained for 400 epochs with a batch size of 128.

\paragraph{Experiments and observations.} Now, we elaborate on the tasks conducted in Figures \ref{fig:intro1}, \ref{fig:umap}, \ref{fig:intro2}, \ref{fig:intro_init}, and \ref{fig:intro_init_2} and draw conclusions from our observations.  

\begin{itemize}[leftmargin=*]
    \item In \Cref{fig:intro1} (resp. \Cref{fig:intro_init}), we plot the numerical rank and training accuracy of the features against the layer index after training (resp. before training). Here, the numerical rank of a matrix is defined as the number of top singular values whose sum collectively accounts for more than 95\% of its nuclear norm. In terms of training accuracy, we add a linear classification layer to a given layer $l$ of the neural network and train this added layer using the cross-entropy loss to compute the classification accuracy. It is observed in \Cref{fig:intro1} that the training accuracy rapidly increases in the initial layers and nearly saturates in the deep layers, whereas the numerical rank increases in the initial layers and then decreases in the deep layers progressively in the trained networks. This observation suggests that the initial layers of a network create linearly separable features that can achieve accurate classification, while the subsequent layers further compress these features progressively.  

    \item In \Cref{fig:umap}, we employ a 2D UMAP plot \citep{mcinnes2018umap} with the default settings in the UMAP Python package to visualize the evolution of features from shallow to deep layers. It is observed that in the first layer, features do not exhibit obvious structures in terms of feature compression and discrimination; the features from the same class are more and more compressed, while the features from different classes become more and more separable from layer 2 to layer 4; this within-class compression and between-class discrimination pattern is strengthened from layer 4 to layer 6. This visualization demonstrates that features in the same class are compressed, while the features from different classes are discriminated progressively.  

    \item In \Cref{fig:intro2} (resp. \Cref{fig:intro_init_2}), we plot the metrics of within-class feature compression $C_l$ and between-class feature discrimination $D_l$ defined in Definition~\ref{def:nc1} against the layer index after training (resp. before training). 
    According to \Cref{fig:intro2}, we observe a consistent trend of decreasing $C_l$ and increasing $D_l$ against the layer index after training for both networks, albeit the hybrid network exhibits a smoother transition in $C_l$ and $D_l$ compared to the fully nonlinear network. This observation supports our result in Theorem~\ref{thm:NC}.   
\end{itemize}

\paragraph{The role of linear layers in feature learning.} Comparing \Cref{fig:intro1}(a) to \Cref{fig:intro1}(b) and \Cref{fig:intro2}(a) to \Cref{fig:intro2}(b), we conclude that linear layers play the same role as MLP layers in the deep layers of nonlinear networks in feature learning, compressing within-class features and discriminating between-class features progressively. Intuitively, this is because the representations from the initial layers are already linearly separable, hence we can replace the deeper nonlinear layers with DLNs to achieve the same functionality without sacrificing the training performance. Moreover, comparing the results after training (Figures \ref{fig:intro1} and \ref{fig:intro2}) with those before training (Figures \ref{fig:intro_init} and \ref{fig:intro_init_2}), we conclude that neural networks trained with SGD do not operate in the lazy-training regime described by the neural tangent kernel (NTK) theory, where network parameters remain close to their initialization and training behaves nearly linearly. This observation is consistent with findings from recent studies~\citep{bietti2022learning,Damian2022}.


\subsubsection{Effects of Depth in Linear Layers for Improving Generalization}\label{subsubsec:linear_gene}

In this subsection, we study the impact of depth in deep networks for generalization. To begin, we provide the experimental setup for our experiments. 

\paragraph{Network architectures.} We consider a 2-layer MLP network as our base network and construct networks of varying depths from 3 to 9 by adding linear or MLP layers to this base network. We use the PyTorch default initialization \citep{he2015delving} to initiate all weights.

\paragraph{Training dataset and training methods.} We train these networks on the FashionMNIST \citep{xiao2017fashion} and CIFAR-10 \citep{krizhevsky2009learning} datasets, using the same training approach in \Cref{subsubsec:resemb_1}, except that the initial learning rate is set as $10^{-2}$. We use $5$ different random seeds to initialize the network weights and report the average test accuracy.

\paragraph{Experiments and observations.} After training, we report the average test accuracy throughout training against the number of layers of constructed networks in \Cref{fig:why_dln}. Notably, we observe a consistent improvement in test accuracy as the networks grow deeper, regardless of whether the added layers are linear or MLP layers. This observation provides further evidence for the resemblance between DLNs and the deeper layers in MLPs. Additionally, it shows that DLNs can benefit from increased depth, similar to their counterparts in traditional nonlinear networks.

\subsection{Experimental Verification of Assumption~\ref{AS:2} and Theorem~\ref{thm:NC}}\label{subsec:theory_and_assump}

In this subsection, we conduct numerical experiments to verify our assumptions and theorem presented in \Cref{sec:main}.  Unless otherwise specified, we use the following experimental setup. 

\paragraph{Network architectures.} We study DLNs and MLPs with ReLU activation with different widths and depths specified in each figure. 

\paragraph{Training dataset.} In all experiments, we fix the number of classes $K=3$ and the number of training samples $N = 30$, with varying hidden dimensions $d\ge N$ as specified in each figure. To generate input data $\bm X \in \R^{d\times N}$ that satisfies Assumption \ref{AS:1}, we first generate a matrix $\bm A \in \R^{d\times N}$ with entries i.i.d. sampled from the standard normal distribution and a matrix $\bm B \in \R^{d\times N}$ with entries i.i.d. sampled from the uniform distribution on $[0,1]$. Then, we apply a compact SVD to $\bm A$ and set ${\bm X}_0 \in \mathbb R^{d \times N}$ as the left singular matrix. Next, we obtain a matrix $\bm N$ by normalizing $\bm B$ such that its Frobenius norm is 1. Finally, we generate $\bm X$ via $\bm X = {\bm X}_0 + \bm N$ such that $\bm X$ is nearly orthogonal satisfying Assumption~\ref{AS:1}.

\paragraph{Training method and weight initialization.} Unless otherwise specified, we train networks using full batch GD in \eqref{eq:gd}, with a fixed learning rate $\eta = 0.1$. We use orthogonal weight initialization as described in \eqref{eq:init}  with varying initialization scaling $\xi$. This initialization ensures that the weights at initialization meet the condition \eqref{eq:bala}.

\subsubsection{Experiments for Verifying Assumption~\ref{AS:2}}\label{subsec:exp-assump}

In this subsection, we conduct experiments to verify Assumption~\ref{AS:2} and corroborate the discussions in \Cref{subsec:discuss} based on the above experimental setup. Specifically, we train DLNs using orthogonal initialization in \eqref{eq:init} with $\xi=0.1$ for various network depths and widths. Given a DLN, we train it 10 times with different random seeds. In each run, we terminate training once the training loss is less than $10^{-11}$. We plot the following metrics over 10 runs against depth and width in  \Cref{fig:assump_2_all}(a) and \Cref{fig:assump_2_all}(b), respectively: 
\begin{align}\label{metric:bias}
    &\text{minimum-norm residual:} \quad ||\bY(\bm X^T\bm X)^{-1}\bX^T - \bW_{L:1}||_F, \notag \\
    &\text{averaged balancedness:}\quad \frac{1}{L-1}\sum_{l=1}^{L-1}||\bW_{l+1}^T\bW_{l+1} - \bW_l\bW_l^T||_F, \\
    & \text{variance of singular values of all weights:}\ \frac{1}{(L-1)(d-2K)} \sum_{l=1}^{L-1} \sum_{i=K+1}^{d-K} \left( \sigma_{i}(\bm W_l) - \mu \right)^2,\notag 
\end{align}
where $$\mu =  \frac{1}{(L-1)(d-2K)}\sum_{l=1}^{L-1} \sum_{i=K+1}^{d-K} \sigma_{i}(\bm W_l).$$ It can be observed from \Cref{fig:assump_2_all} that the weights of a DLN with different depths and widths trained under our settings approximately satisfy Assumption~\ref{AS:2}. This supports our discussions in \Cref{subsec:discuss}.  

\subsubsection{Experiments for Verifying Theorem~\ref{thm:NC}}\label{subsec:exp-thm}

\begin{figure*}[t]
    \begin{subfigure}{1.0\textwidth}
    \includegraphics[width = 0.325\linewidth]{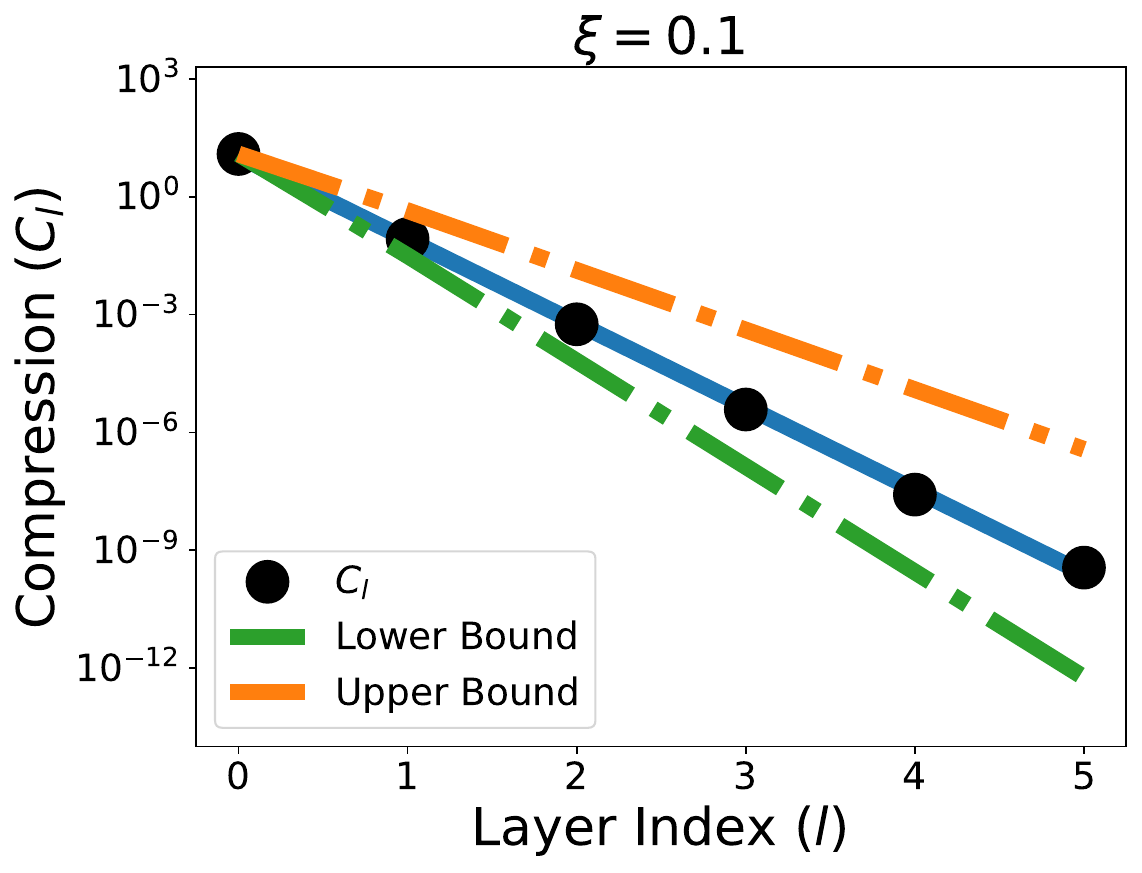}
    \includegraphics[width = 0.32\linewidth]{Figures/validate_claim/6_layer_eps03_within_unif.pdf}
    \includegraphics[width = 0.32\linewidth]{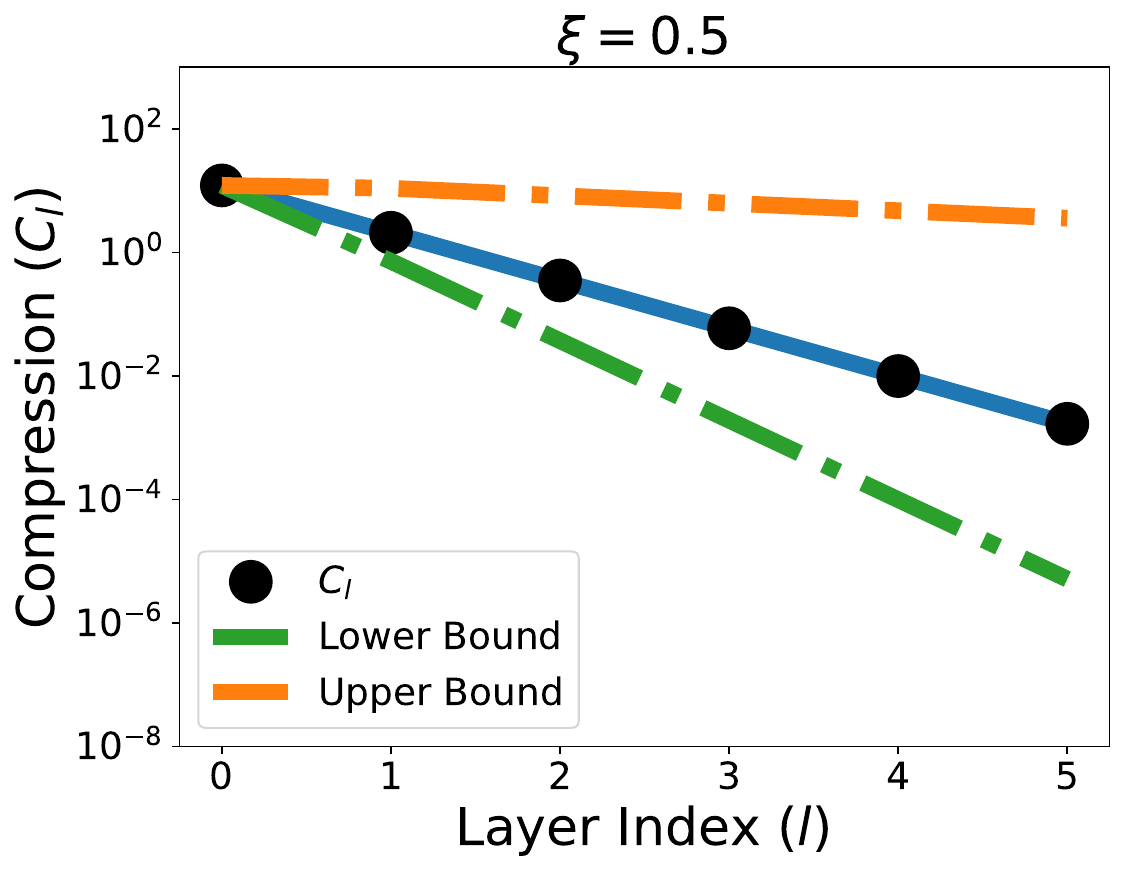}
    \caption{The metric of within-class compression $C_l$} 
    \end{subfigure} 

    \vspace{1mm}
    \begin{subfigure}{1.0\textwidth}
    \includegraphics[width = 0.32\linewidth]{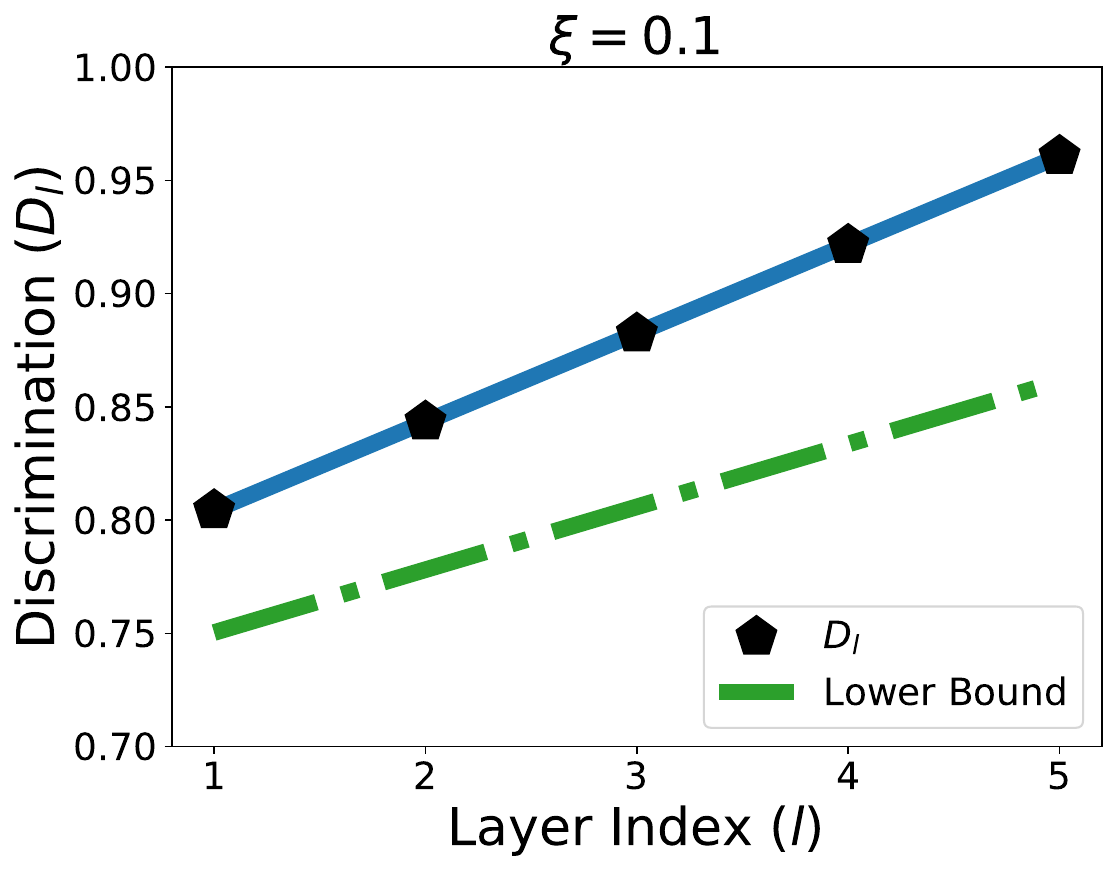}
    \includegraphics[width = 0.32\linewidth]{Figures/validate_claim/6_layer_eps03_between_unif.pdf}
    \includegraphics[width = 0.32\linewidth]{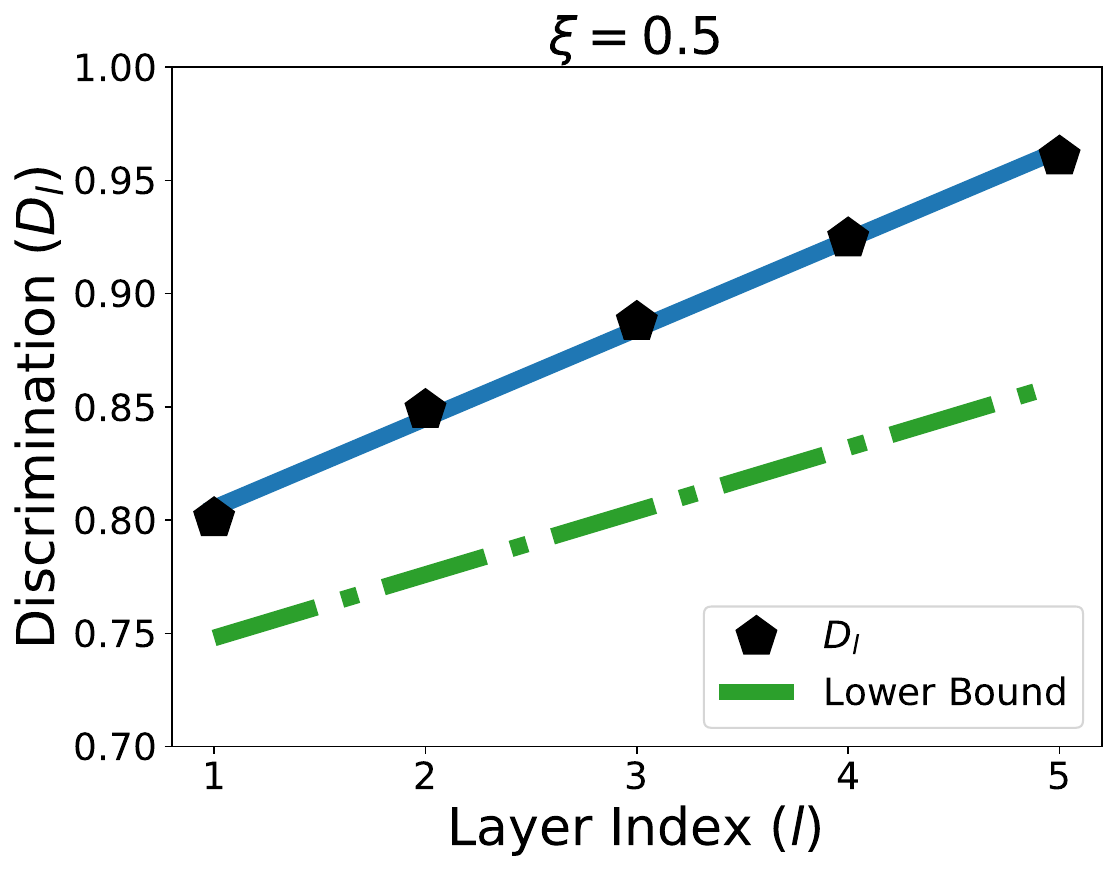}
    \caption{The metric of between-class discrimination $D_l$} 
    \end{subfigure} 

    \caption{\textbf{Visualization of the metrics of within-class compression $C_l$ and between-class discrimination across layers.} We train 6-layer DLNs with $\xi = 0.1, 0.3, 0.5$, respectively. In the top figures, we plot the metric of within-class compression $C_l$ along with the corresponding theoretical upper and lower bounds given in \eqref{eq2:compres}. In the bottom figures, we plot the metric of between-class discrimination given in \eqref{eq:discri}. }
    \label{fig:check_bound}
\end{figure*}

Now, we conduct experiments to validate Theorem~\ref{thm:NC}. Specifically, we train 6-layer DLNs using GD with varying initialization scales $\xi=0.1, 0.3, 0.5$, respectively. Then, we plot the metrics of within-class compression $C_l$ and between-class discrimination $D_l$ against layer index in \Cref{fig:check_bound}(a) and \Cref{fig:check_bound}(b), respectively. It is observed that the feature compression metric $C_l$ decreases exponentially, while the feature discrimination metric $D_l$ increases linearly w.r.t. the layer index.  In the top row of \Cref{fig:check_bound}, the solid blue line is plotted by fitting the values of $C_l$ at different layers, and the green and orange dash-dotted lines are plotted according to the lower and upper bounds in \eqref{eq2:compres}, respectively. We can observe that the solid blue line is tightly sandwiched between the green and orange dash-dotted lines. This indicates that \eqref{eq1:compres} and \eqref{eq2:compres} provide a valid and tight bound on the decay rate of the feature compression metric. According to the bottom row of \Cref{fig:check_bound}, we conclude that \eqref{eq:discri} provides a valid bound on the growth rate of the feature discrimination metric. 


Moreover, we respectively train 6-layer DLN and 10-layer MLP networks with hidden dimension $d=50$ via GD with an initialization scale $\xi=0.3$. In \Cref{fig:progressive NC}, we plot the feature compression and discrimination metrics on both DLN and MLP networks, respectively.  We observe that the exponential decay of feature compression and the linear increase of feature discrimination hold exactly in DLNs and approximately in nonlinear networks. 

\subsection{Exploratory Experiments}\label{subsec:exp_explore} 

\subsubsection{Empirical Results Beyond Theory}

In this subsection, we conduct exploratory experiments to demonstrate the universality of our result. Unless otherwise specified, we use the experimental setup outlined at the beginning in \Cref{subsec:theory_and_assump} in these experiments. Here are our observed findings: 
\begin{itemize}[leftmargin=*]
    \item \textbf{Progressive feature compression and discrimination in nonlinear deep networks.}  
    We train 8-layer and 16-layer MLP networks, respectively. After training, we plot the metrics of feature compression and discrimination defined in \eqref{def:nc1} in \Cref{fig:test-nonlinear}. It is observed that 
    the feature compression metric $C_l$ decays at an \emph{approximate} geometric rate in nonlinear networks, while the feature discrimination metric $D_l$ increases at an \emph{approximate} linear rate. Additionally, it is worth mentioning that a similar ``law of separation'' phenomenon has been reported in \cite{he2023law,li2022principled} on nonlinear networks, but their results are based upon a different metric of data separation. 
    
    \item  \textbf{Progressive feature compression and discrimination on DLNs with generic initialization.} 
    In most of our experiments and discussion in \Cref{subsec:discuss}, we mainly focused on orthogonal initialization, which simplifies our analysis due to induced weight balancedness across layers. To demonstrate the generality of our results, we also test the default initialization in the PyTorch package and train the DLN. As shown in \Cref{fig:test-init}, we can observe that the compression metric $C_l$ decays from shallow to deep layers at an \emph{approximate} geometric rate with different network depth $L$. Additionally, the discrimination metric $D_l$ increases at an {\em approximate} linear rate.

    \item \textbf{Progressive feature compression on real datasets.} We train a hybrid network on FashionMNIST and CIFAR datasets using the network architectures and training methods in \Cref{subsubsec:linear_gene} and plot the metric of within-class compression against layer index in \Cref{fig:test-init-dataset}. Although we use real datasets and PyTorch default initialization, we can still observe that the within-class compression metric $C_l$ decays progressively at an approximate geometric rate. This further demonstrates the universality of our studied phenomenon. Moreover, the results in \Cref{fig:test-init-dataset} also illustrate the role of depth, where the decay rates are approximately the same across all settings, and thus deeper networks lead to more feature compression.
  
\end{itemize}

Beyond these findings, we conduct additional experiments on the SST-5 text dataset \citep{socher2013recursive} and the ImageNet \citep{deng2009imagenet} dataset to further support our claims. These results are reported in \Cref{app:extra-exp-modality}. 

\begin{figure*}[t]
    \begin{subfigure}{0.49\textwidth}
    \includegraphics[width = 0.49\linewidth]{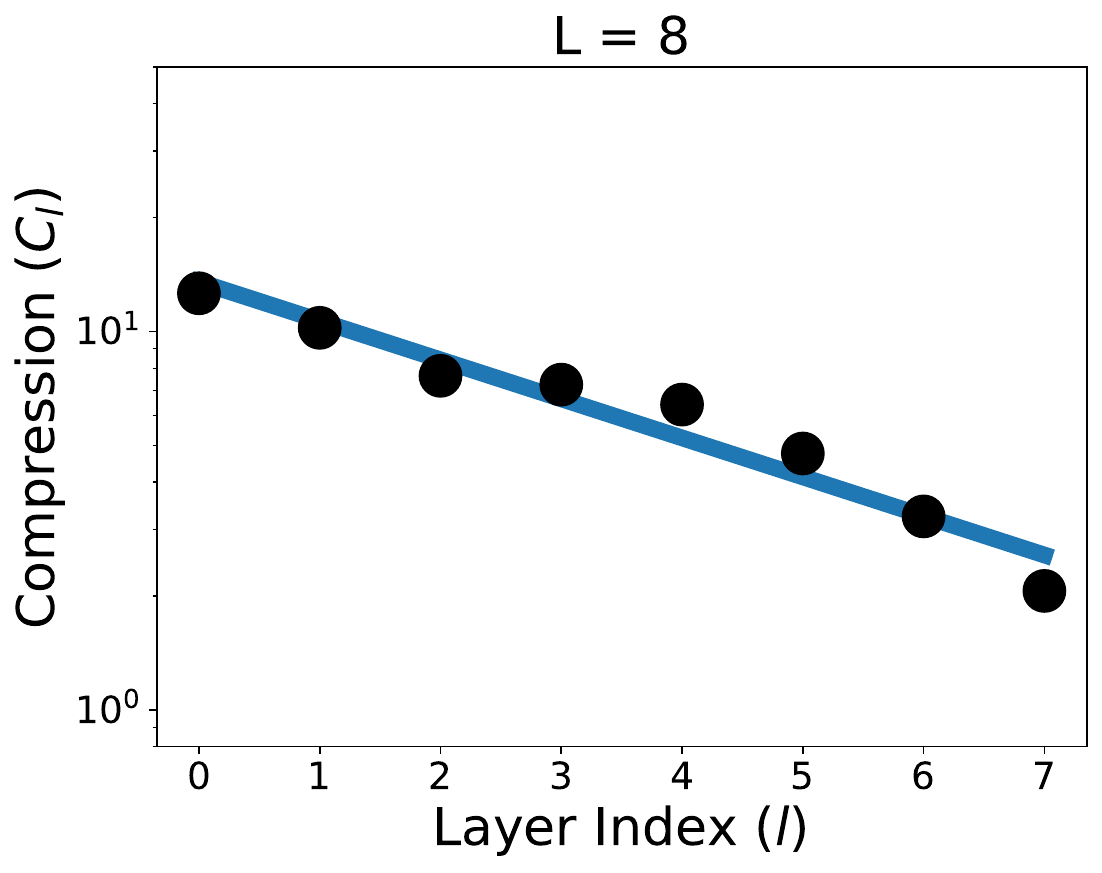}
    \includegraphics[width = 0.49\linewidth]{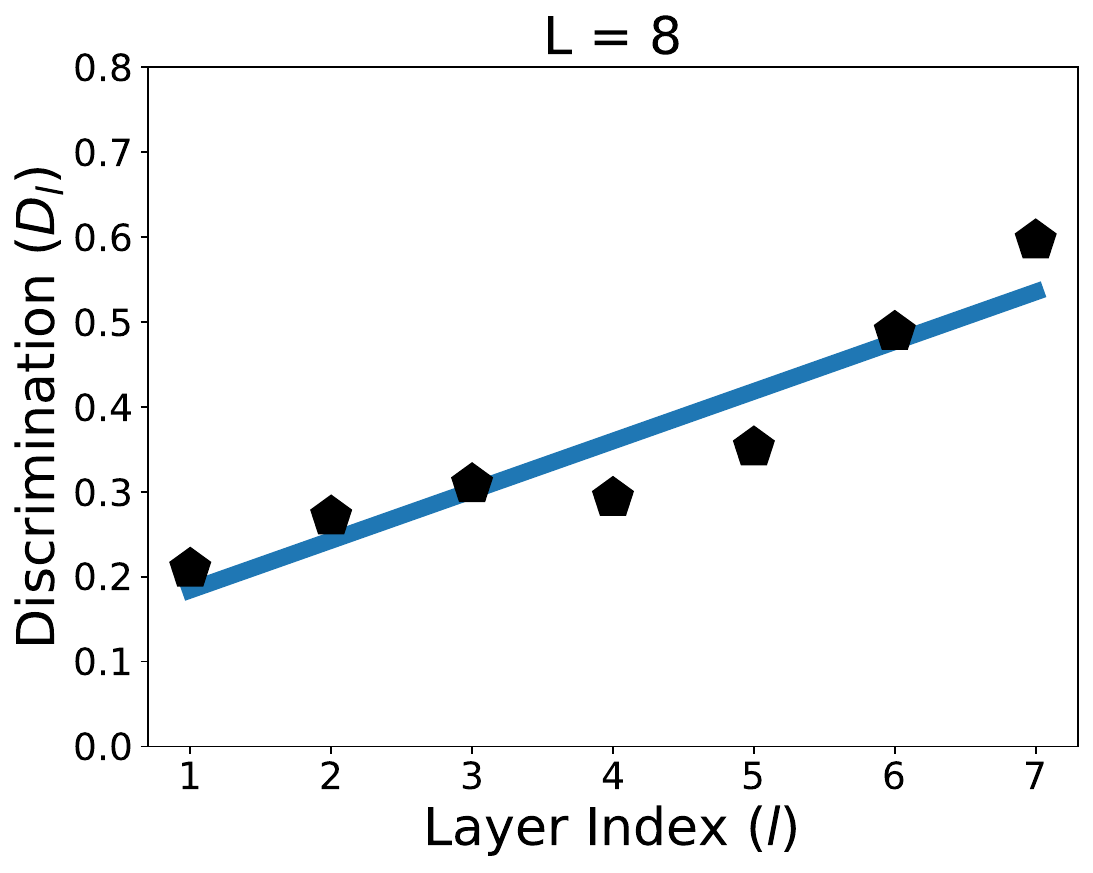} 
    \caption{An 8-layer MLP.} 
    \end{subfigure} 
    \hspace{1mm}
    \begin{subfigure}{0.49\textwidth}
    \includegraphics[width = 0.49\linewidth]{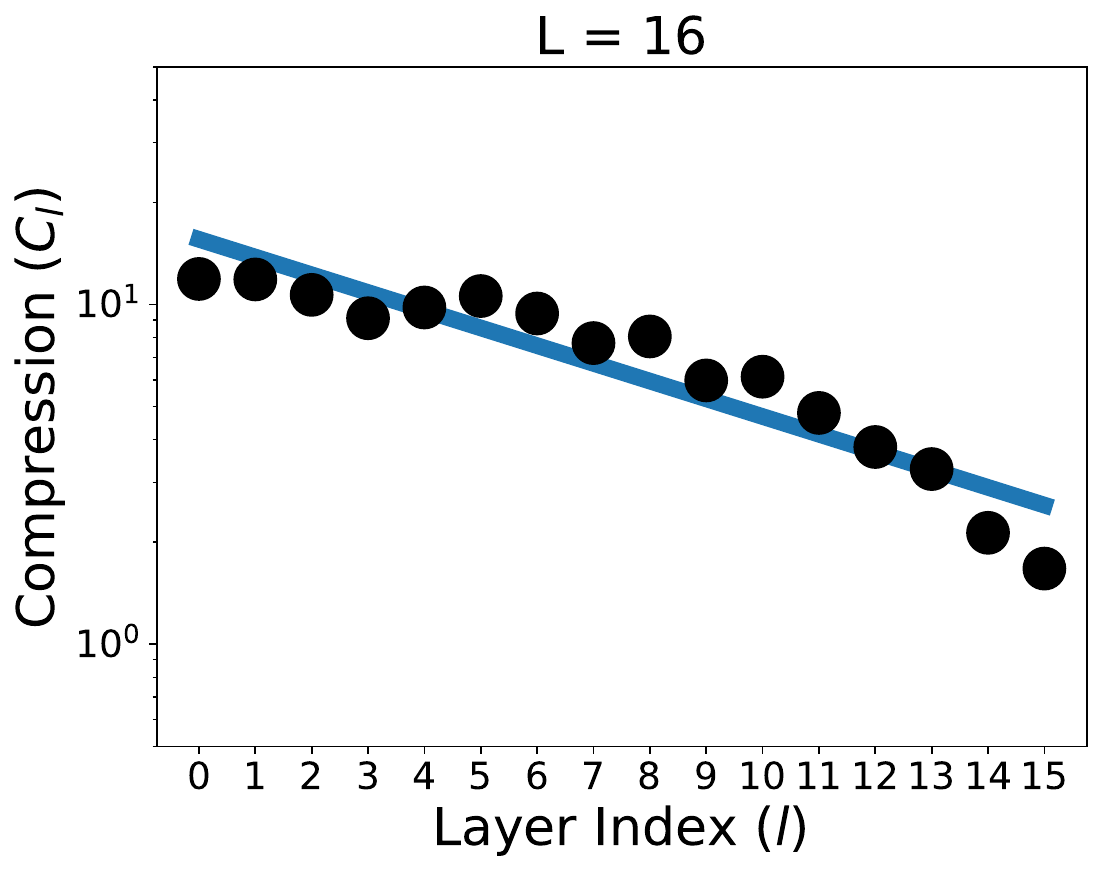}
    \includegraphics[width = 0.49\linewidth]{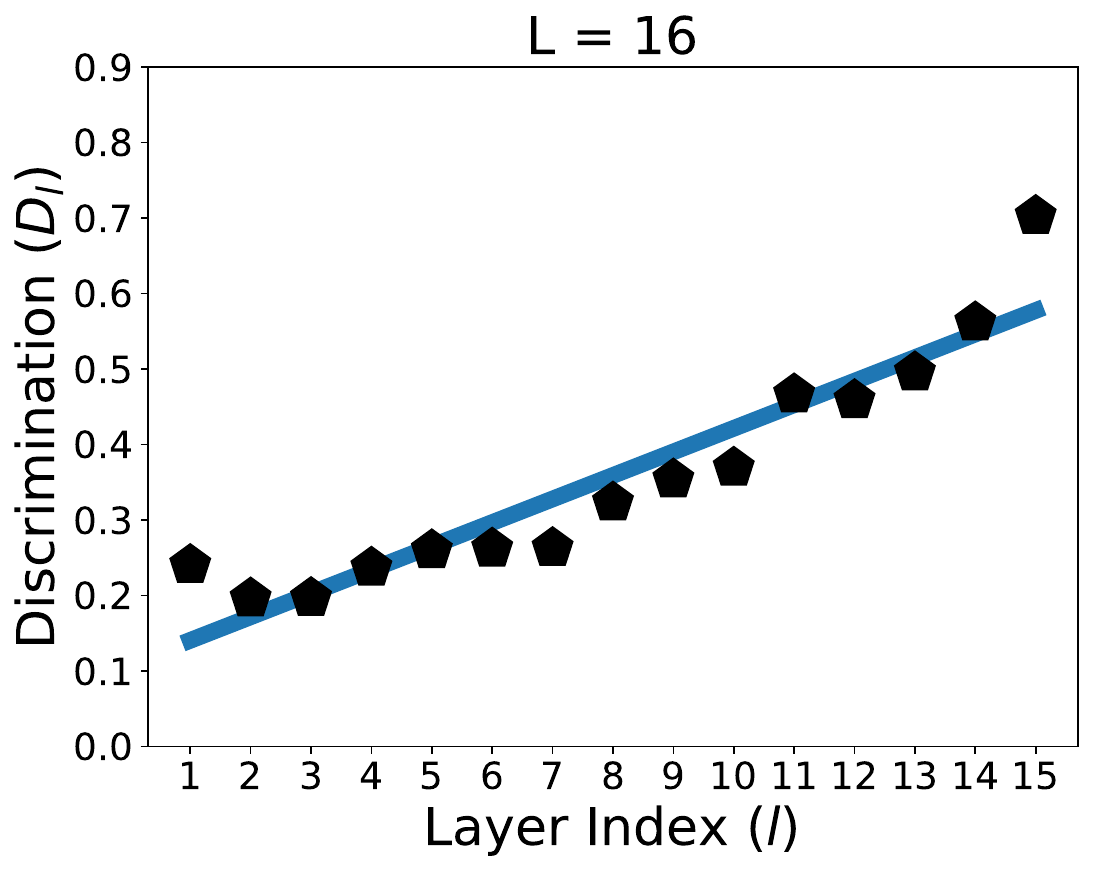}
    \caption{A 16-layer MLP.} 
    \end{subfigure} 

    \caption{\textbf{Progressive feature compression and discrimination on nonlinear networks.} We train MLP networks using the default orthogonal weight initialization and plot the metrics $C_l$ and $D_l$ against layer index. }
    \label{fig:test-nonlinear}
\end{figure*}
\begin{figure*}[t]
    \begin{subfigure}{0.49\textwidth}
    \includegraphics[width = 0.49\linewidth]{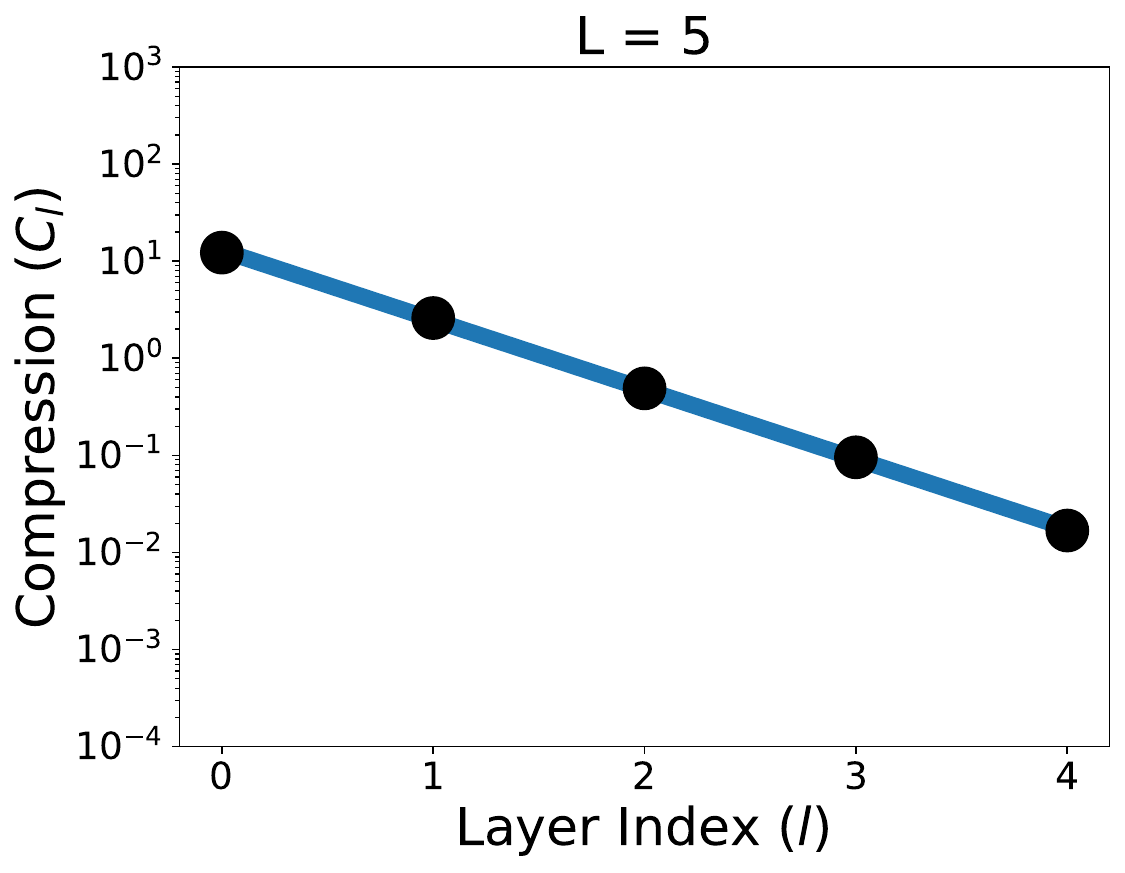}
    \includegraphics[width = 0.49\linewidth]{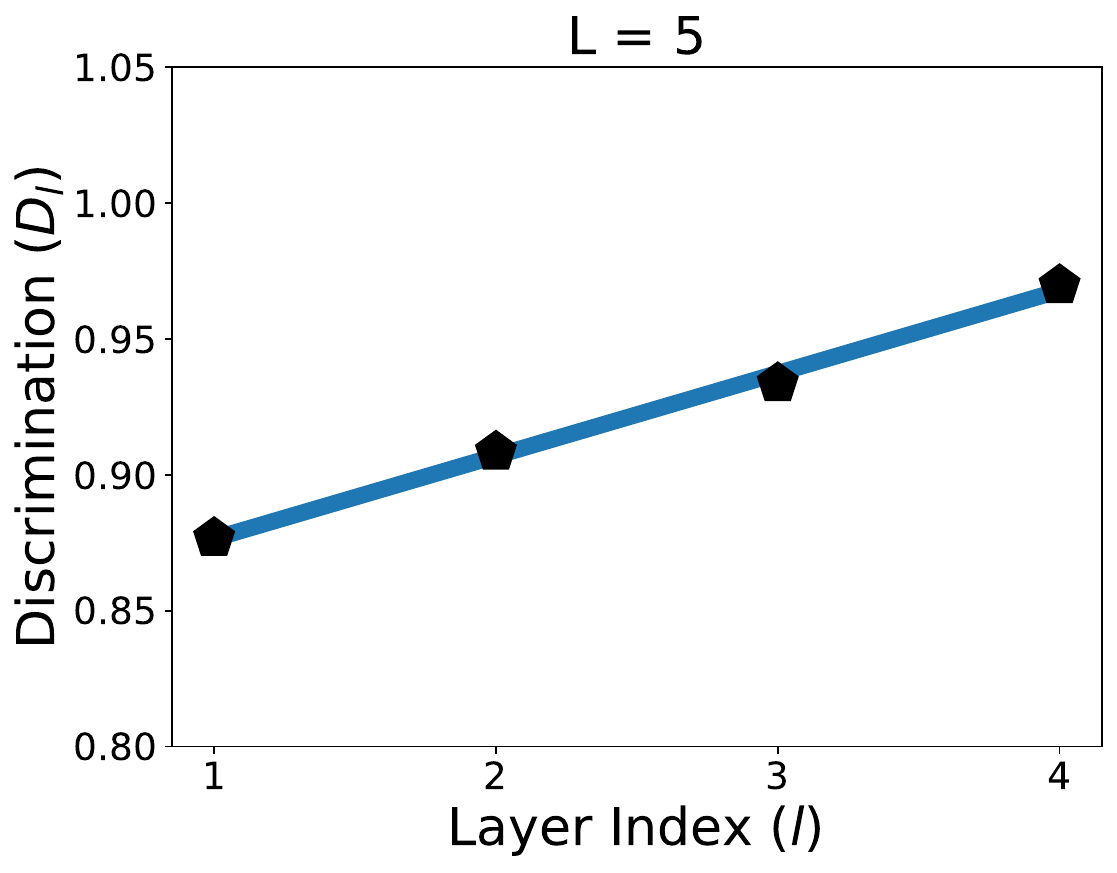} 
    \caption{A 5-layer DLN with Kaiming Init.} 
    \end{subfigure} 
    \hspace{1mm}
    \begin{subfigure}{0.49\textwidth}
    \includegraphics[width = 0.49\linewidth]{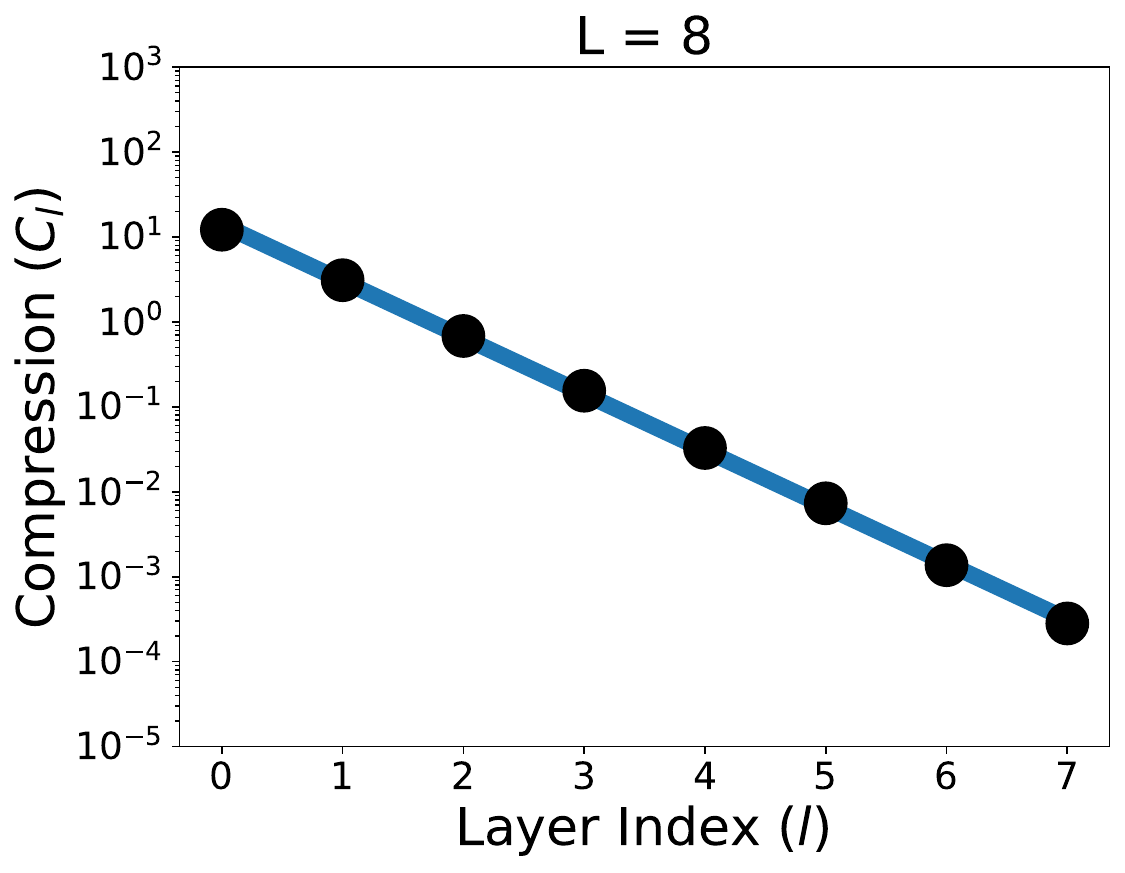}
    \includegraphics[width = 0.49\linewidth]{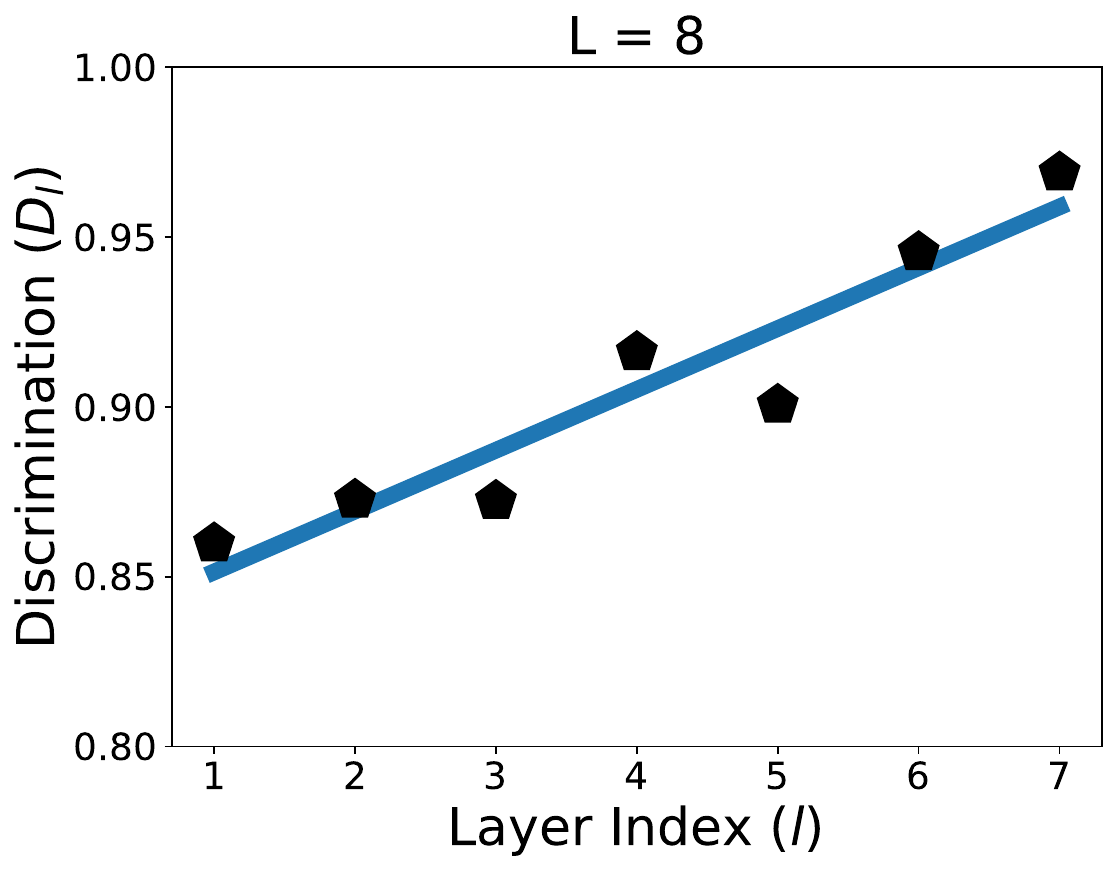}
    \caption{An 8-layer DLN with Kaiming Init.} 
    \end{subfigure} 
    \caption{\textbf{Progressive feature compression and discrimination on DLNs with default PyTorch initialization.} We train DLNs with uniform weight initialization and plot the dynamics of $C_l$ and $D_l$. We can observe progressive linear decay of $C_l$ still happens without the orthogonal initialization while the expanding pattern of $D_l$ disappears. }
    \label{fig:test-init}
\end{figure*}


\subsubsection{Implication on Transfer Learning}\label{subsec:TL}

In this subsection, we experimentally substantiate our claims on transfer learning in \Cref{sec:main} using practical nonlinear networks and real datasets. The results demonstrate that features before projection heads are less collapsed and exhibit better transferability. To support our argument, we conduct our experiments based on the following setup.   

\paragraph{Network architectures.} We employ a ResNet18 backbone architecture \citep{he2016deep}, incorporating $t \in \{1,2,3,4,5\}$ layers of projection heads between the feature extractor and the final classifier, respectively. Here, one layer of the projection head consists of a linear layer followed by a ReLU activation layer. These projection layers are only used in the pre-training phase. On the downstream tasks, they are discarded, and a new linear classifier is trained on the downstream dataset.
\paragraph{Training datasets and training methods.} We use the CIFAR-100 and CIFAR-10 dataset in the pre-training and fine-tuning tasks, respectively. We train the networks using the Rescaled-MSE loss \citep{hui2020evaluation}, with hyperparameters set to $k=5$ and $M=20$ for 200 epochs. During pre-training, we employed the SGD optimizer with a momentum of 0.9, a weight decay of $5 \times 10^{-4}$, and a dynamically adaptive learning rate ranging from $10^{-2}$ to $10^{-5}$, modulated by a CosineAnnealing learning rate scheduler \citep{loshchilov2017sgdr}. During the fine-tuning phase, we freeze all the parameters of the pre-trained model and only conduct linear probing. In other words, we only train a linear classifier on the downstream data for an additional 200 epochs. We run each experiment with 3 different random seeds.

\begin{figure*}[t] 
    \begin{subfigure}{0.50\textwidth}
    \centering
    \includegraphics[width = 0.90\linewidth]{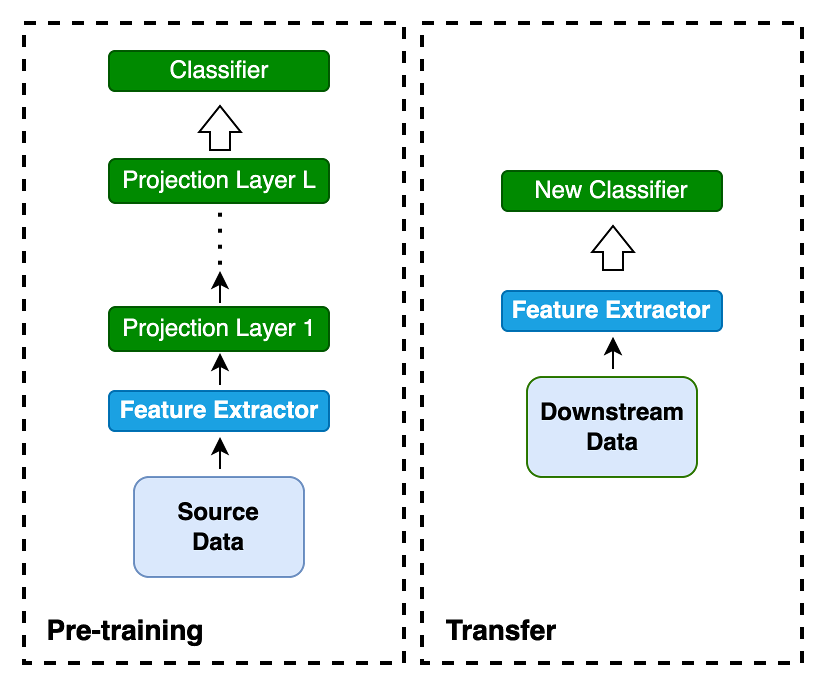}
    \caption{Illustration on the usage of projection heads} 
    \end{subfigure} 
    \begin{subfigure}{0.50\textwidth}
    \centering
    \includegraphics[width = 0.99\linewidth]{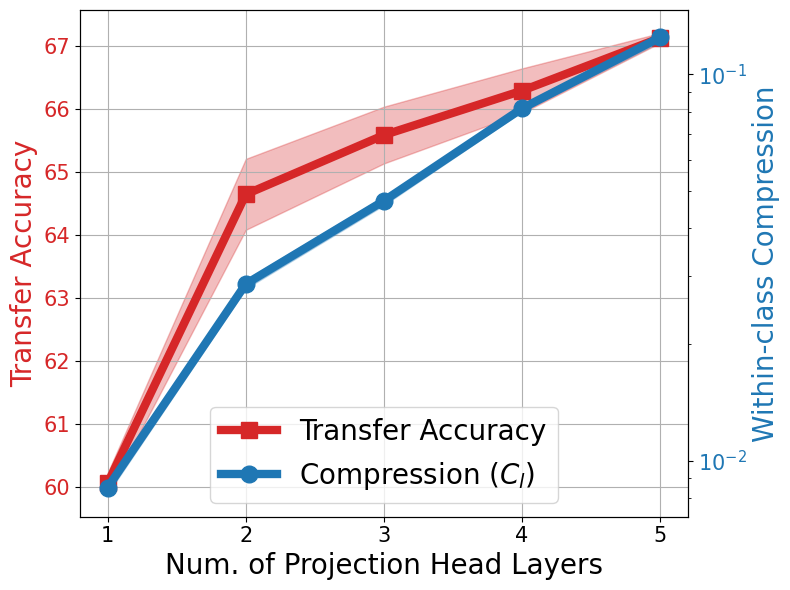}
    \caption{Transfer learning results} 
    \end{subfigure} 
    \caption{\textbf{Adding projection heads mitigates feature collapse and improves transfer accuracy.} We pre-train ResNet18 backbones with 1 to 5 layers of projection heads on the CIFAR-100 dataset. After pre-training, we drop the projection heads and compute the metric $C_l$ on the learned features. Finally, we linearly probe the pre-trained models on the CIFAR-10 dataset and plot the evolution of $C_l$ alongside the corresponding transfer learning performance against the number of projection head layers. The shaded area represents the standard deviation of $3$ random seeds. The plot shows a clear trend that adding more layers of projection heads mitigates feature collapse and improves transfer learning results.} 
    \label{fig:transfer}
\end{figure*} 

\paragraph{Experiments and observations.} \Cref{fig:transfer} illustrates the relationship between the number of layers in a projection head and two distinct metrics: (i) the compression metric $C_l$ of the learn features on the pre-trained dataset (depicted by the blue curve), and (ii) the transfer accuracy of pre-trained models on downstream tasks (depicted by the red curve). It is observed from \Cref{fig:transfer}(b) that an increase in the number of layers in the projection head leads to decreased feature compression and better transfer accuracy. 

These observations confirm our theoretical understanding that feature compression occurs progressively through the layers from shallow to deep and that the use of projection heads during the pre-training phase helps to prevent feature collapse at the feature extractor layer, thereby improving the model's transfer accuracy on new downstream data. Furthermore, adding more layers in the projection head tends to preserve diverse features of the pre-trained model, resulting in improved transfer learning performance.

\section{Conclusion}\label{sec:conclusion}

In this work, we studied hierarchical representations of deep networks by analyzing their intermediate features. In the context of training DLNs for solving multi-class classification problems, we examined how features evolve across layers through within-class feature compression and between-class feature discrimination. We showed that under mild assumptions on the input data and trained weights, each layer of a DLN progressively compresses within-class features at a geometric rate and discriminates between-class features at a linear rate w.r.t. the layer index. Moreover, we discussed the implications of our results for deep learning practices, including interpretation, architecture design, and transfer learning. Our extensive experimental results on synthetic and real data sets not only support our theoretical findings but also highlight their relevance to deep nonlinear networks. 

Our work opens several interesting directions for future work. First, our extensive experiments have demonstrated analogous patterns of progressive compression and discrimination phenomena in deep nonlinear networks. Extending our analysis to deep nonlinear networks emerges as a natural and promising direction for future research. Second, \cite{yu2023white,geshkovski2023emergence} demonstrated that the outputs of self-attention layers in transformers exhibit a similar progressive compression phenomenon. It would be interesting to study this phenomenon in transformers based on our proposed framework.  

\section*{Acknowledgment}
XL, CY, PW, and QQ acknowledge support from NSF CAREER CCF-2143904, NSF CCF-2212066, NSF CCF-2212326, NSF IIS 2312842, ONR N00014-22-1-2529, and an AWS AI Award. QQ also acknowledges the Google Research Scholar Award. PW and LB acknowledge support from DoE award DE-SC0022186, ARO YIP W911NF1910027, and NSF CAREER CCF-1845076. ZZ acknowledges support from NSF grants CCF-2240708,  IIS-2312840, and IIS-2402952. WH acknowledges support from the Google Research Scholar Program. Results presented in this paper were obtained using CloudBank, which is supported by the NSF under Award \#1925001. The authors acknowledge valuable discussions with Dr. Weijie Su (Upenn), Dr. John Wright (Columbia), Dr. Rene Vidal (Upenn), Dr. Hangfeng He (U. Rochester), Mr. Zekai Zhang (U. M), Dr. Yaodong Yu (UC Berkeley), and Dr. Yuexiang Zhai (UC Berkeley) at different stages of the work.


\newpage

\appendix

\section{Proofs of Main Results in \Cref{sec:main}}\label{sec:proof}

In this section, we provide formal proofs to show our main result in Theorem~\ref{thm:NC} concerning the behavior of within-class and between-class features. This involves establishing some key properties of the weight matrices of each layer under Assumptions \ref{AS:1} and \ref{AS:2}. The rest of this section is organized as follows: In \Cref{subsec:weight}, we establish some preliminary results for weight matrices under Assumptions \ref{AS:1} and \ref{AS:2}. In \Cref{subsec:compre}, we prove that within-class features are compressed across layers at a geometric rate. In \Cref{subsec:discr}, we prove that between-class features are discriminated across layers at a linear rate. Finally, we finish the proofs of Theorem~\ref{thm:NC} and Proposition~\ref{prop:AS} in \Cref{subsec:pf thm} and \ref{subsec:pf prop}, respectively.   

Before we proceed, let us introduce some further notation that will be used throughout this section. Using \eqref{eq:min norm}, $\bY = \bm{I}_K \otimes \bm{1}_n^T$, and Assumption \ref{AS:1}, we show that the rank of $\bW_l$ is at least $K$ for each $l \in [L]$; see Lemma \ref{lem:bound Wl}. For all $l \in [L-1]$, let
\begin{align}\label{eq:SVD Wl}
\bW_l = \bU_l\bm{\Sigma}_l\bV_l^T =  \begin{bmatrix}
\bU_{l,1} & \bU_{l,2} 
\end{bmatrix}\begin{bmatrix}
\bm{\Sigma}_{l,1} & \bm{0} \\
\bm{0} & \bm{\Sigma}_{l,2}
\end{bmatrix}\begin{bmatrix}
\bV_{l,1}^T \\ \bV_{l,2}^T
\end{bmatrix} = \bU_{l,1} \bm{\Sigma}_{l,1}\bV_{l,1}^T + \bU_{l,2}\bm{\Sigma}_{l,2}\bV_{l,2}^T, 
\end{align}   
be a singular value decomposition (SVD) of $\bW_l$, where $\bm{\Sigma}_l \in \R^{d\times d}$ is diagonal, $\bm{\Sigma}_{l,1} = \diag\left(\sigma_{l,1},\dots,\sigma_{l,K} \right)$ with $\sigma_{l,1} \ge \cdots \ge \sigma_{l,K} > 0 $ being the first $K$ leading singular values  of $\bW_l$, and $\bm{\Sigma}_{l,2} = 
\diag\left(\sigma_{l,K+1},\dots,\sigma_{l,d} \right) $ with $\sigma_{l,K+1} \ge \cdots \ge \sigma_{l,d} \ge 0 $ being the remaining singular values; $\bU_{l} \in \mO^{d }$ with $\bU_{l,1} \in \R^{d\times K}$, $\bU_{l,2} \in \R^{d \times (d-K)}$; $\bV_{l} \in \mO^{d}$ with $\bV_{l,1} \in \R^{d \times K}$, $\bV_{l,2} \in \R^{d \times (d-K)}$. Noting that $\bW_L \in \R^{K\times d}$, let
\begin{align}\label{eq:SVD WL}
\bW_L =  \bU_L\bm{\Sigma}_L\bV_L^T =  \bU_L \begin{bmatrix}
\bm{\Sigma}_{L,1} & \bm{0} 
\end{bmatrix}\begin{bmatrix}
\bV_{L,1}^T \\ \bV_{L,2}^T
\end{bmatrix} = \bU_{L}\bm{\Sigma}_{L,1}\bV_{L,1}^T,
\end{align} 
be an SVD of $\bm W_L$, where $\bm{\Sigma}_{L} \in \R^{K\times d}$, and $\bm{\Sigma}_{L,1} = \diag\left(\sigma_{L,1},\dots,\sigma_{L,K} \right)$ with $\sigma_{L,1} \ge \cdots \ge \sigma_{L,K} > 0 $ being the singular values; $\bU_{L} \in \mO^{K}$, and $\bV_{L} \in \mO^{d}$ with $\bV_{L,1} \in \R^{d\times K}$, $\bV_{L,2} \in \R^{d \times (d-K)}$.   
Moreover, we respectively denote the $k$-th class mean and global mean by 
\begin{align}\label{eq:xk}
\bar{\bm x}_k = \frac{1}{n}\sum_{i=1}^N \bm x_{k,i},\quad \bar{\bm x} = \frac{1}{K} \sum_{k=1}^K \bar{\bm x}_k,
\end{align}
and let
\begin{align}\label{eq:Xbar}
\bar{\bm X} = \left[\bar{\bm{x}}_1\ \bar{\bm{x}}_2\ \dots\ \bar{\bm{x}}_K\right] \otimes \bm{1}_n^T \in \R^{d\times N}. 
\end{align}
Note that the above notation will be used consistently throughout the rest of this paper.

\subsection{Properties of Weight Matrices}\label{subsec:weight}

In this subsection, we establish some properties of weight matrices $\{\bm W_l\}_{l=1}^L$ under Assumptions \ref{AS:1} and \ref{AS:2}. Towards this goal, we first prove some inequalities on the spectrum of the data matrix $\bm X$ and the distance between $\bm X$ and $\bar{\bm X}$ under Assumption \ref{AS:1}. These inequalities will serve as crucial components in the subsequent analysis of the spectral properties of the weight matrices. 

\begin{lemma}\label{lem:spec X}
Suppose that Assumption \ref{AS:1} holds. Then, we have 
\begin{align}
& \sqrt{1-\theta} \le \sigma_{\min}(\bm X) \le \sigma_{\max}(\bm X) \le \sqrt{1+\theta}, \label{eq:spec X}\\
& \|\bm X - \bar{\bm X}\| \le \sqrt{1+4\theta},\quad N - K - 4\theta \le \|\bm X - \bar{\bm X}\|_F^2 \le N - K + 4\theta.  \label{eq:X Xbar}
\end{align}
\end{lemma}
\begin{proof}
Using $\|\bm A\| \le \|\bm A\|_F$ for all $\bm A \in \R^{N\times N}$, we have
\begin{align}\label{eq1:lem spec X}
\|\bm X^T\bm X - \bm I_N\|^2 \le \|\bm X^T\bm X - \bm I_N\|_F^2  \le  \sum_{i=1}^N\left(\|\bm x_i\|^2 - 1 \right)^2 + \sum_{i=1}^N \sum_{j \neq i}^N \left(\bm x_i^T\bm x_j\right)^2  \le \theta^2,
\end{align}
where the last inequality follows from \eqref{eq:orth}. This, together with Weyl's inequality, implies $1-\theta \le \sigma_{\min}(\bm X^T\bm X) \le \sigma_{\max}(\bm X^T\bm X) \le 1 + \theta$.  Therefore, we obtain \eqref{eq:spec X}.  

Now, we prove \eqref{eq:X Xbar}. For ease of exposition, let $\bm A:= (\bm X - \bar{\bm X})^T(\bm X - \bar{\bm X}) = \bm X^T\bm X - \bm X^T\bar{\bm X} - \bar{\bm X}^T \bm X + \bar{\bm X}^T \bar{\bm X}$. First, we consider  the elements of diagonal blocks of $\bm A$. For each $i=(k-1)n+j$ and $j \in [n]$, we compute
\begin{align*}
 a_{ii}  & =  \left\|\bm x_{k,j} - \bar{\bm x}_k \right\|^2  = \left(1-\frac{1}{n}\right)^2 \|\bm x_{k,j}\|^2  + \frac{1}{n^2}\sum_{j^\prime \neq j}\|\bm x_{k,j^\prime}\|^2 - \\
 &\quad \frac{2}{n} \sum_{j \neq j^\prime} \langle \bm x_{k,j}, \bm x_{k,j^\prime} \rangle + \frac{1}{n^2} \sum_{j=1}^n\sum_{j^\prime \neq j} \langle \bm x_{k,j}, \bm x_{k,j^\prime} \rangle.  
\end{align*}
This, together with \eqref{eq:orth} in Assumption \ref{AS:1}, yields
\begin{align}\label{eq2:lem spec X}
1 - \frac{1}{n} - \frac{4\theta}{N} \le a_{ii} \le 1 - \frac{1}{n} + \frac{4\theta}{N}.
\end{align}
Using the similar argument, we compute for each $i=(k-1)n+j$, $i^\prime=(k-1)n+j^\prime$, and $j \neq j^\prime \in [n]$,
\begin{align}\label{eq3:lem spec X}
-\frac{1}{n}- \frac{4\theta}{N} \le a_{i,i^\prime}  \le -\frac{1}{n} + \frac{4\theta}{N}.  
\end{align}
Now, we consider the elements of off-diagonal blocks of $\bm A$. Similarly, we compute for each $i=(k-1)n+j$, $i^\prime=(k^\prime-1)n+j^\prime$, $k \neq k^\prime$, and $j,j^\prime \in [n]$, 
\begin{align}\label{eq4:lem spec X}
-\frac{4\theta}{N} \le a_{i,i^\prime} \le \frac{4\theta}{N}. 
\end{align}
This, together with \eqref{eq2:lem spec X} and \eqref{eq3:lem spec X}, yields $\bm A = \bm I_N - \frac{1}{n} \bm I_K \otimes \bm E_n + \bm \Delta$, where $|\delta_{ij}| \le {4\theta}/{N}$. Therefore, we have
\begin{align*}
\|\bm A\| \le \left\| \bm I_N - \frac{1}{n} I_K \otimes \bm E_n \right\| + \|\bm \Delta\|_F \le 1 + 4\theta,
\end{align*}
which implies the first inequality in \eqref{eq:X Xbar}. Moreover, we have 
\begin{align}\label{eq0:lem norm X}
 \|\bm X - \bar{\bm X}\|_F^2 = \|\bm X\|_F^2 - 2\langle \bm X, \bar{\bm X} \rangle + \|\bar{\bm X}\|_F^2. 
\end{align}
Then, we bound the above terms in turn. First, noting that $\|\bm X\|_F^2 = \sum_{i=1}^N \|\bm x_i\|^2$ and using \eqref{eq:orth} in Assumption \ref{AS:1} yield 
\begin{align}\label{eq2:lem norm X}
N - \theta \le \|\bm X\|_F^2  \le N + \theta.  
\end{align}
Second, it follows from \eqref{eq:xk} that 
\begin{align*}
\langle \bm X, \bar{\bm X} \rangle = \sum_{k=1}^K \sum_{i=1}^n  \langle \bm x_{k,i}, \bar{\bm x}_k \rangle = \sum_{k=1}^K \sum_{i=1}^n \frac{1}{n} \sum_{j=1}^n  \langle \bm x_{k,i}, \bm x_{k,j} \rangle = \sum_{k=1}^K \sum_{i=1}^n \frac{1}{n}   \left( \|\bm x_{k,i}\|^2 + \sum_{j\neq i}^n \langle \bm x_{k,i}, \bm x_{k,j} \rangle \right). 
\end{align*}
This, together with Assumption \ref{AS:1}, yields
\begin{align}\label{eq3:lem norm X}
K - \theta \le \langle \bm X, \bar{\bm X} \rangle  \le K + \theta. 
\end{align}
Finally, we compute 
\begin{align*}
\|\bar{\bm X}\|_F^2 = n \sum_{k=1}^K \|\bar{\bm x}_k\|^2 = \frac{1}{n}\sum_{k=1}^K \left\| \sum_{i=1}^n \bm x_{k,i} \right\|^2 = \frac{1}{n}\sum_{k=1}^K \left(\sum_{i=1}^n \|\bm x_{k,i}\|^2 +  \sum_{i \neq j}\langle \bm x_{k,i}, \bm x_{k,j} \rangle \right) 
\end{align*}
This, together with Assumption \ref{AS:1}, yields $K - \theta \le \|\bar{\bm X}\|_F^2 \le K + \theta$. Substituting this, \eqref{eq2:lem norm X}, and \eqref{eq3:lem norm X} into \eqref{eq0:lem norm X} yields the second inequality in \eqref{eq:X Xbar}. 
\end{proof} 

Next, we present a lemma that shows that each singular value of the first $L-1$ weight matrices is exactly the same, and each singular value of the last weight matrix is approximately equal to that of the first $L-1$ weight matrices under the balancedness \eqref{eq:bala} in Assumption \ref{AS:2}.   

\begin{lemma}\label{lem:singu bala}
Suppose that the weight matrices $\{\bW_l\}_{l=1}^L$ satisfy \eqref{eq:bala} and admit the SVD in \eqref{eq:SVD Wl} and \eqref{eq:SVD WL}. Then, it holds that 
\begin{align}\label{eq:Sigma l}
\bm \Sigma_{l+1} = \bm \Sigma_l,\ \forall l \in [L-2],\ \|\bm{\Sigma}_{L}^T\bm{\Sigma}_{L}  - \bm{\Sigma}_{l}^2\|_F \le \delta,\ \forall l \in [L-1]. 
\end{align} 
\end{lemma}
\begin{proof}
 Since $\bm W_{l+1}^T \bm W_{l+1} = \bm W_{l} \bm W_{l}^T$ for all $ l \in [L-2]$ and \eqref{eq:SVD Wl}, we have
    \begin{align}\label{eq:lem singu bala}
        \bm V_{l+1} \bm \Sigma_{l+1}^2 \bm V_{l+1}^T = \bm U_{l} \bm \Sigma_{l}^2 \bm U_{l}^T.
    \end{align}
For a given $l \in [L-2]$, the two sides of the above equation are essentially eigenvalue decomposition of the same matrix. This, together with the fact that $\bm \Sigma_{l+1}$ and $\bm \Sigma_{l}$ have non-increasing and non-negative diagonal elements by \eqref{eq:SVD Wl}, yields $\bm \Sigma_{l+1} = \bm \Sigma_{l}$. Using $\|\bm W_{L}^T \bm W_{L} - \bm W_{L-1}\bm W_{L-1}^T \|_F \le \delta$, \eqref{eq:SVD Wl}, and \eqref{eq:SVD WL}, we have  
\begin{align}\label{eq:lem UV}
        \|\bm V_{L} \bm \Sigma_{L}^T \bm \Sigma_{L} \bm V_{L}^T -  \bm U_{L-1} \bm \Sigma_{L-1}^2 \bm U_{L-1}^T \|_F \le \delta. 
    \end{align}
This, together with  \cite[Lemma 4]{arora2018convergence}, implies $\|\bm{\Sigma}_{L}^T\bm{\Sigma}_{L}  - \bm{\Sigma}_{L-1}^2\|_F \le \delta$. 
According to this and $\bm \Sigma_{l+1} = \bm \Sigma_{l}$ for all $l \in [L-2]$, implies \eqref{eq:Sigma l} for all $l \in [L-1]$.
\end{proof}

Based on the above two lemmas, we are ready to estimate the top $K$ singular values of each weight matrix $\bm W_l$ for all $l \in [L]$ under Assumptions \ref{AS:1} and \ref{AS:2}.  

\begin{lemma}\label{lem:bound Wl}
Suppose that Assumption \ref{AS:1} holds, and the weight matrices $\{\bW_l\}_{l=1}^L$ satisfy \eqref{eq:min norm} and \eqref{eq:bala} with 
\begin{align}\label{eq:delta}
\delta \le  \frac{(2n)^{1/L}}{30L^2}. 
\end{align}  
Then, the following statements hold: \\
(i) It holds that
\begin{align}\label{eq:bound WL1}
\frac{\sqrt{n}}{\sqrt{1+\theta}} \le \sigma_{K}(\bW_{L:1}) \le \sigma_{1}(\bW_{L:1}) \le \frac{\sqrt{n}}{\sqrt{1-\theta}},
\end{align}
(ii) It holds that 
\begin{align}\label{eq:bound Wl}
\left( \sqrt{\frac{n}{2}}\right)^{1/L} \le \sigma_K\left(\bW_l\right)  \le \sigma_1\left(\bW_l\right) \le \left( 2\sqrt{n} \right)^{1/L},\ \forall l \in [L],
\end{align}
and 
\begin{align}
& \left( \left( \frac{1}{1+\theta} - 4\delta\right) n \right)^{1/2L} \le \sigma_K(\bW_l) \le \sigma_1(\bW_l) \le \left( \left( \frac{1}{1-\theta} + 4\delta\right) n \right)^{1/2L},\ \forall l \in [L-1], \label{eq:bound Wl1}\\
&  \left( \left( \frac{1}{1+\theta} - 4\delta\right) n \right)^{1/2L} - \sqrt{\delta}  \le \sigma_K(\bW_L) \le \sigma_1(\bW_L) \le \left( \left( \frac{1}{1-\theta} + 4\delta\right) n \right)^{1/2L} + \sqrt{\delta}. \label{eq:bound Wl2}
\end{align}
\end{lemma} 
\begin{proof}
According to \eqref{eq:spec X} in Lemma \ref{lem:spec X}, we obtain that $\bm X$ is of full column rank, and thus $\bm X^+ := (\bm X^T \bm X)^{-1}\bm X^T$ is of full row rank with $1/\sqrt{1+\theta} \le \sigma_{\min}(\bm X^+) \le \sigma_{\max}(\bm X^+) \le 1/\sqrt{1-\theta}$. This, together with $\bm Y = \bm I_K \otimes \bm{1}_n^T$, \eqref{eq:min norm}, and Lemma \ref{lem:mini singular}, yields  
\begin{align}\label{eq1:lem bound W1}
\frac{\sqrt{n}}{\sqrt{1+\theta}} \le \sigma_{K}(\bW_{L:1}) \le \sigma_{1}(\bW_{L:1}) \le \frac{\sqrt{n}}{\sqrt{1-\theta}},
\end{align}
which, together with $\theta \in [0,1/4]$, implies
\begin{align}
\frac{2\sqrt{n}}{\sqrt{5}} \le \sigma_{K}(\bW_{L:1}) \le \sigma_{1}(\bW_{L:1}) \le \frac{2\sqrt{n}}{\sqrt{3}}. 
\end{align} 
Using this, \eqref{eq:bala} with \eqref{eq:delta}, and invoking \cite[Lemma 6]{arora2018convergence}, we obtain $\sigma_1(\bm W_l) \le \left(2\sqrt{n}\right)^{1/L}$ for all $l \in [L]$. Using  $\bW_{l+1}^T\bW_{l+1} = \bW_l\bW_l^T$ for all $l \in [L-2]$ by \eqref{eq:bala}, we compute 
\begin{align}\label{eq3:lem bound W1}
    \|\bm W_{L:1}^T \bm W_{L:1}  - \left(\bm W_1^T \bm W_1\right)^L \|_F & = \|\bm W_{L:1}^T\bm W_{L:1}  - \bm W_{L-1:1}^T \bW_{L-1}\bW_{L-1}^T \bW_{L-1:1} \|_F \notag\\
    & = \|\bm W_{L-1:1}^T \left(\bW_L^T\bW_L -  \bW_{L-1}\bW_{L-1}^T\right) \bW_{L-1:1} \|_F \notag\\
    & \le \|\bW_L^T\bW_L -  \bW_{L-1}\bW_{L-1}^T\|_F \prod_{l=1}^{L-1} \|\bm W_l\|^2 \notag\\
    & \le \delta \left( 4n \right)^{(L-1)/L} \le \frac{2n}{15L^2},
\end{align}
where the second inequality uses \eqref{eq:bala}, \eqref{eq:delta}, and \eqref{eq:bound Wl} for all $l \in [L-1]$, and the last inequality follows from \eqref{eq:delta}. Using this, \eqref{eq1:lem bound W1}, and Weyl's inequality gives 
\begin{align*}
    \sigma_K\left(\left(\bm W_1^T \bm W_1\right)^L\right) & \ge \sigma_K\left(\bm W_{L:1}^T\bm W_{L:1}\right) - \|\bm W_{L:1}^T\bm W_{L:1}  - \left(\bm W_1^T \bm W_1\right)^L \|\\ 
 & \ge \frac{n}{1+\theta} - \|\bm W_{L:1}^T\bm W_{L:1}  - \left(\bm W_1^T \bm W_1\right)^L \|_F  \ge \left( \frac{1}{1+\theta} - \frac{2}{15L^2}\right) n, 
\end{align*} 
which implies $\sigma_K(\bW_1) \ge \left( \left(\frac{1}{1+\theta} - \frac{2}{15L^2}\right)n \right)^{1/2L}$. This, together with  Lemma \ref{lem:singu bala}, implies 
\begin{align}\label{eq2:lem bound W1}
\sigma_K(\bW_l) \ge \left( \left(\frac{1}{1+\theta} - \frac{2}{15L^2}\right)n \right)^{1/2L},\ \forall l \in [L-1].
\end{align}
 Using Weyl's inequality again and \eqref{eq:bala}, we have 
\begin{align*}
   \sigma_K\left( \bW_{L}^T\bW_{L} \right) & \ge  \sigma_K\left( \bW_{L-1}^T\bW_{L-1} \right) -\|\bW_{L}^T\bW_{L} - \bW_{L-1}\bW_{L-1}^T\| \\
   & \ge \left( \frac{1}{1+\theta} - \frac{2}{15L^2}  \right)^{1/L}n^{1/L} - \delta \ge \left( \frac{n}{2} \right)^{1/L},   
\end{align*}
where the last inequality follows from $\delta \le (2n)^{1/L}/(30L^2)$ and
\begin{align*}
\left( \frac{1}{1+\theta} - \frac{2}{15L^2}  \right)^{1/L} \ge \left(\frac{3}{4}\right)^{1/L}  \ge  \frac{1}{15L^2} + \left(\frac{1}{2}\right)^{1/L},\ \forall L \ge 1. 
\end{align*}
Here the first inequality uses $L \ge 2$, and the second inequality is due to $f(x) = (3/4)^x - (1/2)^x - x^2/15$ is increasing for $x \in (0,1]$ and $f(0) = 0$. Therefore, we have $\sigma_K(\bW_L) \ge \left( \frac{n}{2} \right)^{1/2L}$. This, together with \eqref{eq2:lem bound W1} and $\sigma_1(\bm W_l) \le \left(2\sqrt{n}\right)^{1/L}$ for all $l \in [L]$, yields \eqref{eq:bound Wl}.   

Next, we are devoted to proving \eqref{eq:bound Wl1} and \eqref{eq:bound Wl2}. It follows from \eqref{eq3:lem bound W1} that 
\begin{align}
\|\bm W_{L:1}^T \bm W_{L:1}  - \left(\bm W_1^T \bm W_1\right)^L \|_F \le 4 n\delta. 
\end{align}
This, together with Weyl's inequality and \eqref{eq1:lem bound W1}, yields 
\begin{align*}
\sigma_1\left(\left(\bm W_1^T \bm W_1\right)^L\right) & \le \sigma_1\left(\bm W_{L:1}^T\bm W_{L:1}\right) + \|\bm W_{L:1}^T\bm W_{L:1}  - \left(\bm W_1^T \bm W_1\right)^L \| \le \left( \frac{1}{1-\theta} + 4\delta\right) n, \\
    \sigma_K\left(\left(\bm W_1^T \bm W_1\right)^L\right) & \ge \sigma_K\left(\bm W_{L:1}^T\bm W_{L:1}\right) - \|\bm W_{L:1}^T\bm W_{L:1}  - \left(\bm W_1^T \bm W_1\right)^L \| \ge  \left( \frac{1}{1+\theta} - 4\delta\right) n. 
\end{align*} 
This, together with Lemma \ref{lem:singu bala}, directly implies \eqref{eq:bound Wl1}. Using Weyl's inequality and \eqref{eq:bala}, we have  
\begin{align*}
\sigma_1\left( \bm W_L^T \bm W_L \right) & \le \sigma_1\left(\bm W_{L-1}^T\bm W_{L-1}\right) + \|\bm W_{L}^T\bm W_{L}  - \bm W_{L-1} \bm W_{L-1}^T \| \le \left( \frac{1}{1-\theta} + 4\delta\right)^{1/L} n^{1/L} + \delta, \\
\sigma_K\left( \bm W_L^T \bm W_L \right) & \ge \sigma_K\left(\bm W_{L-1}^T\bm W_{L-1}\right) - \|\bm W_{L}^T\bm W_{L}  - \bm W_{L-1} \bm W_{L-1}^T \| \ge \left( \frac{1}{1-\theta} + 4\delta\right)^{1/L} n^{1/L} - \delta, 
\end{align*} 
which directly implies \eqref{eq:bound Wl2}. 
\end{proof}

Using this lemma and $\varepsilon \in (0,1)$, we conclude that the leading $K$ singular values of $\bW_l$ for all $l \in [L-1]$ are well separated from the remaining $d-K$ singular values according to \eqref{eq:singu}. Consequently, we can handle them with the associated singular vectors separately. Since $\bm \Sigma_{l+1} = \bm \Sigma_l$ for all $l \in [L-2]$ according to Lemma \ref{lem:singu bala}, suppose that $\bm \Sigma_{l,1}$ and $\bm\Sigma_{l,2}$ have distinct $p \ge 1$ and $q \ge 1$ singular values for all $l \in [L-1]$, respectively. Let $\sigma_1 > \dots > \sigma_p >  \sigma_{p+1} > \dots > \sigma_{p+q}$ be the distinct singular values with the corresponding multiplicities $h_1,h_2,\dots,h_{p+q} \in \mathbb{N}$. In particular, we have $\sum_{i=1}^p h_i = K$, $\sum_{i=p+1}^{p+q} h_i = d-K$, and $\sigma_{p+1} \le \varepsilon$ due to \eqref{eq:singu}. Then, we can write 
\begin{align}\label{eq:block S}
\bm{\Sigma}_{l } = \tilde{\bm\Sigma} = {\rm BlkDiag} ( \tilde{\bm\Sigma}_1, \tilde{\bm\Sigma}_2 ) = {\rm BlkDiag}\left(\sigma_{1}\bI_{h_1},\dots,\sigma_{p}\bI_{h_p},\dots,\sigma_{p+q}\bI_{h_{p+q}}\right).
\end{align}  
Using the same partition, we can also write
\begin{align}\label{eq:Ul Vl}
\bU_{l} =  \left[\bm{U}_{l}^{(1)},\dots,\bm{U}_{l}^{(p)},\dots,\bm{U}_{l}^{(p+q)}\right],\ \bV_{l} =  \left[\bm{V}_{l}^{(1)},\dots,\bm{V}_{l}^{(p)},\dots,\bm{V}_{l}^{(p+q)}\right],
\end{align}
where $\bm{U}_l^{(i)},\bm{V}_l^{(i)} \in \mO^{d\times h_i}$ for all $i \in [p+q]$. Using the above definitions, we can characterize the relationship between singular vector matrices $\bU_l$ and $\bV_{l+1}$. 

\begin{lemma}\label{lem:UV}
Suppose that Assumptions \ref{AS:1} and \ref{AS:2}  hold. Then, the following statements hold:\\
(i) For all $l \in [L-2]$, there exists orthogonal matrix $\bm Q_{l,i} \in \mO^{h_i}$ such that  
\begin{align}\label{eq:Ul=Vl}
\bm{U}_{l}^{(i)} = \bm{V}_{l+1}^{(i)}\bm Q_{l,i},\ \forall i \in [p+q].
\end{align}
(ii) It holds that 
\begin{align}\label{eq:UL=VL}
\|\bm{V}_{L,1}^T\bm{U}_{L-1,2}\|_F \le \frac{2\sqrt{\delta}\sqrt[4]{K}}{n^{1/2L}},\ \sigma_{\min}\left(\bm{V}_{L,1}^T\bm{U}_{L-1,1}\right) \ge 1 - \frac{2\sqrt{\delta}\sqrt[4]{K}}{  n^{1/2L}}.
\end{align}
\end{lemma}
\begin{proof}
(i) Using the first equation in \eqref{eq:Sigma l}, \eqref{eq:lem singu bala}, \eqref{eq:block S}, and \eqref{eq:Ul Vl}, we obtain
\begin{align*}
\sum_{i=1}^{p+q} \sigma_i^2 \bm V_{l+1}^{(i)} \bm V_{l+1}^{(i)^T} = \sum_{i=1}^{p+q} \sigma_i^2 \bm U_{l}^{(i)} \bm U_{l}^{(i)^T}
\end{align*}
Multiplying $\bm U_{l}^{(1)}$ on the both sides of the above equality, we obtain
\begin{align*}
\sum_{i=1}^{p+q} \sigma_i^2 \bm U_{l}^{(1)^T}\bm V_{l+1}^{(i)} \bm V_{l+1}^{(i)^T}\bm U_{l}^{(1)} = \sigma_1^2\bm{I}, 
\end{align*}
where the equality follows from the fact that $\bm U_l \in \mO^d$ takes the form of \eqref{eq:Ul Vl}. Therefore, for any $\bm{x} \in \R^{h_1}$ with $\|\bm{x}\|=1$, we have
\begin{align*}
\sigma_1^2 & = \sum_{i=1}^{p+q} \sigma_i^2  \left\| \bm V_{l+1}^{(i)^T}\bm U_{l}^{(1)}\bx \right\|^2 \le \sigma_1^2  \left\| \bm V_{l+1}^{(1)^T}\bm U_{l}^{(1)}\bx \right\|^2 + \sigma_2^2 \sum_{i = 2}^{p+q} \left\| \bm V_{l+1}^{(i)^T}\bm U_{l}^{(1)}\bx \right\|^2 \\
& = \sigma_1^2  \left\| \bm V_{l+1}^{(1)^T}\bm U_{l}^{(1)}\bx \right\|^2 + \sigma_2^2\left( 1 -  \left\| \bm V_{l+1}^{(1)^T}\bm U_{l}^{(1)}\bx \right\|^2  \right) = \sigma_2^2 + \left( \sigma_1^2 - \sigma_2^2\right)\left\| \bm V_{l+1}^{(1)^T}\bm U_{l}^{(1)}\bx \right\|^2,
\end{align*}
where the inequality follows from $\sigma_1 > \sigma_2 > \dots > \sigma_{p+q}$, and the second equality uses $\bm V_{l+1} \in \mO^d$. This, together with $\sigma_1 > \sigma_2 > 0$, implies $\left\| \bm V_{l+1}^{(1)^T}\bm U_{l}^{(1)}\bx \right\|^2 = 1$. Using this and $\|\bm{x}\|=1$, we have $\bm{U}_{l}^{(1)} = \bm{V}_{l+1}^{(1)}\bm Q_{l,1}$ for some $\bm Q_{l,1} \in \mO^{h_1}$. 
By repeatedly applying the same argument to $\bm U_l^{(i)}$ for all $i \ge 2$, we prove \eqref{eq:Ul=Vl}.
 
(ii) According to \eqref{eq:Sigma l} and the block structures of $\bm \Sigma_L, \bm \Sigma_{L-1}$, we obtain
\begin{align}\label{eq1:lem UV}
\|\bm{\Sigma}_{L-1,2}^2\|_F \le \delta. 
\end{align}
In addition, it follows from \eqref{eq:SVD Wl}, \eqref{eq:SVD WL}, and \eqref{eq:lem UV} that 
\begin{align*}
\|\bV_{L,1}\bm{\Sigma}^2_{L,1}\bV_{L,1}^T  - \bU_{L-1,1}\bm{\Sigma}^2_{L-1,1}\bU_{L-1,1}^T - \bU_{L-1,2}\bm{\Sigma}_{L-1,2}^2\bU_{L-1,2}^T\|_F \le \delta.  
\end{align*}
Using this, $\|\bm U^T \bm A \bm U\|_F \le \|\bm A \|_F$ for any $\bm U \in \mO^{d\times (d-K)}$, we further obtain  
    \begin{align*}
    \|\bU_{L-1,2}^T\bV_{L,1}\bm{\Sigma}^2_{L,1}\bV_{L,1}^T\bU_{L-1,2}   -  \bm{\Sigma}_{L-1,2}^2 \|_F \le \delta. 
    \end{align*} 
    This, together with \eqref{eq1:lem UV}, implies 
    \begin{align*}
    \|\bU_{L-1,2}^T\bV_{L,1}\bm{\Sigma}^2_{L,1}\bV_{L,1}^T\bU_{L-1,2}  \|_F \le 2\delta. 
    \end{align*}
Using Lemma \ref{lem:singular AA} with the fact that $\bm{\Sigma}_{L,1}\bV_{L,1}^T\bU_{L-1,2}$ is of rank $K$, we further have
    \begin{align}\label{eq3:lem UV}
    \|\bm{\Sigma}_{L,1}\bV_{L,1}^T\bU_{L-1,2}\|_F^2 \le  2\delta \sqrt{K}. 
    \end{align}
    Noting that $\|\bm{\Sigma}_{L,1}\bV_{L,1}^T\bU_{L-1,2}\|_F \ge \sigma_{\min}(\bm{\Sigma}_{L,1})\|\bV_{L,1}^T\bU_{L-1,2}\|_F$, we have 
    \begin{align}\label{eq4:lem UV}
        \|\bV_{L,1}^T\bU_{L-1,2}\|_F  \le \frac{\|\bm{\Sigma}_{L,1}\bV_{L,1}^T\bU_{L-1,2}\|_F}{\sigma_{\min}(\bm{\Sigma}_{L,1})} \le \frac{\sqrt{2\delta}\sqrt[4]{K}}{\left( n/2 \right)^{1/2L}} \le  \frac{2\sqrt{\delta}\sqrt[4]{K}}{  n^{1/2L}},
    \end{align}
    where the second inequality follows from \eqref{eq3:lem UV} and Lemma \ref{lem:bound Wl}. Now, we compute
\begin{align*}
\sigma_{\min}^2(\bV_{L,1}^T\bU_{L-1,1}) & =  \min_{\|\bx\|=1} \|\bU_{L-1,1}^T\bV_{L,1}\bx\|^2  =    \min_{\|\bx\|=1} \bx^T \bV_{L,1}^T\bU_{L-1,1}\bU_{L-1,1}^T\bV_{L,1}\bx  \\
& = \min_{\|\bx\|=1} \bx^T \bV_{L,1}^T\left(\bI - \bU_{L-1,2}\bU_{L-1,2}^T\right)\bV_{L,1}\bx = 1 - \max_{\|\bx\|=1} \|\bU_{L-1,2}^T\bV_{L,1}\bx\|^2 \\
& = 1 - \sigma_{\max}^2(\bU_{L-1,2}^T\bV_{L,1}) \ge 1 - \|\bU_{L-1,2}^T\bV_{L,1}\|_F^2,
\end{align*}
where the third equality follows from $ \bU_{L-1} \in \mO^d$. This, together with \eqref{eq4:lem UV}, implies
\begin{align*}
\sigma_{\min}(\bV_{L,1}^T\bU_{L-1,1}) \ge \sqrt{ 1 - \|\bU_{L-1,2}^T\bV_{L,1}\|_F^2} \ge 1 - \|\bU_{L-1,2}^T\bV_{L,1}\|_F \ge 1 - \frac{2\sqrt{\delta}\sqrt[4]{K}}{  n^{1/2L}}. 
\end{align*}
\end{proof} 

\subsection{Analysis of Progressive Within-Class Compression}\label{subsec:compre}

In this subsection, our goal is to prove progressive compression of within-class features across layers. Towards this goal, we study the behavior of $\mathrm{Tr}(\bm \Sigma_W)$ and $\mathrm{Tr}(\bm \Sigma_B)$ across layers, respectively. For ease of exposition, let 
\begin{align} 
& \bm{\Delta}_W := \left[\bm{\delta}_{1,1},\dots,\bm{\delta}_{1,n},\dots,\bm{\delta}_{K,1},\dots,\bm{\delta}_{K,n} \right] \in \R^{d\times N},\ \text{where}\ \bm{\delta}_{k,i} = \bm{x}_{k,i}-\bar{\bx}_k,\ \forall k,i, \label{eq:Delta W} \\
& \bm{\Delta}_B := \left[\bar{\bm{\delta}}_1,\dots,\bar{\bm{\delta}}_K \right] \in \R^{d\times K},\ \text{where}\ \bar{\bm{\delta}}_k = \bar{\bx}_k - \bar{\bx},\ \forall k \in [K], \label{eq:Delta B} 
\end{align} 
where $\bar{\bm x}_k$ and $\bar{\bm x}$ are defined in \eqref{eq:xk}. According to Definition \ref{def:nc1} and $n_k=n$ for all $k \in [K]$, one can verify 
\begin{align}\label{eq:tr WB}
\mathrm{Tr}(\bm \Sigma_W^l) = \frac{1}{N} \|\bm W_{l:1} \bm{\Delta}_W \|_F^2,\quad \mathrm{Tr}(\bm \Sigma_B^l) = \frac{1}{K} \|\bm W_{l:1} \bm{\Delta}_B \|_F^2. 
\end{align}
Recall from \eqref{eq:block S} that  $\tilde{\bm \Sigma} = {\rm BlkDiag}(\tilde{\bm \Sigma}_1, \tilde{\bm \Sigma}_2)$, where
\begin{align}\label{eq:Sigma 12}
\tilde{\bm \Sigma}_1 = 
{\rm BlkDiag}\left(\sigma_{1}\bI_{h_1},\dots,\sigma_{p}\bI_{h_p}\right),\ \tilde{\bm \Sigma}_2 = {\rm BlkDiag}\left(\sigma_{p+1}\bI_{h_{p+1}},\dots,\sigma_{p+q}\bI_{h_{p+q}}\right)
\end{align}
According to \eqref{eq:SVD Wl}, \eqref{eq:block S}, and \eqref{eq:Ul=Vl}, there exist block diagonal matrices $\bm{O}_{l,1} = \mathrm{Blkdiag}(\bm{Q}_{l,1},\dots,\bm{Q}_{l,p}) \in \R^{K\times K}$ and $\bm{O}_{l,2} = \mathrm{Blkdiag}(\bm{Q}_{l,p+1},\dots,\bm{Q}_{l,p+q}) \in \R^{(d-K)\times (d-K)}$ with $\bm{Q}_{l,i} \in \mO^{h_i}$ for all $i \in [p+q]$ such that for each $l \in [L-1]$,
\begin{align}\label{eq:Wl1}
\bm W_{l:1}  = \bm{U}_{l,1} \bm{O}_{l,1} \tilde{\bm \Sigma}^{l}_1 \bm{V}_{1,1}^T + \bm{U}_{l,2} \bm{O}_{l,2} \tilde{\bm \Sigma}^{l}_2 \bm{V}_{1,2}^T = \bm{U}_{l,1} \tilde{\bm \Sigma}^{l}_1\bm{O}_{l,1}  \bm{V}_{1,1}^T + \bm{U}_{l,2}  \tilde{\bm \Sigma}^{l}_2\bm{O}_{l,2} \bm{V}_{1,2}^T,
\end{align}
where $\bm U_{l,1}$, $\bm U_{l,2}$, $\bm V_{l,1}$, and $\bm V_{l,2}$ for ell $l \in [L-1]$ are defined in \eqref{eq:SVD Wl}. 

To bound $\mathrm{Tr}(\bm \Sigma_W^l)$ and $\mathrm{Tr}(\bm \Sigma_B^l)$ across layers, we present a lemma that establishes a relationship between the magnitudes of $\bm V_{1,1}^T\bm{\Delta}_W$ and $\bm V_{1,2}^T\bm{\Delta}_W$, as well as a relationship between the magnitudes of $\bm V_{1,1}^T\bm{\Delta}_B$ and $\bm V_{1,2}^T\bm{\Delta}_B$ based on the above setup. 

\begin{lemma}\label{lem:V DeltaW}
Suppose that Assumptions \ref{AS:1} and \ref{AS:2} hold with 
\begin{align}\label{eq:delta1}
\delta \le \frac{n^{1/L}}{64\sqrt{K}},\ n \ge 16. 
\end{align}
Then, the following statements hold: \\
(i) It holds that
\begin{align}\label{eq:V DeltaW}
\|\bm V_{1,1}^T\bm{\Delta}_W\|_F \le \frac{2\sqrt{2\delta}\sqrt[4]{K}\varepsilon^{L-1}}{\sqrt{n}} \left\|\bm{V}_{1,2}^T\bm{\Delta}_W\right\|_F. 
\end{align}
(ii) It holds that 
\begin{align}\label{eq:V DeltaB}
\|\bm{V}_{1,2}^T\bm \Delta_B\|_F \le 2\theta \|\bm V_{1,1}^T \bm \Delta_B\|_F. 
\end{align}
\end{lemma}
\begin{proof}
(i) It follows from \eqref{eq:min norm} that $\bm W_{L:1}\bm X = \bm Y$. This, together with $\bm Y = \bm I_K \otimes \bm{1}_n$ and \eqref{eq:xk}, implies $\bm{W}_{L:1} \bx_{k,i} = \bm{e}_k$ and $\bm{W}_{L:1}\bar{\bx}_k = \bm{e}_k$ for all $i \in [n]$ and $k \in [K]$. It follows from this and \eqref{eq:Delta W} that
\begin{align}\label{eq:W DeltaW}
\bm{W}_{L:1}\bm{\Delta}_W = \bm{0}. 
\end{align}  
Using this and \eqref{eq:Wl1},  there exist block diagonal matrices $\bm{O}_{1} = \mathrm{BlkDiag}(\bm{Q}_{1},\dots,\bm{Q}_{p})$ $\in \R^{K\times K}$ and $\bm{O}_{2} = \mathrm{BlkDiag}(\bm{Q}_{p+1},\dots,\bm{Q}_{p+q}) \in \R^{(d-K)\times (d-K)}$ with $\bm{Q}_{i} \in \mO^{h_i}$ for all $i \in [p+q]$ such that  
\begin{align*}
\bm U_L \bm{\Sigma}_{L,1} \bm V_{L,1}^T \left(\bm{U}_{L-1,1}\tilde{\bm \Sigma}_1^{L-1}\bm{O}_{1}\bm{V}_{1,1}^T + \bm{U}_{L-1,2}\tilde{\bm \Sigma}_2^{L-1}\bm{O}_{2}\bm{V}_{1,2}^T\right)\bm{\Delta}_W = \bm{0}.
\end{align*}
This, together with $\bm{U}_L \in \mO^K$ and \eqref{eq:bound Wl} in Lemma \ref{lem:bound Wl}, yields 
\begin{align}\label{eq0:W DeltaW}
\bm V_{L,1}^T \left(\bm{U}_{L-1,1}\tilde{\bm \Sigma}_1^{L-1}\bm{O}_{1}\bm{V}_{1,1}^T + \bm{U}_{L-1,2}\tilde{\bm \Sigma}_2^{L-1}\bm{O}_{2}\bm{V}_{1,2}^T\right)\bm{\Delta}_W = \bm{0}.
\end{align}
Therefore, we have 
\begin{align}\label{eq1:lem V DeltaW}
\left\|\bm V_{L,1}^T  \bm{U}_{L-1,1}\tilde{\bm \Sigma}_1^{L-1}\bm{O}_{1}\bm{V}_{1,1}^T\bm{\Delta}_W\right\|_F & = \left\| \bm V_{L,1}^T\bm{U}_{L-1,2}\tilde{\bm \Sigma}_2^{L-1}\bm{O}_{2}\bm{V}_{1,2}^T\bm{\Delta}_W\right\|_F \notag\\
& \le \| \bm V_{L,1}^T\bm{U}_{L-1,2}\|_F \|\tilde{\bm \Sigma}_2^{L-1}\| \left\|\bm{V}_{1,2}^T\bm{\Delta}_W\right\|_F \notag\\
& \le \frac{\sqrt{2\delta}\sqrt[4]{K}}{\left( n/2 \right)^{1/2L}} \varepsilon^{L-1} \left\|\bm{V}_{1,2}^T\bm{\Delta}_W\right\|_F,
\end{align}
where the last inequality follows from \eqref{eq:UL=VL} in Lemma \ref{lem:UV} and $\|\tilde{\bm \Sigma}_2\|\le \varepsilon$ due to \eqref{eq:singu} and \eqref{eq:Sigma 12}. On the other hand, we compute
\begin{align*}
\left\|\bm V_{L,1}^T  \bm{U}_{L-1,1}\tilde{\bm \Sigma}_1^{L-1}\bm{O}_{1}\bm{V}_{1,1}^T\bm{\Delta}_W\right\|_F & \ge \sigma_{\min}\left( \bm V_{L,1}^T  \bm{U}_{L-1,1} \right) \sigma_{\min}\left( \tilde{\bm \Sigma}_1^{L-1} \right)\left\|\bm{V}_{1,1}^T\bm{\Delta}_W\right\|_F \\
& \ge \left( 1 - \frac{\sqrt{2\delta}\sqrt[4]{K}}{\left( n/2 \right)^{1/2L}}\right) \left(\frac{n}{2}\right)^{\frac{L-1}{2L}}\left\|\bm{V}_{1,1}^T\bm{\Delta}_W\right\|_F,
\end{align*}
where the first inequality uses Lemma \ref{lem:mini singular}, and the last inequality follows from \eqref{eq:bound Wl} in Lemma \ref{lem:bound Wl} and \eqref{eq:UL=VL} in Lemma \ref{lem:UV}. This, together with   \eqref{eq1:lem V DeltaW}, yields  
\begin{align*}
\left\|\bm{V}_{1,1}^T\bm{\Delta}_W\right\|_F \le  \frac{\frac{\sqrt{2\delta}\sqrt[4]{K}}{\left( n/2 \right)^{1/2L}}\varepsilon^{L-1}}{\left( 1 - \frac{\sqrt{2\delta}\sqrt[4]{K}}{\left( n/2 \right)^{1/2L}}\right) \left(\frac{n}{2}\right)^{\frac{L-1}{2L}}} \left\|\bm{V}_{1,2}^T\bm{\Delta}_W\right\|_F \le \frac{2\sqrt{2\delta}\sqrt[4]{K}\varepsilon^{L-1}}{\sqrt{n}} \left\|\bm{V}_{1,2}^T\bm{\Delta}_W\right\|_F, 
\end{align*}
where the second inequality follows from  $(n/2)^{1/2L} - \sqrt{2\delta}\sqrt[4]{K} \ge  n^{1/2L}/\sqrt{2} $ due to \eqref{eq:delta1} and $n \ge 1$. 

 (ii) For ease of exposition, let 
\begin{align}\label{eq:P}
\bm{P}_1 := \bm{U}_{L-1,1}^T\bm{V}_{L,1} \in \R^{K\times K},\ \bm{P}_2 := \bm{U}_{L-1,2}^T\bm{V}_{L,1} \in \R^{(d-K)\times K}. 
\end{align}
According to \eqref{eq:Wl1},  there exist block diagonal matrices $\bm{O}_{1} = \mathrm{BlkDiag}(\bm{Q}_{1},\dots,\bm{Q}_{p})$ $\in \R^{K\times K}$ and $\bm{O}_{2} = \mathrm{BlkDiag}(\bm{Q}_{p+1},\dots,\bm{Q}_{p+q}) \in \R^{(d-K)\times (d-K)}$ with $\bm{Q}_{i} \in \mO^{h_i}$ for all $i \in [p+q]$ such that  
\begin{align*}
\bm{W}_{L-1:1} = \bm{U}_{L-1,1} \tilde{\bm \Sigma}^{L-1}_1\bm{O}_{1} \bm{V}_{1,1}^T + \bm{U}_{L-1,2}  \tilde{\bm \Sigma}^{L-1}_2\bm{O}_{2} \bm{V}_{1,2}^T. 
\end{align*}  
This, together with \eqref{eq:SVD WL} and \eqref{eq:P}, yields 
\begin{align}\label{eq4:lem V DeltaB}
\bm{W}_{L:1}^T\bm{W}_{L:1} & = \bm{W}_{L-1:1}^T \bW_L^T\bW_L \bm{W}_{L-1:1} = \underbrace{\bm V_{1,1}\bm{O}_{1}^T\tilde{\bm \Sigma}_1^{L-1}\bm P_1 \bm{\Sigma}_{L,1}^2 \bm P_1^T \tilde{\bm \Sigma}_1^{L-1} \bm{O}_{1}\bm V_{1,1}^T}_{=: \bm \Delta_1}  + \notag \\
& \underbrace{\bm V_{1,1}\bm{O}_{1}^T\tilde{\bm \Sigma}_1^{L-1}\bm P_1 \bm{\Sigma}_{L,1}^2 \bm P_2^T \tilde{\bm \Sigma}_2^{L-1} \bm{O}_{2}\bm V_{1,2}^T}_{=: \bm \Delta_2} + \underbrace{\bm V_{1,2}\bm{O}_2^T\tilde{\bm \Sigma}_2^{L-1}\bm P_2 \bm{\Sigma}_{L,1}^2 \bm P_1^T \tilde{\bm \Sigma}_1^{L-1} \bm{O}_1\bm V_{1,1}^T}_{=:\bm \Delta_3} + \notag\\
&  \underbrace{\bm V_{1,2}\bm{O}_2^T\tilde{\bm \Sigma}_2^{L-1}\bm P_2 \bm{\Sigma}_{L,1}^2 \bm P_2^T \tilde{\bm \Sigma}_2^{L-1} \bm{O}_2\bm V_{1,2}^T}_{=:\bm \Delta_4}. 
\end{align}
Then, we bound the above terms in turn. First, we compute
\begin{align}\label{eq2:lem V DeltaB}
\|\bm \Delta_2\|_F = \|\bm \Delta_3\|_F  & \le   \|\tilde{\bm \Sigma}_1^{L-1}\|\|\bm P_1\| \|\bm{\Sigma}_{L,1}^2\| \|\tilde{\bm \Sigma}_2^{L-1}\|\|\bm P_2\|_F\notag \\
&  \le \left( \sqrt{2n}\right)^{\frac{L+1}{L}} \varepsilon^{L-1} \frac{\sqrt{2\delta}\sqrt[4]{K}}{(n/2)^{1/2L}} \le 2^{1+\frac{1}{L}} \varepsilon^{L-1}\sqrt{n\delta}\sqrt[4]{K},
\end{align} 
where second inequality follows from \eqref{eq:singu}, \eqref{eq:bound Wl}, \eqref{eq:UL=VL}, and \eqref{eq:P}. Using the same argument, we compute
\begin{align}\label{eq3:lem V DeltaB}
\|\bm \Delta_4\|_F \le \|\bm{\Sigma}_{L,1}^2\| \|\tilde{\bm \Sigma}_2^{L-1}\|^2\|\bm P_2\|_F^2  \le (2n)^{\frac{1}{L}}\varepsilon^{2(L-1)} \frac{2\delta\sqrt{K}}{(n/2)^{1/L}} \le 2^{1+\frac{2}{L}}\varepsilon^{2(L-1)}\delta\sqrt{K}.
\end{align}
Before we proceed further, let
\begin{align}\label{eq:A}
\bm{A}_1 = \bm{O}_1^T\tilde{\bm \Sigma}_1^{L-1}\bm P_1 \bm{\Sigma}_{L,1}.
\end{align}
It follows from \eqref{eq:UL=VL} in Lemma \ref{lem:UV} and \eqref{eq:P} that 
\begin{align}
\sigma_{\min}\left(\bm P_1 \right) \ge 1 - \frac{\sqrt{2\delta}\sqrt[4]{K}}{(n/2)^{1/2L}} \ge \frac{3}{4},
\end{align}
where the second inequality uses \eqref{eq:delta1}.  Using $\bm{O}_1 \in \mO^K$, we have 
\begin{align*}
\sigma_{\min}\left(\bm A_1 \right) & = \sigma_{\min}\left(\tilde{\bm \Sigma}_1^{L-1}\bm P_1 \bm{\Sigma}_{L,1} \right) \ge \sigma_{\min}\left(\tilde{\bm \Sigma}_1^{L-1}\right) \sigma_{\min}\left(\bm P_1 \right) \sigma_{\min}\left(\bm{\Sigma}_{L,1} \right) \ge \frac{3}{4}\sqrt{\frac{n}{2}},
\end{align*}
where the first inequality uses Lemma \ref{lem:mini singular} and the second inequality follows from \eqref{eq:bound Wl} in Lemma \ref{lem:bound Wl}. This implies that $\bm{A}_1\bm{A}_1^T \in \R^{K\times K}$ is of full rank, and thus $\bm \Delta_1 = \bm V_{1,1}\bm{A}_1\bm{A}_1^T\bm V_{1,1}^T$ is of rank $K$. Then, let 
\begin{align}\label{eq6:lem V DeltaB}
\bm{\Delta}_1 = \bm{U}_{\Delta_1}\bm{\Sigma}_{\Delta_1}\bm{U}_{\Delta_1}^T  
\end{align}
be a compact eigenvalue decomposition of $\bm{\Delta}_1$, where $\bm{U}_{\Delta_1} \in \mO^{d\times K}$ and $\bm{\Sigma}_{\Delta_1} \in \R^{K\times K}$.  Using this and Lemma \ref{lem:QQ=UU} with $\bm{\Delta}_1 = \bm{V}_{1,1}\bm{A}_1\bm{A}_1^T\bm{V}_{1,1}^T$ by \eqref{eq4:lem V DeltaB}, we obtain  
\begin{align}\label{eq5:lem V DeltaB}
\bm{V}_{1,1} \bm{V}_{1,1}^T = \bm{U}_{\Delta_1,1}\bm{U}_{\Delta_1,1}^T. 
\end{align}
Noting that $\bm{W}_{L:1}$ is of rank $K$, let 
\begin{align}\label{eq8:lem V DeltaB}
\bm{W}_{L:1}^T\bm{W}_{L:1} =   \bm{U}_{1}\bm{\Sigma}_{1}\bm{U}_{1}^T,
\end{align}
be a compact eigenvalue decomposition of $\bm{W}_{L:1}^T\bm{W}_{L:1} $, where  $\bU_{1} \in \R^{d\times K}$ and $\bm{\Sigma}_{1} \in \R^{K\times K}$. According to \eqref{eq:orth} and \eqref{eq:spec X}, we have $\sigma_K(\bm{W}_{L:1}) \ge \sqrt{n/(1-\theta)}$. This, together with \eqref{eq4:lem V DeltaB}, \eqref{eq6:lem V DeltaB}, and Davis-Kahan Theorem (see, e.g., \cite[Theorem V.3.6]{stewart1990matrix}), yields  
\begin{align*}
\| \bm{U}_{\Delta_1,1} \bm{U}_{\Delta_1,1}^T - \bm{U}_1\bm{U}_1^T \|_F & \le \frac{\|\bm{\Delta}_2 + \bm{\Delta}_3 + \bm{\Delta}_4\|_F}{n(1-\theta)} \le \frac{ 4\varepsilon^{L-1}\sqrt{2n\delta}\sqrt[4]{K} + 4\varepsilon^{2(L-1)}\delta\sqrt{K}}{n(1-\theta)}\\
& \le \frac{8\varepsilon^{L-1}\sqrt{\delta}\sqrt[4]{K}}{\sqrt{n}},
\end{align*}
where the second inequality uses \eqref{eq2:lem V DeltaB} and \eqref{eq3:lem V DeltaB}, and the last inequality follows from $4\varepsilon^{L-1}\sqrt{\delta}\sqrt[4]{K} \le 4\sqrt{\delta}\sqrt[4]{K} \le (6-4\sqrt{2})\sqrt{n}$ due to $\varepsilon \le 1$ and \eqref{eq:delta1}. Using this and \eqref{eq5:lem V DeltaB} yields
\begin{align}\label{eq11:lem V DeltaB}
\| \bm{V}_{1,1} \bm{V}_{1,1}^T - \bm{U}_1\bm{U}_1^T \|_F \le \frac{8\varepsilon^{L-1}\sqrt{\delta}\sqrt[4]{K}}{\sqrt{n}}. 
\end{align}
It follows from \eqref{eq:min norm} that $\bm W_{L:1}\bm X = \bm Y$. This, together with \eqref{eq:X Y} and \eqref{eq:Delta B}, implies $\bm W_{L:1} \bm \Delta_B = \bm I_K - \bm E_K/K$. According to this and \eqref{eq:min norm}, we compute
\begin{align}\label{eq7:lem V DeltaB}
\bm W_{L:1}^T\bm W_{L:1} \bm \Delta_B = \bX(\bm X^T\bm X)^{-1}\bY^T\left( \bm I_K - \frac{\bm E_K}{K} \right) = \bX(\bm X^T\bm X)^{-1}\left( (\bm I_K - \bm E_K/K) \otimes \bm{1}_n \right)
\end{align} 
Noting that $\bm X \in \R^{d\times N}$ is of full column rank due to \eqref{eq:spec X} in Lemma \ref{lem:spec X}, let $\bm X = \bm U_X \bm \Sigma_X\bm V_X^T$ be an SVD of $\bm X$, where $\bm U_X \in \mO^{d\times N}$, $\bm \Sigma_X \in \R^{N\times N}$, and $\bm V_X \in \mO^{N\times N}$. Then, we obtain $\bm X(\bm X^T\bm X)^{-1} = \bm U_X \bm \Sigma_X^{-1} \bm V_X^T$.  Using this, \eqref{eq7:lem V DeltaB}, and $\bm \Delta_B = \bX\left( (\bm I_K -  {\bm E_K}/{K}) \otimes \bm 1_n \right)/n$, we compute
\begin{align}\label{eq9:lem V DeltaB}
\left\| \frac{1}{n}\bm W_{L:1}^T\bm W_{L:1} \bm \Delta_B - \bm \Delta_B \right\|_F & = \frac{1}{n}\left\| \bm U_X \left( \bm \Sigma_X^{-1} - \bm \Sigma_X \right) \bm V_X^T \left( (\bm I_K - \bm E_K/K) \otimes \bm{1}_n \right) \right\|_F \notag\\ 
& = \frac{1}{n} \left\| \left( \bm \Sigma_X^{-1} - \bm \Sigma_X \right)\bm \Sigma_X^{-1}\bm \Sigma_X \bm V_X^T \left( (\bm I_K - \bm E_K/K) \otimes \bm{1}_n \right) \right\|_F \notag\\
& \le \frac{1}{n}\left\|  \bm \Sigma_X^{-2} - \bm I \right\| \|\bm \Delta_B\|_F \le \frac{2\theta}{n}\|\bm \Delta_B\|_F,
\end{align}
where the last inequality follows from $\left\|  \bm \Sigma_X^{-2} - \bm I \right\| \le 2\theta$ due to \eqref{eq:spec X} in Lemma \ref{lem:spec X} and $\theta \le 1/4$. According to \eqref{eq8:lem V DeltaB}, we have
\begin{align}\label{eq10:lem V DeltaB}
\left\| \frac{1}{n}\bm W_{L:1}^T\bm W_{L:1} - \bm U_1\bm U_1^T \right\| = \left\|  \frac{\bm \Sigma_1}{n} - \bm I \right\| \le \frac{\theta}{1-\theta},
\end{align}
where  the inequality uses \eqref{eq:bound WL1} in Lemma \ref{lem:bound Wl}. Since $\bm V_1 = [\bm V_{1,1}\ \bm V_{1,2}] \in \mO^d$, we compute
\begin{align*}
\|\bm V_{1,2}^T\bm \Delta_B\|_F & =  \|(\bm I - \bm V_{1,1}\bm V_{1,1}^T) \bm \Delta_B\|_F \le  \left\| \left(\bm I - \frac{1}{n}\bm W_{L:1}^T\bm W_{L:1}\right) \bm \Delta_B \right\|_F + \\
&\quad \left\| \left( \frac{1}{n}\bm W_{L:1}^T\bm W_{L:1} - \bm U_1\bm U_1^T\right)\bm \Delta_B\right\|_F + \left\|\left( \bm{V}_{1,1} \bm{V}_{1,1}^T - \bm{U}_1\bm{U}_1^T \right) \bm \Delta_B \right\|_F \\
& \le \left( \frac{2\theta}{n} + \frac{\theta}{1-\theta} + \frac{8\varepsilon^{L-1}\sqrt{\delta}\sqrt[4]{K}}{\sqrt{n}} \right) \|\bm \Delta_B\|_F \le \frac{3}{2}\theta \|\bm \Delta_B\|_F, 
\end{align*}
where the second inequality follows from \eqref{eq11:lem V DeltaB}, \eqref{eq9:lem V DeltaB}, and \eqref{eq10:lem V DeltaB}, the last inequality uses $\theta \le 1/4$ and $n \ge 16$. This, together with $\|\bm \Delta_B\|_F^2 = \|\bm V_{1,1}^T \bm \Delta_B\|_F^2 + \|\bm V_{1,2}^T \bm \Delta_B\|_F^2$, yields 
\begin{align*}
 \|\bm V_{1,2}^T \bm \Delta_B\|_F^2 \le \frac{9\theta^2/4}{1 - 9\theta^2/4}  \|\bm V_{1,1}^T \bm \Delta_B\|_F^2,
\end{align*}
which, together with $\theta \le 1/4$, implies \eqref{eq:V DeltaB}. 
\end{proof}

Before we proceed, we make some remarks on the above lemma. According to \eqref{eq:tr WB}, \eqref{eq:Wl1}, and \eqref{eq:V DeltaW}, it is evident that the leading $K$ singular values with the associated singular vectors of each weight matrix $\bm W_l$ for all $l \in [L-1]$ contribute minimally to the value of $\mathrm{Tr}(\bm \Sigma_W^l)$. Conversely, the remaining singular values with the associated singular vectors play a dominant role in determining the value of $\mathrm{Tr}(\bm \Sigma_W^l)$. In contrast, the primary influence on the value of $\mathrm{Tr}(\bm \Sigma_B^l)$ comes from the leading $K$ singular values along with their corresponding singular vectors, while the impact  of the remaining ones is relatively less significant. 

With the above preparations, we are ready to prove progressive within-class compression \eqref{eq1:compres} and \eqref{eq2:compres} as presented in Theorem \ref{thm:NC}. We first provide a formal version of \eqref{eq1:compres} and \eqref{eq2:compres} and their proof. The main idea of the proof is to compute the upper and lower bounds of $\mathrm{Tr}(\bm \Sigma_W^{l+1})/\mathrm{Tr}(\bm \Sigma_W^{l})$ and $\mathrm{Tr}(\bm \Sigma_B^{l+1})/\mathrm{Tr}(\bm \Sigma_B^{l})$, respectively.

\begin{theorem}\label{thm:comp}
Suppose that Assumptions \ref{AS:1} and \ref{AS:2} hold with $\delta,\varepsilon,\rho$ satisfying \eqref{eq:delta2}. Then, we have 
\begin{align}
& \frac{c_1}{c_2\left((2n)^{1/L} + \varepsilon^2 \right)} \varepsilon^2 \le \frac{C_1}{C_0} \le \frac{2c_2}{c_1(n/2)^{1/L}} \varepsilon^2,  \label{eq0:ratio compre}\\
& \frac{c_1}{c_2\left((2n)^{1/L} + n^{-1/L}\right)} \varepsilon^2  \le \frac{C_{l+1}}{C_{l}} \le \frac{c_2\left( 1 + n^{-1/L} \right)}{c_1(n/2)^{1/L}}\varepsilon^2,\ \forall l \in [L-2], \label{eq1:ratio compre}
\end{align} 
where 
\begin{align}\label{eq:c12}
c_1 = \frac{(n-3)K-1}{(n-1)K+1},\ c_2 = \frac{ \left(1+n^{-1/L}\right)}{\left(1-{\varepsilon^{2(L-1)}}n^{-1}\right)\left(1 - n^{-1/2}\right)}.  
\end{align}
\end{theorem}
\begin{proof}
For ease of exposition, let $\bm v_i \in \R^d$ denote the $i$-th column of $\bm{V}_1\in \mO^d$. Therefore, we write
\begin{align}\label{eq:V1}
\bm{V}_1 = \left[\bm{V}_{1,1}, \bm{V}_{1,2} \right],\ \bm{V}_{1,1} = \left[\bm{v}_{1},\dots, \bm{v}_{K} \right],\ \bm{V}_{1,2} = \left[\bm{v}_{K+1},\dots, \bm{v}_{d} \right]. 
\end{align}  
According to \eqref{eq:Wl1},  there exist block diagonal matrices $\bm{O}_{l,1} = \mathrm{BlkDiag}(\bm{Q}_{l,1},\dots,\bm{Q}_{l,p})$ $\in \R^{K\times K}$ and $\bm{O}_{l,2} = \mathrm{BlkDiag}(\bm{Q}_{l,p+1},\dots,\bm{Q}_{l,p+q}) \in \R^{(d-K)\times (d-K)}$ with $\bm{Q}_{l,i} \in \mO^{h_i}$ for all $i \in [p+q]$ such that  
\begin{align}\label{eq0:thm compre}
\bm{W}_{l:1} = \bm{U}_{l,1} \bm{O}_{l,1} \tilde{\bm \Sigma}^{l}_{1} \bm{V}_{1,1}^T + \bm{U}_{l,2}\bm{O}_{l,2} \tilde{\bm \Sigma}^{l}_2 \bm{V}_{1,2}^T,\ \forall l \in [L-1]. 
\end{align}   
This implies for all $l \in [L-1]$, 
\begin{align}\label{eq1:thm compre}
	\|\bm{W}_{l:1}\bm\Delta_W\|_F^2 &= \|\tilde{\bm \Sigma}^{l}_1 \bm{V}_{1,1}^T\bm\Delta_W\|_F^2 + \|\tilde{\bm \Sigma}^{l}_2\bm{V}_{1,2}^T\bm\Delta_W\|_F^2 \notag\\
	& \le \|\tilde{\bm \Sigma}^{l}_1\|^2 \|\bm{V}_{1,1}^T\bm\Delta_W\|_F^2 + \|\tilde{\bm \Sigma}^{l}_2\|^2\|\bm{V}_{1,2}^T\bm\Delta_W\|_F^2 \notag \\
	&\le \left( 2n \right)^{l/L} \frac{8\delta\sqrt{K}\varepsilon^{2(L-1)}}{n} \left\|\bm{V}_{1,2}^T\bm{\Delta}_W\right\|_F^2 + \varepsilon^{2l} \|\bm{V}_{1,2}^T\bm\Delta_W\|_F^2 \notag \\
	& \le \left( \frac{16\delta\sqrt{K}}{n^{1/L}} \varepsilon^{2(L-1)}  + \varepsilon^{2l} \right)(N-K+4\theta)\notag  \\
	& \le \left(1 + \frac{1}{n^{1/L}}\right) \varepsilon^{2l} (N-K+1),
	\end{align}
where the second inequality uses \eqref{eq:singu}, \eqref{eq:bound Wl} in Lemma \ref{lem:bound Wl}, and \eqref{eq:V DeltaW} in Lemma \ref{lem:V DeltaW}, the third inequality follows from $l \in [L-1]$ and $\|\bm{V}_{1,2}^T\bm\Delta_W\|_F^2 \le \|\bm\Delta_W\|_F^2 \le N-K+4\theta$ due to $\|\bm{V}_{1,2}\| \le 1$, \eqref{eq:X Xbar}, and \eqref{eq:Delta W}, and the last inequality uses $\theta \le 1/4$ and $\delta \le 1/(16\sqrt{K})$ due to \eqref{eq:delta2}. According to $\bm V_1 \in \mO^d$ and \eqref{eq:V DeltaW}, we have
\begin{align*}
\|\bm \Delta_W\|_F^2 & = \|\bm{V}_1^T\bm \Delta_W\|_F^2 = \|\bm{V}_{1,1}^T\bm \Delta_W\|_F^2 + \|\bm{V}_{1,2}^T\bm \Delta_W\|_F^2 \\
& \le \left( 1 + \frac{8\delta\sqrt{K}\varepsilon^{2(L-1)}}{n}\right)\|\bm{V}_{1,2}^T\bm \Delta_W\|_F^2 \\
& \le \left( 1 + \frac{\varepsilon^{2(L-1)}}{n}\right)\|\bm{V}_{1,2}^T\bm \Delta_W\|_F^2,
\end{align*}
where the first inequality follows from \eqref{eq:delta2} and \eqref{eq:V DeltaW}, and the last inequality uses $\delta \le 1/(16\sqrt{K})$ due to \eqref{eq:delta2}. This, together with \eqref{eq:X Xbar} and \eqref{eq:Delta W}, implies
\begin{align} \label{eq2:thm compre}
\|\bm{V}_{1,2}^T\bm \Delta_W\|_F^2 \ge \frac{\|\bm \Delta_W\|_F^2}{1 +{\varepsilon^{2(L-1)}}/n} \ge \frac{N-K-1}{1 +  {\varepsilon^{2(L-1)}}/n},
\end{align}
where the last inequality follows from $\theta \le 1/4$. Moreover, we compute 
\begin{align}\label{eq3:thm compre}
 \sum_{i=K+1}^{d-K} \|\bm \Delta_W^T\bm v_i\|^2 & = \|\bm{V}_{1,2}^T\bm\Delta_W\|_F^2 - \sum_{i=d-K+1}^{d} \|\bm \Delta_W^T\bm v_i\|^2 \ge \frac{N-K-4\theta}{1 +  {\varepsilon^{2(L-1)}}/n} - K(1+4\theta) \notag \\
 &  \ge \left( 1 -  \frac{\varepsilon^{2(L-1)}}{n}  \right)\left(N-3K-1\right),
\end{align}
where the equality follows from \eqref{eq:V1}, the first inequality uses \eqref{eq2:thm compre} and $\|\bm \Delta_W^T\bm v_i\| \le \|\bm \Delta_W\| \le \sqrt{1+4\theta}$ by \eqref{eq:X Xbar}, and the last inequality is due to $\theta \le 1/4$. According to \eqref{eq0:thm compre}, we obtain
\begin{align}\label{eq10:thm compre}
\|\bm{W}_{l:1}\bm\Delta_W\|_F^2 & \ge   \|\tilde{\bm \Sigma}^{l}_2\bm{V}_{1,2}^T\bm\Delta_W\|_F^2 = \sum_{i=K+1}^{d} \sigma_i^{2l}\|\bm \Delta_W^T\bm v_i\|^2 \ge   \left(\varepsilon-\rho\right)^{2l} \sum_{i=K+1}^{d-K} \|\bm \Delta_W^T\bm v_i\|^2 \notag\\
& \ge \left( 1 -  \frac{\varepsilon^{2(L-1)}}{n} \right)\left(\varepsilon-\rho\right)^{2l} \left(N-3K-1\right),
\end{align}
where the equality follows from \eqref{eq:V1}, the second inequality uses \eqref{eq:singu}, and the last inequality is due to 
\eqref{eq3:thm compre}. This, together with $\theta \le 1/4$, \eqref{eq:X Xbar}, \eqref{eq:Delta W}, \eqref{eq:c12}, and \eqref{eq1:thm compre}, yields 
\begin{align}\label{eq4:thm compre} 
\left(1-\frac{\varepsilon^{2(L-1)}}{n}\right)\left(1-\frac{\rho}{\varepsilon}\right)^2 c_1 \varepsilon^2 \le \frac{\|\bm{W}_{1}\bm\Delta_W\|_F^2}{\| \bm\Delta_W\|_F^2}  \le  \left(1+\frac{1}{n^{1/L}}\right)\frac{\varepsilon^2}{c_1}. 
\end{align}	
Using \eqref{eq1:thm compre} and \eqref{eq10:thm compre} again, we have for all $l\in [L-2]$,
\begin{align}\label{eq6:thm compre}
\frac{1-{\varepsilon^{2(L-1)}}n^{-1}}{1+n^{-1/L}}\left( 1 - \frac{\rho}{\varepsilon}\right)^{2(l+1)} c_1\varepsilon^2 \le  \frac{\|\bm{W}_{l+1:1}\bm\Delta_W\|_F^2}{\|\bm{W}_{l:1}\bm\Delta_W\|_F^2} \le \frac{ 1+n^{-1/L}}{1-{\varepsilon^{2(L-1)}}n^{-1}} \frac{ \varepsilon^2}{c_1\left(1-\rho/\varepsilon\right)^{2l}}. 
\end{align} 
Using the Bernoulli's inequality and $\rho \le \varepsilon/(2L\sqrt{n})$ by \eqref{eq:delta2}, we have  $\left( 1 -  {\rho}/{\varepsilon}\right)^{2l} \ge 1 - 2l\rho/\varepsilon \ge 1 - 2L\rho/\varepsilon \ge 1 - n^{-1/2}$ due to  for all $l \in [L-1]$. This, together with \eqref{eq:c12}, \eqref{eq4:thm compre}, \eqref{eq6:thm compre}, implies for all $l \in [L-1]$,
\begin{align}\label{eq8:thm compre}
\frac{c_1}{c_2}\varepsilon^2 \le  \frac{\|\bm{W}_{l:1}\bm\Delta_W\|_F^2}{\|\bm{W}_{l-1:1}\bm\Delta_W\|_F^2} \le \frac{c_2}{c_1}\varepsilon^2. 
\end{align}


Next, we are devoted to bounding ${\|\bm{W}_{l-1:1}\bm\Delta_B\|_F^2}/{\|\bm{W}_{l:1}\bm\Delta_B\|_F^2}$. Using \eqref{eq0:thm compre}, we have for all $l=0,1,\dots,L-1$, 
\begin{align*}
\|\bm W_{l:1}\bm \Delta_B\|_F^2 = \|\tilde{\bm \Sigma}^{l}_1 \bm{V}_{1,1}^T\bm\Delta_B\|_F^2 + \|\tilde{\bm \Sigma}^{l}_2\bm{V}_{1,2}^T\bm\Delta_B\|_F^2.
\end{align*}
Therefore, we compute  
\begin{align*}
\frac{\|\bm\Delta_B\|_F^2}{\|\bm{W}_{1}\bm\Delta_B\|_F^2} & = \frac{\|\bm{V}_{1,1}^T\bm\Delta_B\|_F^2 + \|\bm{V}_{1,2}^T\bm\Delta_B\|_F^2}{\|\tilde{\bm \Sigma}_1 \bm{V}_{1,1}^T\bm\Delta_B\|_F^2 + \|\tilde{\bm \Sigma}_2\bm{V}_{1,2}^T\bm\Delta_B\|_F^2} \le \frac{1+ 4\theta^2}{(n/2)^{1/L}} \le \frac{2}{(n/2)^{1/L}},
\end{align*}
where the first inequality follows from \eqref{eq:bound Wl} and \eqref{eq:V DeltaB}, and the last inequality uses $\theta \le 1/4$. Moreover, we compute
\begin{align*}
\frac{\|\bm\Delta_B\|_F^2}{\|\bm{W}_{1}\bm\Delta_B\|_F^2} \ge \frac{ \|\bm{V}_{1,1}^T\bm\Delta_B\|_F^2}{(2n)^{1/L} \| \bm{V}_{1,1}^T\bm\Delta_B\|_F^2 + \varepsilon^2\|\bm{V}_{1,2}^T\bm\Delta_B\|_F^2} \ge \frac{1}{(2n)^{1/L} + \varepsilon^2},
\end{align*}
where the first inequality uses \eqref{eq:bound Wl} and \eqref{eq:singu}, and the last inequality follows from \eqref{eq:V DeltaB} and  $\theta \le 1/4$. These, together with \eqref{eq:nc1} and \eqref{eq4:thm compre}, yield \eqref{eq0:ratio compre}. Next, we compute for all $l \in [L-2]$, 
\begin{align*}
\frac{\|\bm{W}_{l:1}\bm\Delta_B\|_F^2}{\|\bm{W}_{l+1:1}\bm\Delta_B\|_F^2} & = \frac{\|\tilde{\bm \Sigma}^{l}_1 \bm{V}_{1,1}^T\bm\Delta_B\|_F^2 + \|\tilde{\bm \Sigma}^{l}_2\bm{V}_{1,2}^T\bm\Delta_B\|_F^2}{\|\tilde{\bm \Sigma}^{l+1}_1 \bm{V}_{1,1}^T\bm\Delta_B\|_F^2 + \|\tilde{\bm \Sigma}^{l+1}_2\bm{V}_{1,2}^T\bm\Delta_B\|_F^2} \le \frac{\|\tilde{\bm \Sigma}^{l}_1 \bm{V}_{1,1}^T\bm\Delta_B\|_F^2 + \varepsilon^{2l}\|\bm{V}_{1,2}^T\bm\Delta_B\|_F^2}{(n/2)^{1/L}\|\tilde{\bm \Sigma}^{l}_1 \bm{V}_{1,1}^T\bm\Delta_B\|_F^2} \\
& \le \frac{1+ \frac{\varepsilon^{2l}}{4\left( n/2 \right)^{l/L}} }{\left( n/2 \right)^{1/L}} \le \frac{1+n^{-1/L}}{(n/2)^{1/L}},
\end{align*}
where the first inequality uses \eqref{eq:singu} and \eqref{eq:bound Wl}, and the second inequality follows from 
\begin{align}\label{eq4:thm sepa}
\|\tilde{\bm \Sigma}^{l}_1 \bm{V}_{1,1}^T\bm\Delta_B\|_F \ge \left( n/2 \right)^{l/2L} \|\bm{V}_{1,1}^T\bm\Delta_B\|_F \ge 2 \left( n/2 \right)^{l/2L} \|\bm{V}_{1,2}^T\bm\Delta_B\|_F
\end{align}
due to \eqref{eq:bound Wl} in Lemma \ref{lem:bound Wl} and \eqref{eq:V DeltaB} with $\theta \le 1/4$, and the last inequality is due to $\varepsilon < 1$. On the other hand, we compute for all $l \in [L-1]$,  
\begin{align*}
\frac{\|\bm{W}_{l:1}\bm\Delta_B\|_F^2}{\|\bm{W}_{l+1:1}\bm\Delta_B\|_F^2} & = \frac{\|\tilde{\bm \Sigma}^{l}_1 \bm{V}_{1,1}^T\bm\Delta_B\|_F^2 + \|\tilde{\bm \Sigma}^{l}_2\bm{V}_{1,2}^T\bm\Delta_B\|_F^2}{\|\tilde{\bm \Sigma}^{l+1}_1 \bm{V}_{1,1}^T\bm\Delta_B\|_F^2 + \|\tilde{\bm \Sigma}^{l+1}_2\bm{V}_{1,2}^T\bm\Delta_B\|_F^2} \\ 
& \ge \frac{\|\tilde{\bm \Sigma}^{l}_1 \bm{V}_{1,1}^T\bm\Delta_B\|_F^2}{(2n)^{1/L}\|\tilde{\bm \Sigma}^{l}_1 \bm{V}_{1,1}^T\bm\Delta_B\|_F^2 + (\varepsilon-\rho)^{2(l+1)}\|\bm{V}_{1,2}^T\bm\Delta_B\|_F^2} \\
& \ge \frac{1}{(2n)^{1/L} + \frac{(\varepsilon-\rho)^{2(l+1)}}{4 \left( n/2 \right)^{l/L}}} \ge \frac{1}{(2n)^{1/L} + n^{-1/L}},
\end{align*}
where the first inequality uses \eqref{eq:singu}  and \eqref{eq:bound Wl}, the second inequality follows from \eqref{eq4:thm sepa}, and the last inequality uses $0 \le \rho \le \varepsilon \le 1$. Therefore, we have for all $l \in [L-2]$,
\begin{align*}
\frac{1}{(2n)^{1/L} + n^{-1/L}} \le \frac{\|\bm{W}_{l:1}\bm\Delta_B\|_F^2}{\|\bm{W}_{l+1:1}\bm\Delta_B\|_F^2} \le  \frac{1+n^{-1/L}}{(n/2)^{1/L}}. 
\end{align*}
This, together with \eqref{eq:nc1} and \eqref{eq8:thm compre}, yield \eqref{eq1:ratio compre}.  

\end{proof}

\subsection{Analysis of Progressive Between-Class Discrimination}\label{subsec:discr}

Now, let us turn to study the progressive discrimination of between-class features across layers. Recall that the mean of features in the $k$-th class at the $l$-th layer is 
\begin{align}\label{eq:muk}
\bm \mu_{k}^l = \frac{1}{n}\sum_{i=1}^n \bm z_{k,i}^l = \bm W_{l:1} \bar{\bm x}_k,\ \forall k \in [K],\ l \in [L-1],
\end{align}
where $\bm z_{k,i}^l$ and $\bar{\bm x}_k$ are defined in \eqref{eq:zl} and \eqref{eq:xk}, respectively. Moreover, recall from \eqref{eq:SVD Wl} that $\bm V_{1,1}$ is the right singular matrix of $\bm W_1$. Using the results in Sections \ref{subsec:weight} and \ref{subsec:compre}, we are ready to derive the explicit form of $\bm V_{1,1}\bar{\bm x}_k$ and bound the norm of $\bm \mu_k^l$.  

\begin{lemma}\label{lem:Vx}
Suppose that Assumptions \ref{AS:1} and \ref{AS:2} hold with $\delta$ and $\varepsilon$ satisfying \eqref{eq:delta2}. Then, the following statements hold: \\
(i) For each $k \in [K]$,  there exist a block diagonal matrix $\bm{O} = \mathrm{BlkDiag}(\bm{Q}_{1},\dots,\bm{Q}_{p}) \in \R^{K\times K}$, where $\bm Q_i \in \mO^{h_i}$ for all $i \in [p]$, and a matrix $\bm P \in \mO^K$ such that 
\begin{align}\label{eq:Vx}
\bm V_{1,1}^T \bar{\bm x}_k = \bm O \tilde{\bm \Sigma}^{-(L-1)}_1 \bm P \bm \Sigma_{L,1}^{-1} \bm U_L^T\bm e_k  + \bm \omega_k, 
\end{align}
where $\|\bm \omega _k\| \le  {4\sqrt{\delta}\sqrt[4]{K}}/n^{\frac{1}{2}+\frac{1}{2L}}$. \\
(ii) It holds for all $k \in [K]$ and $ l \in [L-1]$ that
\begin{align}\label{eq:norm muk}
\|\bm \mu_k^l\| \ge \frac{1}{2\sqrt{2n}}\left( \frac{n}{2} \right)^{\frac{l}{2L}}.  
\end{align}
\end{lemma}
\begin{proof}
According to \eqref{eq:orth}, we have $\bm W_{L:1}\bm X = \bm Y$. This, together with $\bm Y = \bm I_K \otimes \bm{1}_n^T$ and \eqref{eq:xk}, implies 
\begin{align}\label{eq1:lem VX}
\bm W_{L:1}\bar{\bm x}_k = \bm e_k,\ \forall k \in [K]. 
\end{align} 
This, together with \eqref{eq:SVD WL} and \eqref{eq:Wl1}, yields that there exist block diagonal matrices $\bm{O}_{1} = \mathrm{BlkDiag}$ $(\bm{Q}_{1},\dots,\bm{Q}_{p}) \in \R^{K\times K}$ and $\bm{O}_{2} = \mathrm{BlkDiag}(\bm{Q}_{p+1},\dots,\bm{Q}_{p+q}) \in \R^{(d-K)\times (d-K)}$ with $\bm{Q}_{i} \in \mO^{h_i}$ for all $i \in [p+q]$ such that for all $k \in [K]$,
\begin{align}
\bm U_L \bm \Sigma_{L,1} \bm V_{L,1}^T\left( \bm{U}_{L-1,1} \bm{O}_{1} \tilde{\bm \Sigma}^{L-1}_1 \bm{V}_{1,1}^T + \bm{U}_{L-1,2}  \bm{O}_{2} \tilde{\bm \Sigma}^{L-1}_2 \bm{V}_{1,2}^T \right)\bar{\bm x}_k = \bm e_k. 
\end{align}
Since $\bm U_L \in \mO^K$ and $\bm \Sigma_{L,1} \in \R^{K\times K}$ is invertible due to \eqref{eq:bound Wl}, we have
\begin{align}\label{eq2:lem VX}
\bm V_{L,1}^T\left( \bm{U}_{L-1,1} \bm{O}_{1} \tilde{\bm \Sigma}^{L-1}_1 \bm{V}_{1,1}^T + \bm{U}_{L-1,2} \bm{O}_{2} \tilde{\bm \Sigma}^{L-1}_2 \bm{V}_{1,2}^T \right)\bar{\bm x}_k = \bm \Sigma_{L,1}^{-1} \bm U_L^T\bm e_k 
\end{align}  
Let $\bV_{L,1}^T\bU_{L-1,1} = \bm P \bm \Lambda \bm Q^T$ be an SVD of $\bV_{L,1}^T\bU_{L-1,1}$, where $\bm P, \bm Q \in \mO^K$ and $\bm \Lambda \in \R^{K\times K}$ is diagonal. This allows us to compute
\begin{align}\label{eq3:lem VX}
\left\|\bm P \bm Q^T - \bV_{L,1}^T\bU_{L-1,1}\right\| = \left\|\bm P\left( \bm I - \bm \Lambda\right) \bm Q^T\right\|  = \| \bm I - \bm \Lambda\|  \le \frac{\sqrt{2\delta}\sqrt[4]{K}}{n^{1/2L}},
\end{align}
where the inequality follows from \eqref{eq:UL=VL} in \Cref{lem:UV}. Now, we rewrite \eqref{eq2:lem VX} as
\begin{align}\label{eq4:lem VX}
\bm P \bm Q^T \bm{O}_{1} \tilde{\bm \Sigma}^{L-1}_1 \bm{V}_{1,1}^T\bar{\bm x}_k  = \bm \Sigma_{L,1}^{-1} \bm U_L^T\bm e_k  + \bm \xi_k,  
\end{align} 
where $\bm \xi_k : = ( \bm P \bm Q^T - \bV_{L,1}^T\bU_{L-1,1} )\bm{O}_{1} \tilde{\bm \Sigma}^{L-1}_1 \bm{V}_{1,1}^T\bar{\bm x}_k - \bm V_{L,1}^T\bm{U}_{L-1,2} \bm{O}_{2}  \tilde{\bm \Sigma}^{L-1}_2\bm{V}_{1,2}^T\bar{\bm x}_k $. This, together with \eqref{eq:bound Wl}, yields
\begin{align}\label{eq5:lem VX}
\bm{V}_{1,1}^T\bar{\bm x}_k = \tilde{\bm \Sigma}^{-(L-1)}_1	\bm{O}_{1}^T\bm{Q}\bm{P}^T \bm \Sigma_{L,1}^{-1} \bm U_L^T\bm e_k  + \bm \omega_k, 
\end{align}
where $\bm \omega_k := \tilde{\bm \Sigma}^{-(L-1)}_1 \bm{O}_{1}^T \bm{Q}\bm{P}^T \bm \xi_k$. Then, we are devoted to bounding the norm of $\bm \omega_k$. We compute
\begin{align*}
\|\bm \xi_k\| & \le  \|( \bm P \bm Q^T - \bV_{L,1}^T\bU_{L-1,1} )\bm{O}_{1} \tilde{\bm \Sigma}^{L-1}_1\bm{V}_{1,1}^T\bar{\bm x}_k\| + \|\bm V_{L,1}^T\bm{U}_{L-1,2} \bm{O}_{2} \tilde{\bm \Sigma}^{L-1}_2 \bm{V}_{1,2}^T\bar{\bm x}_k\| \\
& \le \frac{\sqrt{2\delta}\sqrt[4]{K}}{n^{1/2L}}\left(\sqrt{2n} \right)^{\frac{L-1}{L}} \|\bar{\bm x}_k\|  + \frac{\sqrt{2\delta}\sqrt[4]{K}\varepsilon^{L-1} }{n^{1/2L}} \|\bar{\bm x}_k\| \le \frac{5\sqrt{\delta}\sqrt[4]{K}}{4n^{1/2L}} \left( \sqrt{2n} \right)^{\frac{L-1}{L}} \left( \frac{1}{\sqrt{n}} + \sqrt{\frac{\theta}{nK}} \right),
\end{align*}
where the second inequality follows from \eqref{eq:singu}, \eqref{eq:bound Wl}, \eqref{eq:UL=VL}, and \eqref{eq3:lem VX}, and the last inequality uses $\varepsilon \le n^{1/2L}/4$ by \eqref{eq:delta2} and $\|\bar{\bm x}_k\|^2 = \|\sum_{i=1}^n \bm x_{k,i}\|^2/n^2 \le 1/n + \theta/N$ due to \eqref{eq:orth} and $N=nK$. This, together with \eqref{eq:bound Wl}, implies 
\begin{align}\label{eq6:lem VX}
\|\bm \omega_k\| &= \|\tilde{\bm \Sigma}^{-(L-1)}_1\bm{O}_{1}^T \bm{Q}\bm{P}^T \bm \xi_k\| \le \|\tilde{\bm \Sigma}^{-(L-1)}_1\| \|\bm \xi_k\| \notag \\
& \le \frac{5\sqrt{\delta}\sqrt[4]{K}}{4n^{1/2L}} \left( \frac{1}{\sqrt{n}} + \sqrt{\frac{\theta}{nK}} \right) \left( 2n \right)^{\frac{L-1}{2L}}  \left( \frac{2}{n} \right)^{\frac{L-1}{2L}} \le \frac{4\sqrt{\delta}\sqrt[4]{K}}{n^{\frac{1}{2}+\frac{1}{2L}}}, 
\end{align}
where the last inequality follows from $\theta/K \le 1/8$. This, together with \eqref{eq5:lem VX} and interchangeness between $\tilde{\bm \Sigma}_1$ and $\bm O_1$, completes the proof of (i). 

Next, we prove (ii). According to \eqref{eq:Wl1} and \eqref{eq:muk}, we have
\begin{align}\label{eq7:lem VX} 
\|\bm \mu_k^l\|^2 = \|\bm W_{l:1}\bar{\bm x}_k\|^2 = \|\tilde{\bm \Sigma}^{l}_1 \bm{V}_{1,1}^T\bar{\bm x}_k\|^2 + \|\tilde{\bm \Sigma}^{l}_2 \bm{V}_{1,2}^T\bar{\bm x}_k\|^2.
\end{align}
According to \eqref{eq:Vx}, we compute
\begin{align*}
\|\tilde{\bm \Sigma}^{l}_1 \bm{V}_{1,1}^T\bar{\bm x}_k\| & = \left\|\tilde{\bm \Sigma}^{l}_1\left(\bm{O} \tilde{\bm \Sigma}^{-(L-1)}_1\bm Q \bm \Sigma_{L,1}^{-1} \bm U_L^T\bm e_k  + \bm \omega_k\right)\right\| \\
& \ge  \sigma_{\min}\left(\tilde{\bm \Sigma}^{l}_1\right) \left(\left\| \tilde{\bm \Sigma}^{-(L-1)}_1\bm Q \bm \Sigma_{L,1}^{-1} \bm U_L^T\bm e_k\right\|  - \|\bm \omega_k \|\right)  \\
& \ge  \left( \frac{n}{2} \right)^{\frac{l}{2L}}\left( \frac{1}{\sqrt{2n}} - \frac{4\sqrt{\delta}\sqrt[4]{K}}{n^{\frac{1}{2}+\frac{1}{2L}}} \right) \ge \frac{1}{2\sqrt{2n}}\left( \frac{n}{2} \right)^{\frac{l}{2L}},
\end{align*}
where the first inequality uses the triangular inequality and the fact that $\bm O$ and $\tilde{\bm \Sigma}_1$ can be interchanged, the second inequality follows from \eqref{eq:bound Wl} and \eqref{eq6:lem VX}, and the last inequality follows from \eqref{eq:delta2}. This, together with \eqref{eq5:lem VX}, implies \eqref{eq:norm muk}.  
\end{proof}

With the above preparations, we are ready to show \eqref{eq:discri} in Theorem \ref{thm:NC}. We now provide a formal version of \eqref{eq:discri} and its proof. 

\begin{theorem}\label{thm:disc}
Suppose that Assumptions \ref{AS:1} and \ref{AS:2} hold with $\delta,\varepsilon,\rho$ satisfying \eqref{eq:delta2}. Then, we have
\begin{align}\label{eq:Dl}
D_l \ge 1 - 32 \left( \theta + 4\delta \right) \left(2-\frac{l+1}{L}\right) - n^{-\Omega(1)},\ \forall l \in [L-1]. 
\end{align} 
\end{theorem}
\begin{proof}
According to \eqref{eq:Wl1} and \eqref{eq:muk}, we compute for all $k \neq k^\prime$,
\begin{align}\label{eq2:thm disc}
 \langle \bm \mu_{k}^l, \bm \mu_{k^\prime}^l \rangle  & =  \bar{\bm x}_k^T \bm W_{l:1}^T\bm W_{l:1} \bar{\bm x}_{k^\prime} =  \bar{\bm x}_k^T  \bm V_{1,1} \tilde{\bm \Sigma}_1^{2l} \bm V_{1,1}^T\bar{\bm x}_{k^\prime}   + \bar{\bm x}_{k} ^T\bm V_{1,2} \tilde{\bm \Sigma}_2^{2l} \bm V_{1,2}^T   \bar{\bm x}_{k^\prime}. 
\end{align}
First, substituting \eqref{eq:Vx} in Lemma \ref{lem:Vx} into the above first term yields 
\begin{align}\label{eq0:thm disc}
 \bar{\bm x}_k^T  \bm V_{1,1} \tilde{\bm \Sigma}_1^{2l} \bm V_{1,1}^T\bar{\bm x}_{k^\prime} 
& \le \left|\bm e_k^T\bm U_L\bm \Sigma_{L,1}^{-1}\bm P^T\tilde{\bm \Sigma}^{-2(L-1)+2l}_1 \bm P \bm \Sigma_{L,1}^{-1} \bm U_L^T\bm e_{k^\prime}\right| +  \left|\bm \omega_{k^\prime}^T \tilde{\bm \Sigma}_1^{2l}\bm{O} \tilde{\bm \Sigma}^{-(L-1)}_1\bm P \right.   \notag \\
& \left. \bm \Sigma_{L,1}^{-1} \bm U_L^T\bm e_{k}\right| + \left|\bm \omega_k^T \tilde{\bm \Sigma}_1^{2l}\bm O \tilde{\bm \Sigma}^{-(L-1)}_1\bm P \bm \Sigma_{L,1}^{-1} \bm U_L^T\bm e_{k^\prime}\right| + \left|\bm \omega_k^T \tilde{\bm \Sigma}_1^{2l}\bm \omega_{k^\prime}\right|,
\end{align}
where $\bm P, \bm O$ are defined in (i) of Lemma \ref{lem:Vx}.  Now, we bound the above terms in turn. 
According to \eqref{eq:bound Wl1}, we compute for all $l \in [L-1]$,
\begin{align}\label{eq1:thm disc}
&\quad \left\| \tilde{\bm \Sigma}^{-2(L-1)+2l}_1 - n^{\frac{-(L-1)+l}{L}} \bm I \right\| \notag \\
& \le  n^{\frac{-(L-1)+l}{L}}\max \left\{ 1- \left( \frac{1}{1-\theta} + 3\delta \right)^{\frac{-(L-1)+l}{L}}, \left( \frac{1}{1+\theta} - 3\delta \right)^{\frac{-(L-1)+l}{L}} - 1 \right\}  \notag \\
& \le n^{\frac{-(L-1)+l}{L}} \frac{2\left( \theta + 4\delta \right)}{L}(L-l-1),
\end{align}
where the last inequality follows from Bernoulli's inequality. Using the similar argument, it follows from \eqref{eq:bound Wl2} that  
\begin{align}\label{eq4:thm disc}
\|\bm \Sigma_{L,1}^{-2} -  n^{-\frac{1}{L}} \bm I  \| \le 2n^{-\frac{1}{L}}(\theta+4\delta).
\end{align}
Then, we compute 
\begin{align*}
& \qquad \left|\bm e_k^T\bm U_L\bm \Sigma_{L,1}^{-1}\bm P^T\tilde{\bm \Sigma}^{-2(L-1)+2l}_1 \bm P \bm \Sigma_{L,1}^{-1} \bm U_L^T\bm e_{k^\prime}\right| \\
& \le \left|\bm e_k^T\bm U_L \bm \Sigma_{L,1}^{-1}\bm P^T\left( \tilde{\bm \Sigma}^{-2(L-1)+2l}_1 - n^{\frac{-(L-1)+l}{L}} \bm I\right) \bm P \bm \Sigma_{L,1}^{-1}  \bm U_L^T\bm e_{k^\prime}\right| +  \\
&\qquad  \left| n^{\frac{-(L-1)+l}{L}}  \bm e_k^T\bm U_L  \left( \bm\Sigma_{L,1}^{-2} - n^{-\frac{1}{L}} \bm I \right)  \bm U_L^T\bm e_{k^\prime}\right| \\
& \le \|\bm \Sigma_{L,1}\|^{-2} \left\| \tilde{\bm \Sigma}^{-2(L-1)+2l}_1 - n^{\frac{-(L-1)+l}{L}} \bm I \right\|  +  n^{\frac{-(L-1)+l}{L}}\|\bm \Sigma_{L,1}^{-2} -  n^{-\frac{1}{L}} \bm I  \|  \\	
& \le 2^{1+1/L}\left( \theta + 4\delta \right) n^{\frac{-L+l}{L}} \frac{L-l-1}{L} + 2(\theta + 4\delta)n^{\frac{-L+l}{L}},
\end{align*}
where  the second inequality follows from \eqref{eq1:thm disc} and \eqref{eq4:thm disc}. Next, we compute 
\begin{align*}
\left|\bm \omega_{k^\prime}^T \tilde{\bm \Sigma}_1^{2l}\bm{O} \tilde{\bm \Sigma}^{-(L-1)}_1\bm P \bm \Sigma_{L,1}^{-1} \bm U_L^T\bm e_{k}\right| & \le \|\bm \omega_{k^\prime}\| \|\tilde{\bm \Sigma}_1\|^{-L+2l+1} \|\bm \Sigma_{L,1}\|^{-1} \\
& \le 8 \sqrt{\delta}\sqrt[4]{K} n^{\frac{-L+l}{L}} n^{-\frac{1}{2L}},
\end{align*}
where the second inequality follows from (i) in Lemma \ref{lem:Vx} and \eqref{eq:bound Wl}. Using the same argument, we have 
\begin{align*}
\left|\bm \omega_k^T \tilde{\bm \Sigma}_1^{2l}\bm O \tilde{\bm \Sigma}^{-(L-1)}_1\bm P \bm \Sigma_{L,1}^{-1} \bm U_L^T\bm e_{k^\prime}\right| \le 8\sqrt{\delta}\sqrt[4]{K} n^{\frac{-L+l}{L}} n^{-\frac{1}{2L}}. 
\end{align*}
and 
\begin{align*}
\left|\bm \omega_k^T \tilde{\bm \Sigma}_1^{2l}\bm \omega_{k^\prime}\right| \le 16\delta \sqrt{K}n^{\frac{-L+l}{L}} n^{-\frac{1}{L}} \le \sqrt{\delta}\sqrt[4]{K} n^{\frac{-L+l}{L}} n^{-\frac{1}{2L}}. 
\end{align*}
where the last inequality follows from \eqref{eq:delta2}. These, together with \eqref{eq0:thm disc}, yield
\begin{align}\label{eq3:thm disc}
 \bar{\bm x}_k^T  \bm V_{1,1} \tilde{\bm \Sigma}_1^{2l} \bm V_{1,1}^T\bar{\bm x}_{k^\prime}  & \le 2\sqrt{2} \left( \theta + 4\delta \right) n^{\frac{-L+l}{L}} \frac{L-l-1}{L} + 2(\theta + 4\delta)n^{\frac{-L+l}{L}} + \notag \\
&\qquad 9\sqrt{\delta}\sqrt[4]{K} n^{\frac{-L+l}{L}} n^{-\frac{1}{2L}}.  
\end{align}
Moreover, we compute
\begin{align*}
 \bar{\bm x}_{k} ^T\bm V_{1,2} \tilde{\bm \Sigma}_2^{2l} \bm V_{1,2}^T   \bar{\bm x}_{k^\prime}  \le \|\bar{\bm x}_k\|\| \bar{\bm x}_{k^\prime}\| \| \tilde{\bm \Sigma}_2^{2l} \| \le \left( \frac{1}{n} + \frac{\theta}{N} \right)\varepsilon^{2l}. 
\end{align*}
where the second inequality follows from $\|\bar{\bm x}_k\|^2 \le \|\sum_{i=1}^n \bm x_{k,i}/n\|^2 \le 1/n + \theta/N$ due to Assumption \ref{AS:1}. This, together with \eqref{eq:norm muk} in Lemma \ref{lem:Vx}, \eqref{eq2:thm disc}, and \eqref{eq3:thm disc}, yields for all $k,l$,
\begin{align*}
\frac{ \langle \bm \mu_{k}^l, \bm \mu_{k^\prime}^l \rangle }{\|\bm \mu_{k}^l\| \|\bm \mu_{k^\prime}^l\|} & \le 32 \left( \theta + 4\delta \right) \left(2-\frac{l+1}{L}\right)  \\
&\quad + 144\sqrt{\delta}\sqrt[4]{K} n^{-\frac{1}{2L}} + 32n^{-\frac{l}{L}} \left( 1+  \frac{\theta}{K} \right)\varepsilon^{2l}.  
\end{align*} 
This, together with \eqref{eq:nc1}, yields \eqref{eq:Dl}. 
\end{proof}

\subsection{Proof of \Cref{thm:NC}}\label{subsec:pf thm}

Based on \Cref{thm:comp} and \Cref{thm:disc}, we can directly obtain \Cref{thm:NC}. 

\subsection{Proof of \Cref{prop:AS}}\label{subsec:pf prop}

Now, we study the convergence behavior of gradient flow with the initialization \eqref{eq:init} for solving Problem \eqref{eq:obj}. 

\begin{proof} 
Since GD converges to a global optimal solution, we have $\bm W_{L:1}\bm 
X = \bm Y$. This, together with the fact that $\bm X$ is square and orthogonal, yields \eqref{eq:min norm}. 
According to \cite{du2018algorithmic,arora2018optimization}, the iterates of gradient flow for solving Problem \eqref{eq:obj} satisfy \eqref{eq:gf}. This, together with the initialization \eqref{eq:init}, yields for all $t \ge 0$, 
\begin{align*}
    & \|\bm W_{l+1}(t)^T \bm W_{l+1}(t) - \bm W_{l}(t)\bm W_{l}(t)^T \|_F = 0, \forall l \in [L-2],\\
    & \| \bm W_L(t)^T \bm W_L(t) - \bm W_{L-1}(t) \bm W_{L-1}(t)^T\|_F = \xi^2\sqrt{d-K}. 
\end{align*}
This implies that \eqref{eq:bala} holds with $\delta = \xi^2\sqrt{d-K}$. Moreover, using \cite[Theorem 1]{yaras2023law}, we obtain that \eqref{eq:singu} holds with
$\varepsilon=\xi$ and $\rho=0$. Then, the proof is completed. 
\end{proof}

 \begin{figure}[t]
     \begin{subfigure}{0.5\textwidth}
     \includegraphics[width = 0.9\linewidth]{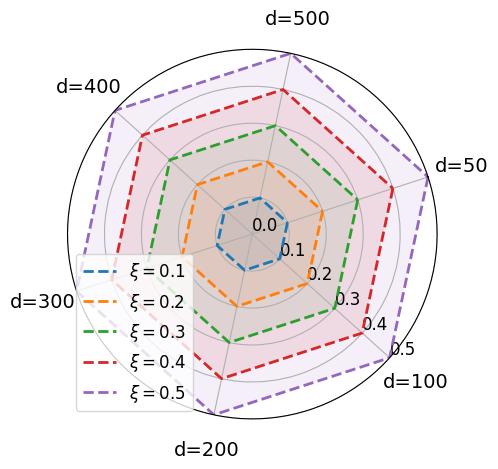}
     \caption{$\min_{l \in [L-1], K+1\leq i\leq d-K} \; \sigma_i(\bm W_l) $} \label{fig:a}
     \end{subfigure} 
     \begin{subfigure}{0.5\textwidth}
     \includegraphics[width = 0.9\textwidth]{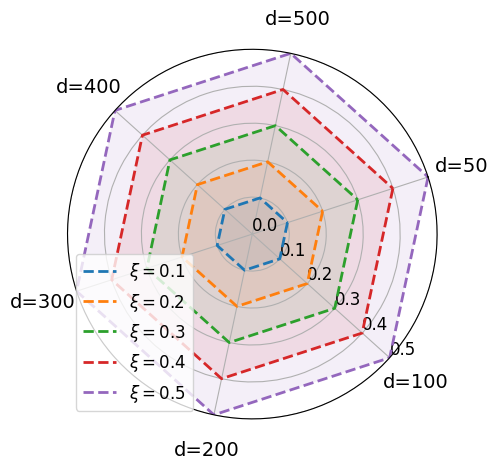}
     \caption{$\max_{l \in [L-1], K+1\leq i\leq d-K} \; \sigma_i(\bm W_l)$}\label{fig:b}
     \end{subfigure} 
     \caption{\textbf{Approximate low-rankness hold universally across different training and network configurations.} The experimental setup is deferred to \Cref{subsec:exp-assump}. For each figure, the vertex of each polygon denotes the singular value of the weights for each $d$ and $\xi$. As we observe, for each fixed $\xi$, the polygon is a hexagon, implying that the singular values are approximately equal and independent of the network width $d$. Second, the singular values grow when the initialization scale $\xi$ increases and also confirms \Cref{prop:AS} that $\varepsilon = \xi$. Moreover, we conclude that $\sigma_{\min} \approx \sigma_{\max}$ when comparing (a) and (b). Therefore, we can see that both the minimum and maximum singular values remain unchanged, demonstrating that the GD operates only within a 2K-dimensional invariant subspace. }
     \label{fig:assump_2_num3}
 \end{figure}

\section{Extra Experiments}\label{app:extra-exp}

\begin{figure*}[t]
    \begin{subfigure}{0.48\textwidth}
    \includegraphics[width = 0.95\linewidth]{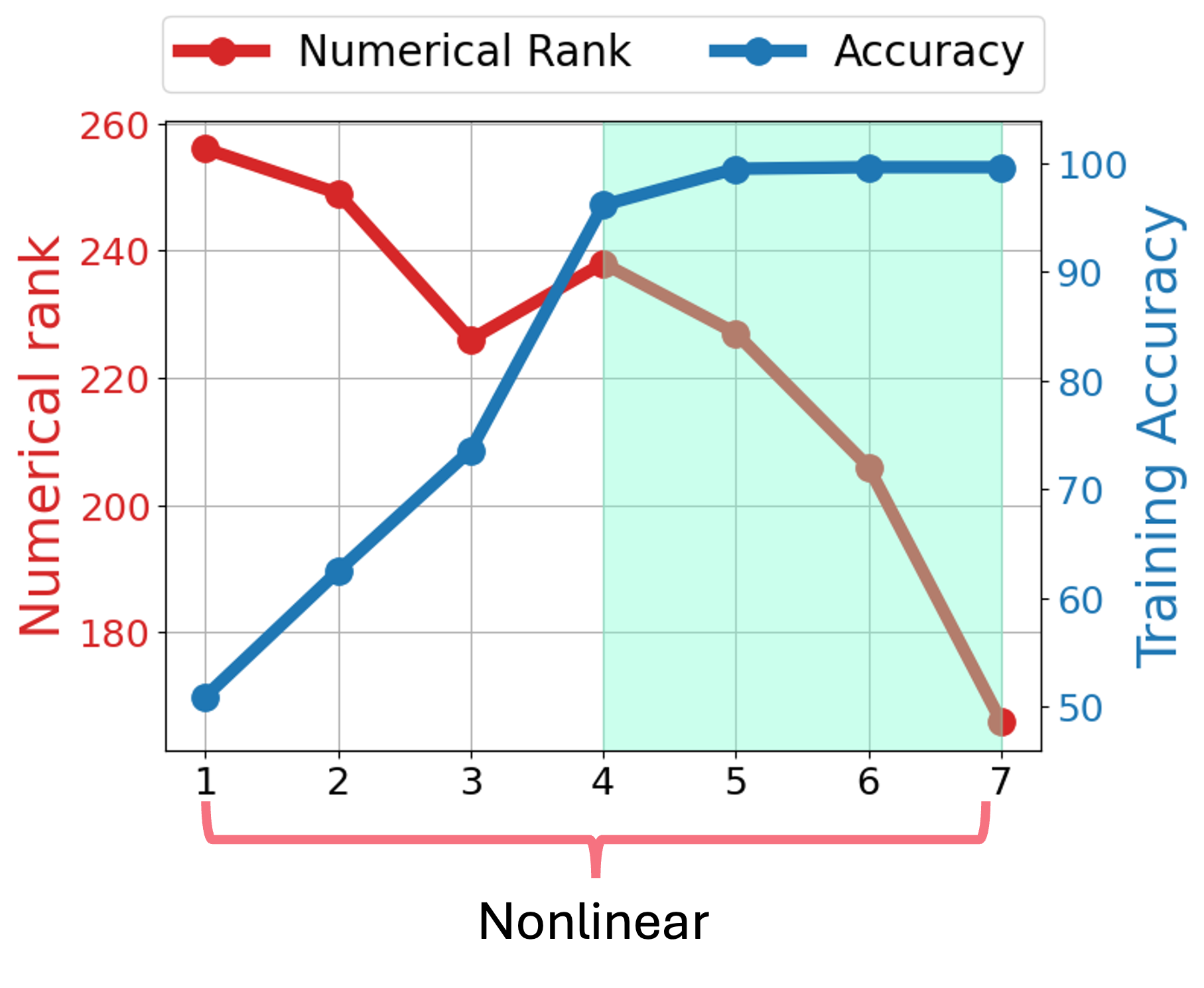}
    \caption{A 7-layer MLP network with ReLU} 
    \end{subfigure} 
    \begin{subfigure}{0.48\textwidth}
    \includegraphics[trim={0 0 0 1cm}, width = 0.94\linewidth]{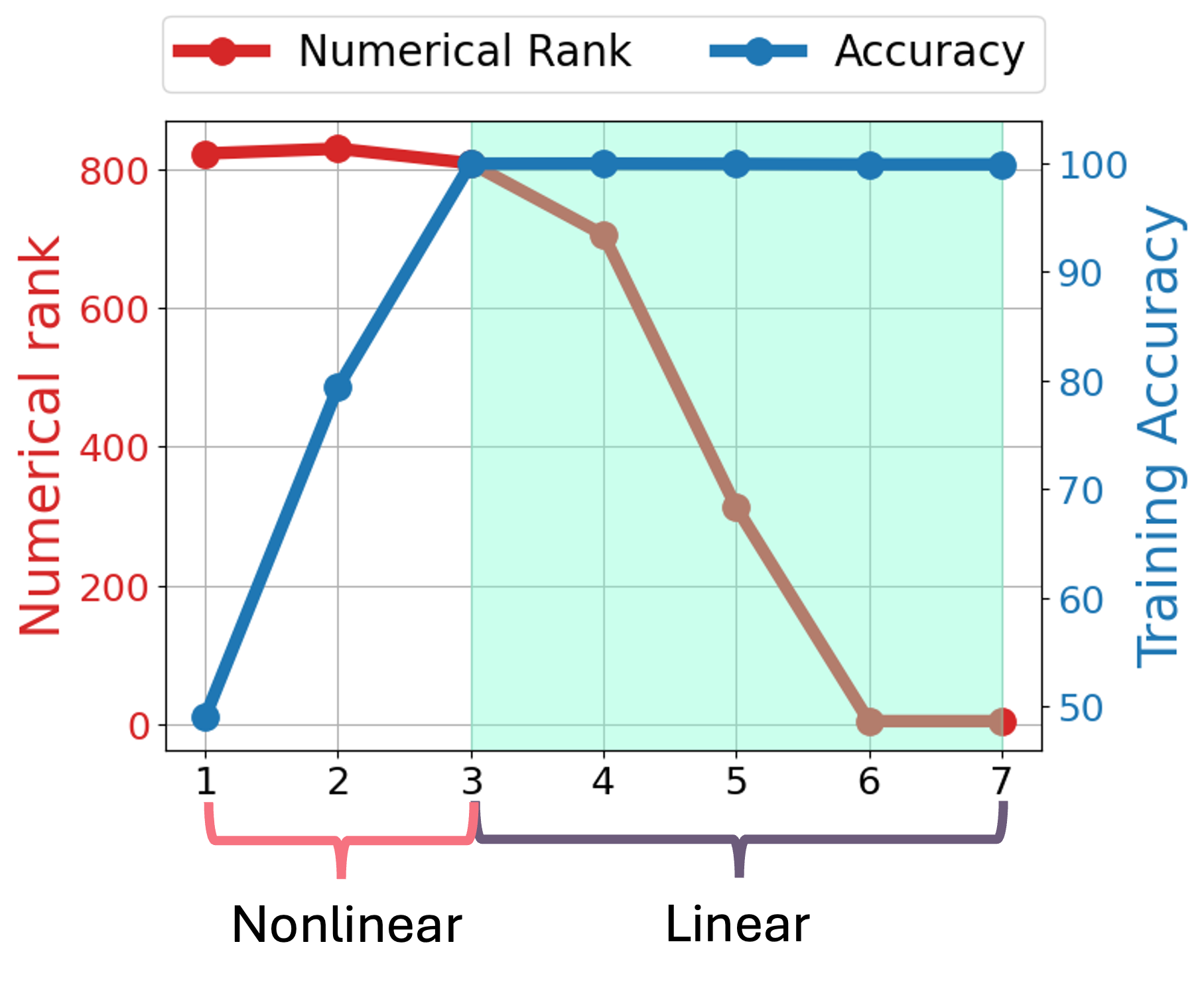}
    \caption{A 7-layer hybrid network} 
    \end{subfigure} 
   \caption{\textbf{Numerical rank and training accuracy across layers on the SST-5 language dataset.} We train two networks with different architectures on the SST-5 dataset: (a) A 7-layer multilayer perceptron (MLP) network with ReLU activation, (b) A hybrid network formed by 3-layer MLP with ReLU activation followed by a 4-layer linear network. For each figure, we plot the numerical rank of the features of each layer and training accuracy by applying linear probing to the output of each layer against the number of layers, respectively. The green shade indicates that the features at these layers are {\em approximately} linearly separable, as evidenced by the near-perfect accuracy achieved by a linear classifier. } 
    \label{fig:nr_acc_language}
\end{figure*}

\begin{figure*}[t]
\begin{center}
       \begin{subfigure}{0.48\textwidth}
    \includegraphics[width = 0.95\linewidth]{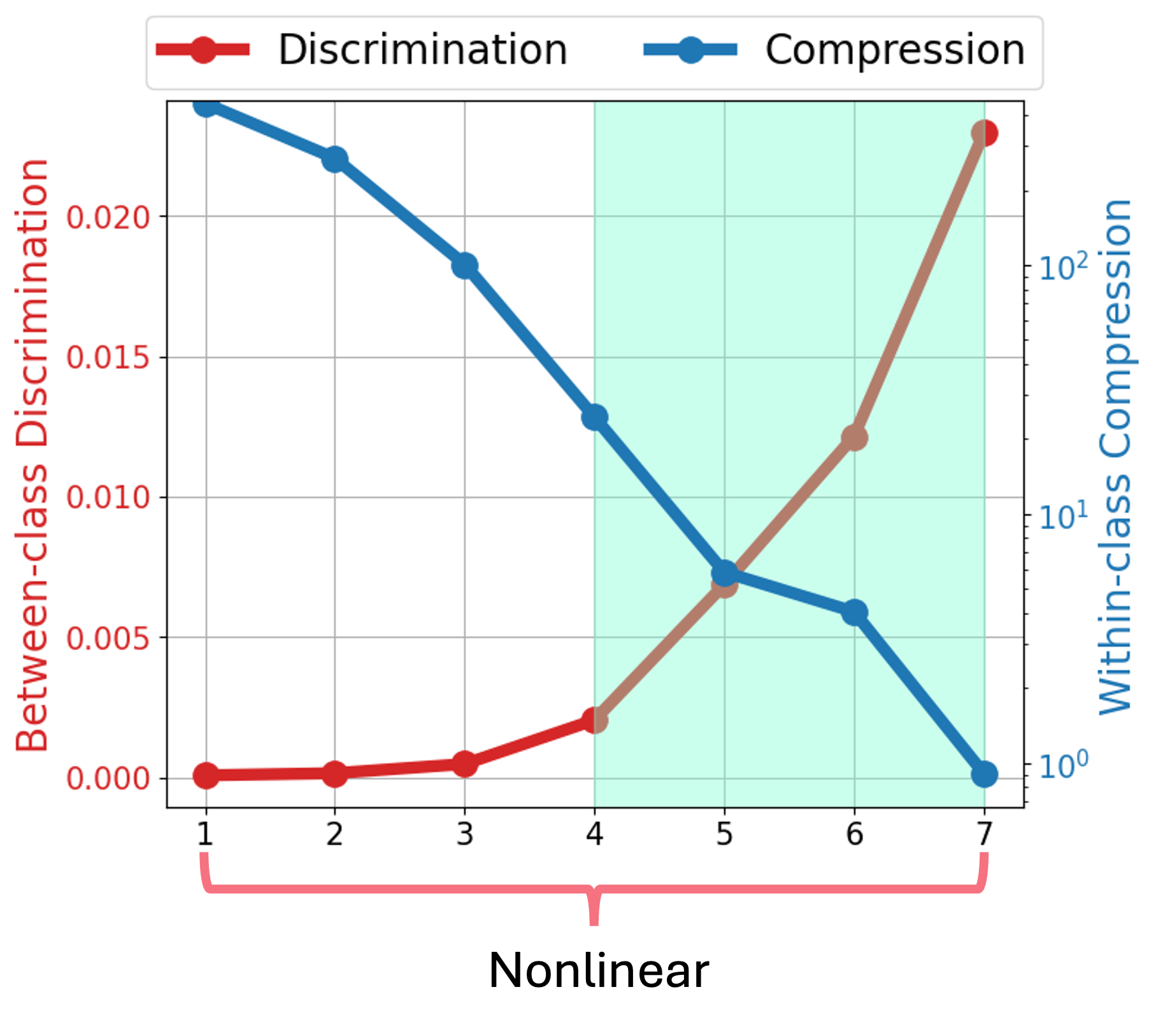}
    \caption{A 7-layer MLP network with ReLU} 
    \end{subfigure} 
    \hfill 
    \begin{subfigure}{0.48\textwidth}
    \includegraphics[trim={0 0 0 1cm}, width = 0.94\linewidth]{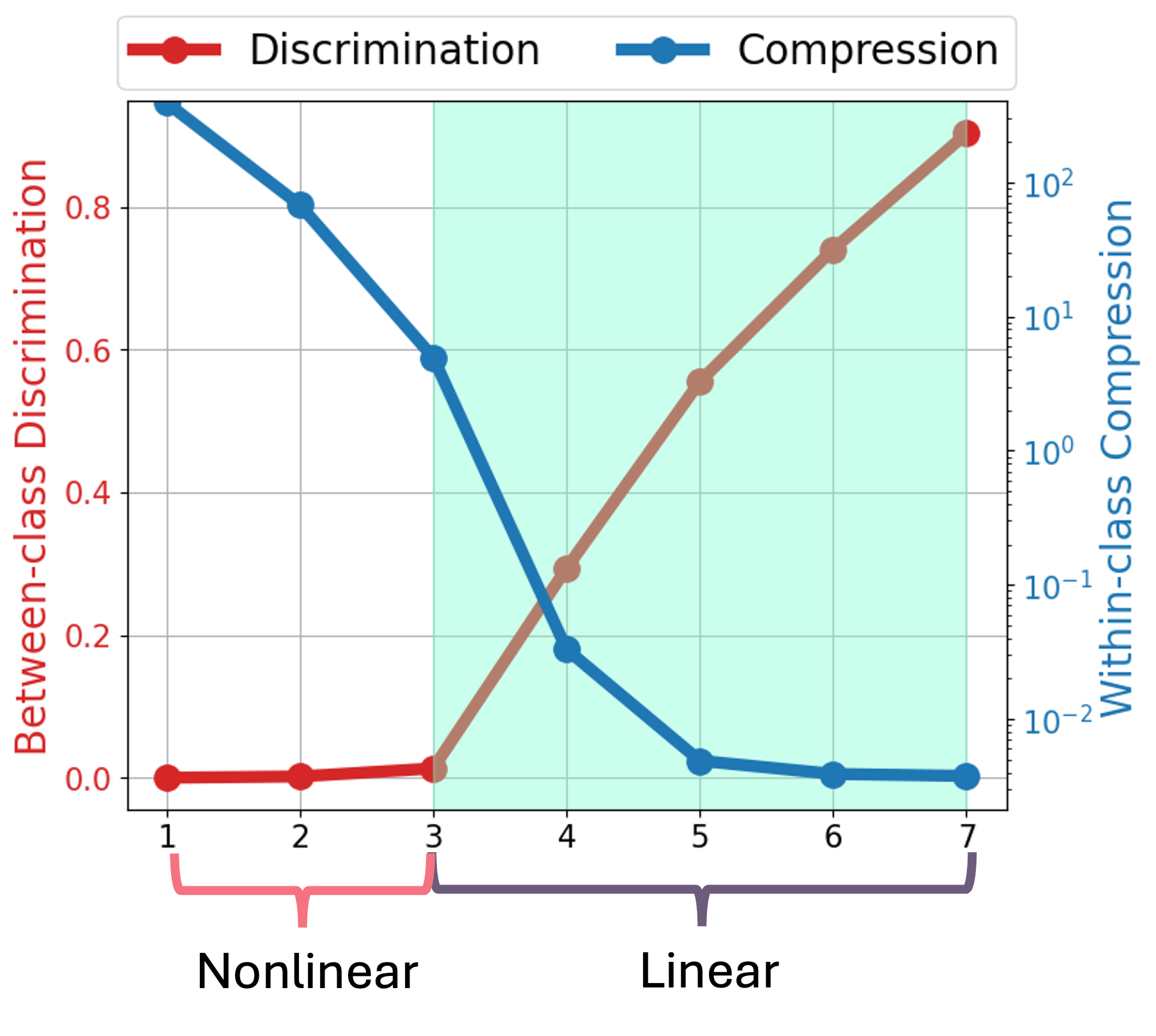}
    \caption{A 7-layer hybrid network} 
    \end{subfigure} 
\end{center}
\vspace{-0.2in}
    \caption{\textbf{Within-class compression and between-class discrimination across layers on the SST-5 text dataset.} We use the same setup as in Figure \ref{fig:nr_acc_language} and plot the metrics of within-class compression and between-class discrimination of the two networks across layers after training, respectively.} 
    \label{fig:comp_disc_language}
\end{figure*}

\subsection{Assumption Discussion}
Here, we provide more experimental results for validating the low-rankness property in \Cref{AS:2} of trained weights in DLN. 
To validate the low-rankness assumption, we train 4-layer DLNs using the orthogonal initialization with width $d \in \{50,100,200,300,400,500\}$ and initialization scale $\xi \in \{0.1,0.2,0.3,0.4,0.5\} $. The training setting is the same with \Cref{subsec:theory_and_assump}. For the weights $\{\bm W_l\}_{l=1}^{L-1}$ of each network that have been trained until convergence, we plot 
  \begin{align*}
      \text{minimum singular value}\quad \sigma_{\min} \;&:=\; \min_{ K+1\leq i\leq d-K, \;l \in [L-1]} \sigma_i(\bm W_l), \\
      \text{maximum singular value}\quad \sigma_{\max} ;&:=\; \max_{ K+1\leq i\leq d-K,\;l \in [L-1]} \sigma_i(\bm W_l) .
  \end{align*}
 The results are shown in \Cref{fig:assump_2_num3}.
  Here, for each $l \in [L-1]$, we assume that $\sigma_1(\bm W_l) \geq \sigma_2(\bm W_l) \geq \cdots \geq \sigma_d(\bm W_l) $. For each figure, the vertex of each polygon denotes the singular value of the weights for each $d$ and $\xi$. As we observe, for each fixed $\xi$, the polygon is a hexagon, implying that the singular values are approximately equal and independent of the network width $d$. Second, the singular values grow when the initialization scale $\xi$ increases and also confirms \Cref{prop:AS} that $\varepsilon = \xi$. Moreover, we conclude that $\sigma_{\min} \approx \sigma_{\max}$ when comparing \Cref{fig:a} and \Cref{fig:b}. Therefore, we can see that both the minimum and maximum singular values remain unchanged, demonstrating that the GD operates only within a 2K-dimensional invariant subspace.

\subsection{Exploration on Other Datasets}\label{app:extra-exp-modality}
In the main body of the paper, we primarily present results on the CIFAR and FashionMNIST datasets. In this subsection, we extend our experiments to a broader range of datasets.


\paragraph{Experiments on Language Dataset.} To further illustrate the analogous roles of deep linear and nonlinear layers across different data modalities, we reproduce the experiments from \Cref{fig:intro1} using the SST-5 text dataset \citep{socher2013recursive}, keeping all experimental settings unchanged except for the dataset. As shown in \Cref{fig:nr_acc_language}, we observe a similar trend: early nonlinear layers enhance linear separability, while deeper layers—regardless of their nonlinearity—progressively compress the features. We also use this experimental setup to verify \Cref{thm:NC}, with results presented in \Cref{fig:comp_disc_language}. As observed, within-class compression and between-class discrimination respectively exhibit progressively collapsing and increasing trends, further confirming that our theoretical insights generalize beyond the image domain.


\paragraph{Experiments on ImageNet.} Due to resource constraints, we use the PyTorch pre-trained VGG11 network \citep{simonyan2014very} on ImageNet \citep{deng2009imagenet} to extract features and evaluate within-class compression and between-class discrimination. The results are presented in \Cref{fig:imagenet}. We observe that the compression metric continues to exhibit an approximately geometric decay, and the discrimination metric shows increasing pattern, thereby validating our theory on the large-scale benchmark dataset.

 \begin{figure}[t]
     \begin{center}
         \includegraphics[width=0.5\linewidth]{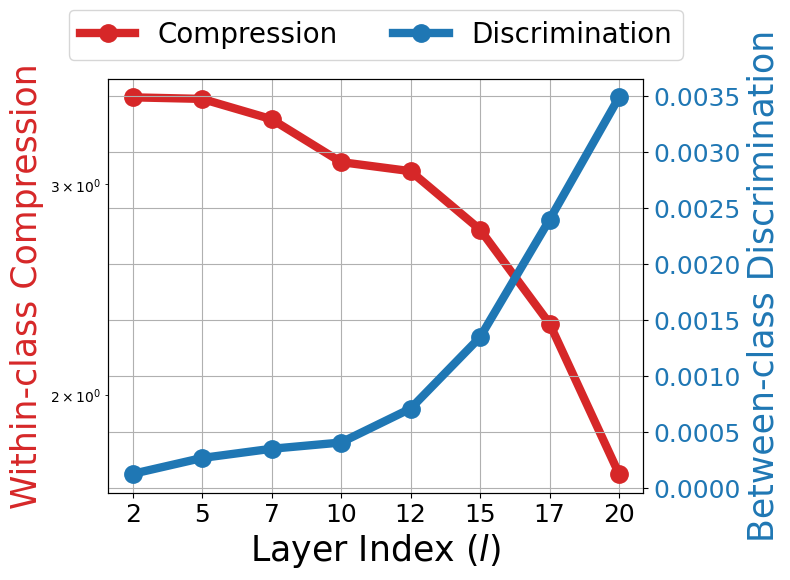}
     \end{center}
     \vspace{-0.2in}
     \caption{\textbf{Within-class compression and between-class discrimination of features from a VGG11 network pre-trained on ImageNet.} We use the PyTorch-provided pre-trained VGG11 model and extract features from layers 2, 5, 7, 10, 12, 15, 17, 20.}
     \label{fig:imagenet}
 \end{figure}


\section{Auxiliary Results}\label{app:theorem}

 \begin{lemma}\label{lem:singular AA}
Given a matrix $\bA \in \R^{m\times n}$ of rank $r \in \mathbb{N}_+$, we have
\begin{align}\label{eq:singular AA}
 \|\bA^T\bA\|_F \le \|\bA\|_F^2 \le \sqrt{r}\|\bA^T\bA\|_F. 
\end{align}
\end{lemma}
\begin{proof}
Let $\bA = \bU\bm{\Sigma}\bV^T$ be a singular value decomposition of $\bA$, where
\begin{align*}
\bm{\Sigma} = \begin{bmatrix}
\tilde{\bm{\Sigma}} & \bm{0} \\
\bm{0} & \bm{0} 
\end{bmatrix},\ \tilde{\bm{\Sigma}} = \diag(\sigma_1,\dots,\sigma_r),
\end{align*} 
$\bU \in \mO^m$, and $\bV \in \mO^n$. Then we compute 
\begin{align*}
 \|\bA^T\bA\|_F^2 = \|\bm{\Sigma}^T\bm{\Sigma}\|_F^2 =  \sum_{i=1}^r \sigma_i^4,\ \|\bA\|_F^2 = \|\bm{\Sigma}\|_F^2 = \sum_{i=1}^r \sigma_i^2,
\end{align*} 
which, together with the AM-QM inequality, directly implies \eqref{eq:singular AA}. 
\end{proof}

\begin{lemma}\label{lem:mini singular}
For arbitrary matrices $\bm{A},\bm{B} \in \R^{n \times n}$, if $\bm{A}$ has full column rank and $\bm{B} \neq \bm{0}$ or $\bm{B}$ has full row rank and $\bm{A} \neq \bm{0}$, it holds that 
\begin{align}
\sigma_{\min}\left(\bm A\bm B\right) \ge \sigma_{\min}\left(\bm A\right) \sigma_{\min}\left(\bm B\right). 
\end{align}
\end{lemma}

\begin{lemma}\label{lem:QQ=UU}
Suppose that $\bm{Q} \in \mO^{d\times K}$ is an orthonormal matrix and $\bm{A} \in \R^{K\times K}$ is a symmetric matrix of full rank, where $d > K$. Let $\bm{U}\bm{\Lambda}\bm{U}^T = \bm{Q}\bm{A}\bm{Q}^T$ be a compact singular value decomposition of $\bm{Q}\bm{A}\bm{Q}^T$, where $\bm{U} \in \mO^{d\times K}$ and $\bm{\Lambda} \in \R^{K\times K}$ is a diagonal matrix. Then, it holds that 
\begin{align}\label{eq:QQ=UU}
\bm{Q}\bm{Q}^T = \bm{U}\bm{U}^T.
\end{align}    
\end{lemma}
\begin{proof}
It follows from $\bm{U}\bm{\Lambda}\bm{U}^T = \bm{Q}\bm{A}\bm{Q}^T$, $\bm{Q} \in \mO^{d\times K}$, and $\bm{U} \in \mO^{d\times K}$ that 
\begin{align}
\bm{Q}^T\bm{U} \bm{\Lambda} = \bm{A}\bm{Q}^T\bm{U}. 
\end{align}
Now, left-multiplying both sides by $\bm{Q}$ yields
\begin{align*}
\bm{Q}\bm{Q}^T\bm{U}\bm{\Lambda} = \bm{Q}\bm{A}\bm{Q}^T\bm{U} = \bm{U}\bm{\Lambda}\bm{U}^T\bm{U} = \bm{U}\bm{\Lambda}.
\end{align*} 
Since $\bm{A}$ is a symmetric matrix of full rank, $\bm{\Lambda}$ is a diagonal matrix with non-zero entries. We can multiply both sides by $\bm{\Lambda}^{-1}$ and obtain $\bm{Q}\bm{Q}^T\bm{U} = \bm{U}$, which implies \eqref{eq:QQ=UU} by $\bm U \in \mO^{d\times K}$. 

\end{proof}

\begin{lemma}\label{lem:Bern}
For every real number $1 \le \theta \le 1$ and $x \ge -1$, it holds that
\begin{align*}
(1+x)^\theta \le 1 + \theta x. 
\end{align*}
\end{lemma}

\vskip 0.2in
{\small 
\bibliographystyle{ieeetr}
\bibliography{pregressive_NC}
}

\end{document}